\documentclass[11pt]{article}
\usepackage{etoolbox}
\usepackage{covington}
\usepackage{graphicx}
\usepackage{latexsym}
\usepackage{amssymb}
\usepackage{amsmath}
\usepackage{amsthm}
\usepackage{thmtools}
\usepackage{caption}
\usepackage{subcaption}
\usepackage{mwe}
\usepackage{derivative}
\usepackage{forloop}
\usepackage[paper=letterpaper,margin=1in]{geometry}
\usepackage[ruled,vlined,linesnumbered]{algorithm2e}
\usepackage{paralist}
\usepackage{xcolor}
\usepackage{nicefrac}
\usepackage{pgfplots}
\usepackage{siunitx}
\usepackage{multirow}
\usepackage{mleftright}\mleftright
\usepackage{numprint}
\usepackage{tikz}
\usepackage{pgfplotstable}
\usepackage[rightcaption]{sidecap}
\usepackage{hyperref}
\usepackage{papercommands}

\makeatletter
\patchcmd{\@algocf@start}
  {-1.5em}
  {0pt}
  {}{}
\makeatother

\usepgfplotslibrary{fillbetween}
\pgfplotsset{compat=1.18}
\SetDiaOffset{-0.35ex}

\definecolor{darkgreen}{rgb}{0,0.5,0}
\hypersetup{
    unicode=false,            
    colorlinks=true,          
    linkcolor=blue!70!black,  
    citecolor=darkgreen,      
    filecolor=magenta,        
    urlcolor=blue!70!black    
}

\author{
	Loay Mualem\thanks{Computer Science Department, University of Haifa. Email: \href{mailto:loaymua@gmail.com}{loaymua@gmail.com}}
	\and
	Murad Tukan\thanks{{DataHeroes}. Email: \href{mailto:muradtuk@gmail.com}{murad@dataheroes.ai}}
	\and
	Moran Feldman\thanks{Computer Science Department, University of Haifa. Email: \href{mailto:moranfe@cs.haifa.ac.il}{moranfe@cs.haifa.ac.il}}
}
\title{Bridging the Gap Between General and Down-Closed Convex Sets in Submodular Maximization}

\newcommand{\say}[1]{``#1''}

\begin{document}

\maketitle

\begin{abstract}

Optimization of DR-submodular functions has experienced a notable surge in significance in recent times, marking a pivotal development within the domain of non-convex optimization. Motivated by real-world scenarios, some recent works have delved into the maximization of non-monotone DR-submodular functions over general (not necessarily down-closed) convex set constraints. Up to this point, these works have all used the minimum $\ell_\infty$ norm of any feasible solution as a parameter. Unfortunately, a recent hardness result due to Mualem \& Feldman~\cite{mualem2023resolving} shows that this approach cannot yield a smooth interpolation between down-closed and non-down-closed constraints. In this work, we suggest novel offline and online algorithms that provably provide such an interpolation based on a natural decomposition of the convex body constraint into two distinct convex bodies: a down-closed convex body and a general convex body. We also empirically demonstrate the superiority of our proposed algorithms across three offline and two online applications.

\medskip

\end{abstract}

\pagenumbering{arabic}

\section{Introduction}

Optimization of continuous DR-submodular functions (and the strongly related discrete submodular set functions) has experienced a notable surge in significance in recent times, marking a pivotal development within the domain of non-convex optimization. The tools developed in this context adaptively tackle challenges related to real-world applications at the forefront of various fields such as data summarization~\cite{mualem2023submodular, hassani2017gradient,bian2019optimal,mitra2021submodular,soma2017nonmonotone}, robotics~\cite{shi2021communication,tukan2023orbslam3}, and human brain mapping~\cite{salehi2017submodular}, among many others. Most of the works in literature focused either on DR-submodular optimization for monotone objective functions, or subject to a down-closed convex set constraint.\footnote{A set $\cK$ is \emph{down-closed} with respect to a domain $\cX$ if, for every two vectors $\vx, \vy \in \cX$, $\vx \in \cK$ whenever $\vy \in \cK$ and $\vy$ coordinate-wise dominates $\vx$. As is standard, we usually omit the domain when talking about down-closed sets, and implicitly assume it to be the domain of the objective function.} However, real-world problems are often naturally captured as optimization of a non-monotone DR-submodular function over a constraint convex set that is not down-closed. For example, 
imagine an online shopping platform optimizing its product recommendations within strict interface size constraints. The challenge faced by the store involves designing concise summaries that respect upper and lower bounds on product inclusion (these bounds are imposed by the user interface, and they form a non-down-closed constraint), and are good with respect to an objective function balancing between diversity and relevance in the displayed recommendations (which naturally yields a non-monotone objective). 

Motivated by scenarios similar to the one described above, and optimization under fairness constraints~\cite{el2023fairness,el2023bfairness,yuan2023group}, some recent works have delved into the maximization of non-monotone DR-submodular functions over general (not necessarily down-closed) convex set constraints. Unfortunately, in general, no constant approximation can be obtained for this problem in sub-exponential time due to a hardness result by Vondr\'{a}k~\cite{vondrak2013symmetry}. However, Durr et al.~\cite{durr2021non} observed that Vondr\'{a}k's proof of the hardness result was based on a class of instances whose convex sets $\cK$ include no point whose $\ell_\infty$-norm is less than $1$.\footnote{For simplicity of the exposition, we implicitly assume that the domain of the objective function and the convex set constraint is $[0, 1]^n$. This assumption is without loss of generality (see Section~\ref{sec:Preliminaries}).} In light of this observation, Durr et al.~\cite{durr2021non} considered a parametrization of the problem based on the minimum $\ell_\infty$ norm of any vector in $\cK$ (this minimum is usually denoted by $m$), and were able to provide a sub-exponential time $\tfrac{1}{3\sqrt{3}}(1 - m)$-approximation for the problem of optimizing a non-monotone DR-submodular function subject to a general convex set $\cK$. 


The work Durr et al.~\cite{durr2021non} has inspired a new line of research. Th\twodias{\'}{\u}{a}ng \& Srivastav~\cite{thang2021online} showed how to obtain a similar result in an online (regret minimization) setting, and Du et al.~\cite{du2022improved} improved the approximation ratio in the offline setting to $\tfrac{1}{4}(1 - m)$ (while still using sub-exponential time). More recently, the first polynomial time algorithm for the problem, guaranteeing the same approximation ratio as~\cite{du2022improved}, has been provided by Du~\cite{du2022lyapunov}; and subsequently, Mualem \& Feldman~\cite{mualem2023resolving} obtained the same result in the online setting. In their work, Mualem \& Feldman also proved a hardness result, showing that $\tfrac{1}{4}(1 - m)$-approximation is the best approximation ratio that can be obtained in sub-exponential time, which shows that the last algorithms are optimal, and settles the approximability of the problem with respect to this parametrization.

\paragraph{Our contribution}

Every down-closed convex body has $m = 0$. The reverse is not true, but one could hope that convex bodies having $m = 0$ admit as good approximation as down-closed convex bodies. Unfortunately, the hardness result of Mualem \& Feldman disproves this hope since the state-of-the-art approximation ratio for down-closed bodies is $0.401$~\cite{buchbinder2023constrained}. This prompts the question of whether there is a different way to look at the problem (beyond parametrization by $m$) that will provide a smooth interpolation between the approximability obtainable for down-closed and general convex bodies. In this work, we suggest such an interpolation based on a decomposition of the convex body constraint into two distinct convex bodies: a down-closed convex body and a general convex body. Our key results based on this decomposition are summarized as follows.
\begin{itemize}
    \item We provide a novel polynomial time (offline) algorithm for maximizing DR-submodular functions over convex sets given as a composition of a down-closed convex body $\DM$ and a general convex body $\NDM$. Our algorithm always recovers at least the $\tfrac{1}{4}(1 - m)$-approximation of Du~\cite{du2022lyapunov}. However, the approximation guarantee smoothly improves when a significant fraction of the value of the optimal solution belongs to the down-closed part $\DM$ of the decomposition. In particular, when the convex body happens to be entirely down-closed, our algorithm guarantees $e^{-1} \approx 0.367$-approximation, which recovers the approximation ratio for down-closed convex bodies obtained by the Measured Continuous Greedy technique~\cite{bian2017nonmonotone,feldman2011unified}.\footnote{As mentioned above, the state-of-the-art approximation ratio for down-closed convex body constraints is $0.401$. However, all the algorithms for such convex bodies that improve over $e^{-1}$ are built on the Measured Continuous Greedy technique, but add additional components to get the improved approximation ratio, often resulting in impractical algorithms~\cite{buchbinder2019constrained,buchbinder2023constrained,chen2023continuous,ene2016constrained}. For the sake of keeping our algorithms simple and our message clear, we only aimed to recover $e^{-1}$-approximation for down-closed convex bodies. We believe that this ratio can be easily improved by combining our ideas with the components suggested by the above works.}
    \item We provide a novel online (regret minimization) algorithm for the same problem that replicates the guarantees of our offline algorithm, up to a regret term that is proportional to the square root of the number of time steps.
    \item We demonstrate that a decomposition of the kind that we use can be naturally obtained for various machine-learning applications, and use this fact to empirically demonstrate the superiority of our proposed algorithms compared to existing methods across various offline and online applications, namely, offline and online revenue maximization, location summarization and quadratic programming. Our result also have theoretical implications for fairness settings, and we give one example in Appendix~\ref{app:fairness}.
\end{itemize}

\subsection{Related work}
DR-submodular maximization has recently emerged as a key tool in numerous applications within the realms of machine learning and statistics. This surge in relevance has prompted a growing number of studies in this field. In what follows, we provide a brief overview of the main results in this field.

\paragraph{Offline DR-submodular optimization.} 
Bian et al.~\cite{bian2017nonmonotone} initiated the study of maximization of DR-submodular functions subject to down-closed convex sets. Their work showcased the effectiveness of a modified Frank-Wolfe algorithm, grounded in the greedy method introduced by~\cite{calinescu2011maximizing} in the context of set functions, ensuring $(1-\nicefrac{1}{e})$-approximation for the problem when the objective function is guaranteed to be monotone (this is optimal due to~\cite{nemhauser1978best}). Hassani et al.~\cite{hassani2017gradient} subsequently identified a limitation in the algorithm proposed by~\cite{bian2017nonmonotone}, revealing its lack of robustness in stochastic settings where only an unbiased estimator of the gradient is available. To mitigate this limitation, Hassani et al.~\cite{hassani2017gradient} demonstrated that gradient methods exhibit robustness in such scenarios, and still achieves $\nicefrac{1}{2}$-approximation.

When the DR-submodular function is not guaranteed to be monotone, the approximation task becomes notably more challenging. In separate works, Bian et al.~\cite{bian2019optimal} and Niazadeh et al.~\cite{niazadeh2020optimal} introduced distinct algorithms, both ensuring $\nicefrac{1}{2}$-approximation for the maximization of non-monotone DR-submodular functions over a hypercube constraint, which is the best possible~\cite{feige2011maximizing}. It is noteworthy that (one version of) the algorithm proposed by~\cite{niazadeh2020optimal} applies also to (non-DR) submodular functions. As mentioned above, the state-of-the-art approximation for maximizing non-monotone DR-submdoular functions subject to a down-closed convex body constraint is $0.401$~\cite{buchbinder2023constrained}. 

Recently, Pedramfar et al.~\cite{pedramfar2023unified} introduced a unified approach for maximization of continuous DR-submodular functions that encompasses a range of settings and oracle access
types, while Mualem \& Feldman~\cite{mualem2022using} suggested a parameter termed \emph{monotonicity-ratio} allowing for a smooth interpolation between $e^{-1}$-approximation for non-monotone objectives and $(1 - \nicefrac{1}{e})$-approximation for monotone objectives.

\paragraph{Online (regret minimization) DR-submodular optimization.}
Chen et al.~\cite{chen2018online} were the first to address online maximization of monotone DR-submodular functions over a convex set (for monotone objective functions the distinction between down-closed and general convex sets is irrelevant). They introduced two algorithms: one ensuring $(1-\nicefrac{1}{e})$-approximation with approximately $O(\sqrt{T})$-regret, and another algorithm that is resilient to stochastic settings, but guarantees only $\nicefrac{1}{2}$-approximation up to the same regret.
Later, Chen et al.~\cite{chen2019projection} proposed an algorithm that combines $(1-\nicefrac{1}{e})$-approximation with roughly $O(\sqrt{T})$-regret and robustness, and Zhang et al.~\cite{zhang2019online} demonstrated how to reduce the number of gradient calculations per time step to one at the expense of increasing the regret to roughly $O(T^{4/5})$. The last reduction is particularly relevant for bandit versions of the problem (we refer the reader to~\cite{pedramfar2023unified} for a detailed overview of such versions).

In the context of online optimization of DR-submodular functions that are not guaranteed to be monotone, Th\twodias{\'}{\u}{a}ng \& Srivastav~\cite{thang2021online} introduced three distinct algorithms. One of these algorithms is applicable to general convex set constraints, and was later improved over by Mualem and Feldman~\cite{mualem2023resolving}, as discussed above. Another is specifically designed for maximization over the entire hypercube, achieving $\nicefrac{1}{2}$-approximation with approximately $O(\sqrt{T})$-regret. The last algorithm of~\cite{thang2021online} addresses down-closed convex set constraints, and attains $e^{-1}$-approximation with roughly $O(T^{2/3})$-regret. 

\subsection{Paper organization}
In Section~\ref{sec:Preliminaries}, we formally describe the problem we consider and present known properties of DR-submodular functions that we use. Then, in Section~\ref{sec:technique}, we discuss the technique underlying our results. Our offline and online algorithms, which are based on this technique, can be found in Sections~\ref{sec:Offline} and~\ref{sec:Online}, respectively. Finally, Section~\ref{sec:Experiments} compares the empirical performance of our algorithms on multiple machine learning applications with the performance of previously suggested algorithms from the literature.
\section{Preliminaries}\label{sec:Preliminaries}

In this section, we formally present the problem we consider in this paper as well as the notation and known results that we use. We begin with the definition of DR-submodular functions, which are continuous analogs of submodular set functions first defined by Bian et al.~\cite{bian2017guaranteed}. Formally, given a domain $\cX = \prod_{i=1}^n\mathcal{X}_i$, where $\mathcal{X}_i$ is a closed range in $\bR$ for every $i \in [n]$, a function $F \colon \mathcal{X}\rightarrow\mathbb{R}$ is called \emph{DR-submodular} if the inequality
\[
	F(\va+k\ve_i)-F(\va)\geq F(\vb+k\ve_i)-F(\vb)
\]
holds for every two vectors $\va,\vb\in \cX$, positive value $k$ and coordinate $i\in[n]$ obeying $\va\leq \vb$ and $\vb+k\ve_i\in \mathcal{X}$ (here and throughout the paper, $\ve_i$ denotes the standard $i$-th basis vector, and comparison between two vectors should be understood to hold coordinate-wise).
Bian et al.~\cite{bian2017guaranteed} observed that for continuously differentiable functions $F$, the above definition of DR-submodulrity is equivalent to 
\[
    \nabla F(\vx)\leq\nabla F(\vy)\quad \forall\;\vx,\vy\in \cX, \vx\geq\vy
		\enspace,
\]
and for twice differentiable functions $F$, it is equivalent to the Hessian being non-positive at every vector $\vx\in\mathcal{X}$.

In this work, we study the problem of maximizing a non-negative DR-submodular function $F\colon \cX \to \nnR$ subject to a convex body $\cK \subseteq \cX$ constraint. We are interested in the approximation guarantee that can be obtained for this problem based on a particular decomposition of $\cK$ into two other convex bodies: a convex body $\DM$ that is down-closed with respect to $\cX$ and a (not necessary down-closed) convex body $\NDM$. Formally, by saying that $\DM$ and $\NDM$ are a \emph{decomposition} of $\cK$, we mean that $\cK = (\NDM + \DM) \cap \cX$, where $\NDM + \DM \triangleq \{\vy + \vz \mid \vy \in \NDM, \vz \in \DM\}$.

For simplicity, we assume (throughout the paper) that the domain $\mathcal{X}$ of our objective functions is $[0,1]^n$. This assumption is without loss of generality since there is a natural linear mapping from $\mathcal{X}$ to $[0,1]^n$ preserving the above discussed properties of $F$ and $\cK$. Additionally, as is standard in the field, we assume that $F$ is $\beta$-smooth for some parameter $\beta > 0$. A function $F \colon [0, 1]^n \to \bR$ is call \emph{$\beta$-smooth} if it is continuously differentiable, and obeys
\[
    \norm{\grad{F\term{\vx}} -\grad{F\term{\vy}}}{2} \leq \beta \norm{\vx-\vy}{2} \quad \forall\; \vx, \vy \in [0, 1]^n\enspace.
\]
Another standard assumption in the field is that the relevant convex-bodies ($\NDM$ and $\DM$ in our case) are solvable, i.e., that one can efficiently optimize linear functions over them. We take a step further, and assume the ability to optimize linear functions over any convex body defined by the intersection of a polynomial number of linear constraints and constraints requiring particular vectors to belong either to $\NDM$ or to $\DM$. This assumption appeared (often implicitly) in many previous works (see, for example,~\cite{buchbinder2023constrained,ene2016constrained,mualem2023resolving}), and is theoretically justified by the well-known equivalence between separability and solvability.\footnote{Technically, the general equivalence between separability and solvability only guarantees that one can obtain nearly-feasible solutions that nearly-optimize linear functions over convex bodies of the kind described, which would make our algorithms produce solutions that are only nearly-feasible. However, for many cases of interest (for example, when $\NDM$ and $\DM$ are polytopes defined by coefficients with a polynomial size representation) the equivalence between separability and solvability implies the exact statement of our assumption.}

We often refer below to the \emph{diameter} $D$ of $\cK$. This diameter is defined as $D \triangleq \max_{\vx, \vy \in \cK} \|\vx - \vy\|_2$.

\subsection{Known Results and Additional Notation for DR-submodular Functions}

Bian et al.~\cite{bian2017guaranteed} observed that DR-submodular functions are concave along non-negative directions. This implies the following important lemma. The notation $\vzero$ and $\vone$ used in this lemma represents the all-zeros and all-ones vectors, respectively.
\begin{lemma} \label{lem:DR_properties}
Let $F\colon [0, 1]^n \to \nnR$ be a non-negative differentiable DR-submodular function. Then,
\begin{enumerate}
	\item $\inner{\nabla F(\vx)}{\vy} \geq F(\vx+\vy) - F(\vx)$ for every $\vx\in [0, 1]^n$ and $\vy\geq \vzero$ such that $\vx+\vy \leq \vone$. \label{prop:dr_bound2_up}
	\item $\inner{\nabla F(\vx)}{\vy} \leq F(\vx) - F(\vx-\vy)$ for every $\vx\in [0, 1]^n$ and $\vy\geq \vzero$ such that $\vx-\vy \geq \vzero$. \label{prop:dr_bound2_down}
\end{enumerate}
\end{lemma}

Following Buchbinder and Feldman~\cite{buchbinder2023constrained}, we use the following coordinate-wise vector operations. To reduce the number of parentheses necessary, we assume that both these operations have a higher precedence compared to vector addition and subtraction.
\begin{definition}
Given two vectors $\vx, \vy \in [0, 1]^n$,
\begin{itemize}
	\item we denote by $\vx \odot \vy$ their coordinate-wise multiplication (also known as the Hadamard product).
	\item we denote by $\vx \psum \vy$ their coordinate-wise probabilistic sum. In other words, for every $i \in [n]$, $(\vx \psum \vy)_i \triangleq x_i + y_i - x_iy_i = 1 - (1 - x_i)(1 - y_i)$.
\end{itemize}
\end{definition}
As was noted by~\cite{buchbinder2023constrained}, the operation $\psum$ is symmetric and associative. Thus, given vectors $\vx^{(1)}, \vx^{(2)},\allowbreak \dotsc, \vx^{(r)}$, one can define $\PSum_{i = 1}^k  \vx^{(i)} \triangleq \vx^{(1)} \psum \vx^{(2)} \psum \dotso \psum \vx^{(r)}$. Using this notation, it is possible to state the following lemma, which generalizes Lemma~2.3 of~\cite{feige2011maximizing}.
\begin{lemma}[Lemma~4.3 of~\cite{buchbinder2023constrained}] \label{lem:basic_bound_discrete}
Given a DR-submodular function $F\colon [0, 1]^n \to \bR$, integer value $r \geq 1$, vectors $\vx^{(1)}, \vx^{(2)}, \dotsc, \vx^{(r)} \in [0, 1]^n$, and values $p_1, p_2, \dotsc, p_r \in [0, 1]$,
\[
    F\left(\PSum_{i = 1}^r (p_i \cdot \vx^{(i)}) \right)
    \geq
    \sum_{S \subseteq [r]} \left[\prod_{i \in S} p_i \cdot \mspace{-9mu} \prod_{i \in [r] \setminus S} \mspace{-9mu} (1 - p_i) \cdot F\left(\PSum_{i \in S} \vx^{(i)}\right)\right]
    \enspace.
\]
\end{lemma}

One important consequence of the last lemma is given by the next corollary. We note that this corollary can be viewed as an extension of Lemma~2.2 of~\cite{feldman2011unified}.
\begin{corollary} \label{cor:norm_bound}
Given a non-negative DR-submodular function $F\colon [0, 1]^n \to \nnR$ and two vectors $\vx, \vy \in [0, 1]^n$, $F(\vx \psum \vy) \geq (1 - \|\vy\|_\infty) \cdot F(\vx)$.
\end{corollary}
\begin{proof}
If $\|\vy\|_\infty = 0$, then $\vy = \vzero$, which makes the corollary trivial. Otherwise, Lemma~\ref{lem:basic_bound_discrete} and the non-negativity of $F$ imply together that
\[
	F(\vx \psum \vy)
	=
	F\left(\vx \psum \left(\|\vy\|_\infty \cdot \frac{\vy}{\|\vy\|_\infty}\right)\right)
	\geq
	(1 - \|\vy\|_\infty) \cdot F(\vx)
	\enspace.
	\qedhere
\]
\end{proof}

\subsection{Online Optimization} \label{ssc:online}

In the online (regret minimization) version of the problem we consider in this work, there are $L$ time steps.\footnote{The number of time steps is usually denoted by $T$ in the literature. However, we use $L$ in this paper to avoid confusion with the parameter $T$ traditionally used by continuous submodular maximization algorithms, including the algorithms we present in this paper.} In every time step $\ell \in [L]$, the adversary selects a non-negative $\beta$-smooth DR-submodular function $F_t$, and then the algorithm should select a distribution $P_\ell$ of points in $\cK = (\NDM + \DM) \cap [0, 1]^n$ without knowing $F_t$ (the function $F_t$ is revealed to the algorithm only after $P_\ell$ is selected). The objective of the algorithm is to maximize $\sum_{\ell = 1}^L \bE_{\vx \sim P_\ell}[F_\ell(\vx)]$, and its success in doing so is measured compared to the best fixed solution (i.e., any two vectors $\optq \in \NDM$ and $\optp \in \DM$ such that $\optq + \optp \in [0, 1]^n$).

Let us elaborate on the last point. If the functions $F_1, F_2, \dotsc, F_L$ were known upfront, one could execute the offline algorithm we develop to get a set of solutions $\vx^{(1)}, \vx^{(2)}, \dotsc, \vx^{(L)}$ such that $\sum_{\ell = 1}^L F_\ell(\vx^{(\ell)}) \geq \psi(\sum_{\ell = 1}^L F_\ell(\optq + \optp), \sum_{\ell = 1}^L F_\ell(\optp))$ for some function $\psi$ (the structure of the function $\psi$ is determined by the guarantee of Theorem~\ref{thm:Offline} below). Since an online algorithm has to select the output distribution $P_\ell$ before seeing the function $F_\ell$, it can only guarantee
\[
	\sum_{\ell = 1}^L \bE_{\vx \sim P_\ell}[F_\ell(\vx)]
	\geq
	\psi\bigg(\sum_{\ell = 1}^L F_\ell(\optq + \optp), \sum_{\ell = 1}^L F_\ell(\optp)\bigg) - \cR(L)
\]
for some regret function $\cR(L)$. Asymptotically, for our online algorithm, $\cR(L)$ grows like $\sqrt{L}$, and therefore, for large $L$ values, the average guarantee of our online algorithm per function $F_\ell$ approaches the one of our offline algorithm. As usual for online settings, we assume that the range of the functions $F_1, F_2, \dotsc, F_L$ is $[0, 1]$.
\section{Our Technique} \label{sec:technique}

Our algorithms maintain vectors $\vy \in \NDM$ and $\vz \in \DM$. Intuitively, the vector $\vy$ is maintained by the Frank-Wolfe variant developed by Mualem and Feldman~\cite{mualem2023resolving} for non-down-closed polytopes, and the vector $\vz$ is maintained by the continuous-greedy-like variant of Frank-Wolfe developed by Bian~\cite{bian2017nonmonotone} for down-closed polytopes. Combining the two algorithms requires us to solve some technical issues. For example, it is necessary to run the two algorithms in parallel since they both depend on the coordinates of the solution growing at a bounded rate, and it is necessary to create a correlation between the algorithms to guarantee that $\vy + \vz$ remains within $[0, 1]^n$. However, it turns out that the more interesting question is regarding the best way to combine the two vectors $\vy$ and $\vz$ into the output solution of the algorithm.

The most natural approach is to consider the sum $\vy + \vz$ as the output solution. Unfortunately, this does not work well since it results in coordinates of the solution growing too fast. To make this more concrete, we note that our algorithms, as well as the algorithms of~\cite{bian2017nonmonotone} and~\cite{mualem2023resolving}, simulate continuous algorithms working from time $t = 0$ until time $t = 1$. Consider now a particular coordinate $j \in [n]$. Up until time $t \in [0, 1]$, our algorithms spend (up to) $t$ units of ``energy'' on this coordinate. A fraction $x \in [0, t]$ of this ``energy'' is invested in growing $y_j$, and the remaining $t - x$ ``energy'' is invested in growing $z_j$. By the properties of the algorithms of~\cite{bian2017nonmonotone} and~\cite{mualem2023resolving}, this investment of ``energy'' leads to $y_i = 1 - e^{-x}$ and $z_i = 1 - e^{x - t}$, which in the worst case can make $(\vy + \vz)_i$ as large as $2(1 - e^{-t/2})$. To get a better upper bound on the coordinates of the solution, we have to use $\vy \psum \vz$ as the output solution. Note that this choice guarantees that $(\vy \psum \vz)_j \leq 1 - [1 - (1 - e^{-x})] \cdot [1 - (1 - e^{x - t})] = 1 - e^{-t}$, which is always better (for $t > 0$) compared to the bound of $2(1 - e^{-t/2})$ obtained above.

While the use of $\vy \psum \vz$ is useful, it does not come without a cost. As mentioned above, our algorithms simulate continuous algorithms, which is a common practice in the literature about submodular maximization. To discretize these algorithms, one has to split time into steps, and then do in each step a single modification of the vectors $\vy$ and $\vz$ simulating all the modifications done by the continuous algorithm throughout the step. The standard way in which this is done is as follows. Assume that, at the beginning of the step, the continuous algorithm increases $\vy$ at a rate of $\vy'$ and $\vz$ at a rate of $\vz'$, then the discrete algorithm should increase $\vy$ by $\eps \vy'$ and $\vz$ by $\eps \vz'$, where $\eps$ is the size of the step. Unfortunately, this standard practice results in $\vy \psum \vz$ changing by $\eps \vy' \odot (\vone - \vz) + \eps \vz' \odot (\vone - \vy) - \eps^2 \cdot (\vy' \odot \vz')$. To see why this is problematic, note that in the continuous algorithm, when $\vy$ and $\vz$ increase at rates of $\vy'$ and $\vz'$, respectively, $\vy \psum \vz$ increases at a rate of $\vy' \odot (\vone - \vz) + \vz' \odot (\vone - \vy)$. Thus, the term $-\eps^2 \cdot (\vy' \odot \vz')$ from the previous expression represents a new kind of discretization error that we need to handle.

Another hurdle worth mentioning is that the vectors $\vy$ and $\vz$ are updated using two different update rules inherited from the algorithms of~\cite{bian2017nonmonotone} and~\cite{mualem2023resolving}, and the interaction between these update rules results in a guarantee on the output of the algorithm that depends also on the value of $F(\vz)$. Thus, it is necessary to make sure that our algorithms maintain $\vz$ in a way that also guarantees that $F(\vz)$ has a good value. In the first version of our offline algorithm (Section~\ref{ssc:known_optp}), we do that by assuming that we know the value $v$ of the part of the optimal solution that belongs to $\DM$. This knowledge allows us to force the algorithm to increase $\vz$ in a way guaranteed to make $F(\vz)$ competitive with $v$. In the other versions of our offline algorithm (Section~\ref{ssc:unknown_optp} and Appendix~\ref{app:practical_offline}) and in our online algorithm (Section~\ref{sec:Online}), we use a potential function argument to avoid the need to know $v$. This potential function argument is similar to an argument used by Feldman~\cite{feldman2021guess} in a different submodular maximization setting.
\section{Offline Maximization} \label{sec:Offline}

In this section, we present and analyze our offline algorithm, whose guarantee is given by the next theorem. 

\begin{theorem} \label{thm:Offline}
Let $\NDM \subseteq [0,1]^n$ be a general solvable convex set, $\DM \subseteq [0,1]^n$ be a down-closed solvable convex set, and $F\colon [0, 1]^n \rightarrow \nnR$ be a non-negative $\beta$-smooth DR-submodular function. Then, there exists a polynomial time algorithm that, given an error parameter $\eps \in (0, 1)$, outputs vectors $\vw \in (\NDM + \DM) \cap [0, 1]^n$ such that
\begin{align*}
	F(\vw) \geq{} & (1 - m) \cdot \max_{t_s \in [0, 1]} \max_{T \in [t_s, 1]} \bigg\{((T - t_s) e^{-T} - O(\eps)) \cdot F(\optptwo) + \left(\frac{t_s^2\cdot e^{-t_s - T}}{2} - O(\eps)\right) \cdot F(\optpone) \\&\mspace{200mu}+ (e^{-T}-e^{-t_s - T} - O(\eps)) \cdot F(\optq + \optpone)\bigg\} -O(\tfrac{\eps \beta D^2}{1 - m})\enspace,
\end{align*}
where $m = \min_{\vx\in \NDM}\norm{\vx}{}$, $D$ is the diameter of $(\NDM + \DM) \cap [0, 1]^n$, $\optq \in \NDM$ and $\optpone \in \DM$ are any vectors whose sum belongs to $(\NDM + \DM) \cap [0, 1]^n$, and $\optptwo$ is any vector in $\DM$.\footnote{The vectors $\optpone$ and $\optptwo$ can be identical.} 
\end{theorem}

It is interesting to note that Theorem~\ref{thm:Offline} recovers two guarantees of previous works. Specifically, by setting $T=1$, $t_s=0$ and $\NDM = \{\vzero\}$, the theorem implies $e^{-1}$-approximation for maximizing a DR-submodular function subject to a down-closed polytope $\DM$, recovering the result of~\cite{bian2017nonmonotone}. Similarly, by setting $T=t_s=\ln 2$ and $\DM = \{\vzero\}$, Theorem~\ref{thm:Offline} implies $\nicefrac{1}{4}(1 - m)$-approximation for maximizing a DR-submodular function subject to a general polytope $\NDM$, recovering the result of~\cite{mualem2023resolving}.

For ease of the presentation, we present three versions of our offline algorithm. The first version, appearing in Section~\ref{ssc:known_optp}, proves Theorem~\ref{thm:Offline} under the assumption that $F(\optpone)$ is known. We then show, in Section~\ref{ssc:unknown_optp}, a modified algorithm that proves Theorem~\ref{thm:Offline} without making this assumption. The algorithm of Section~\ref{ssc:unknown_optp} is theoretically natural, and is also the base for our online algorithm described in Section~\ref{sec:Online}. However, to get the best results in practice, it is natural to make some modifications to this algorithm (including ones motivated by the work of~\cite{bian2017nonmonotone}), which do not improve the theoretical guarantee of the algorithm and cannot be extended to the online version of the algorithm. The resulting modified algorithm is discussed in Appendix~\ref{app:practical_offline}, and is the main version of our algorithm used by the offline experiments described in Section~\ref{sec:Experiments}.

For ease of the reading, we use $\vo$ below to denote the sum $\optq + \optpone$. We also assume, without loss of generality, that $F(\optptwo) \geq F(\optpone)$. If this inequality is violated, then the guarantee of Theorem~\ref{thm:Offline} follows from the guarantee of the same theorem for the case in which $\optptwo$ is replaced with $\optpone$ (which is a case in which the inequality $F(\optptwo) \geq F(\optpone)$ trivially holds).

\subsection{Algorithm Knowing \texorpdfstring{$F(\optpone)$}{F(\textoptpone)}} \label{ssc:known_optp}

In this section, we give Algorithm~\ref{alg:Offline}, which proves Theorem~\ref{thm:Offline} under the assumption that $F(\optpone)$ is known. In the description of Algorithm~\ref{alg:Offline}, we assume for simplicity that $\eps^{-1}$ is integral. If this is not the case, $\eps$ can be replaced with $\lceil \eps^{-1} \rceil^{-1}$, which is smaller than $\eps$ by at most a factor of $2$.

\begin{algorithm}
\DontPrintSemicolon
Let $\vy^{(0)} \leftarrow \argmin_{\vx\in \NDM}\norm{\vx}{}$ and $\vz^{(0)} \gets \vzero$.\\
\For{$i=1$ \KwTo $\eps^{-1}$}
{
	Solve the following linear program. The variables in this program are the vectors $\va^{(i)}$ and $\vb^{(i)}$.
    \begin{alignat*}{2}
  & \text{maximize}   & \quad & \langle\grad{F(\comb{i-1}) \odot (\vone-\vz^{(i-1)}),\va^{(i)}+\vb^{(i)} \odot (\vone-\vy^{(i - 1)})}\rangle         \nonumber \\
  & \text{subject to} &       & \va^{(i)} \in \NDM \\ \nonumber
  &                   &       & \vb^{(i)} \in  \DM\\ \nonumber
  &                   &       & \langle \vb^{(i)} \odot (\vone-\vz^{(i-1)}), \nabla F(\vz^{(i-1)}) \rangle\geq (1 - \eps)^{i - 1} \cdot F(\optpone)- F(\vz^{(i-1)})\\ \nonumber
	&										&				& \va^{(i)} + \vb^{(i)} \odot (\vone - \vy^{(i - 1)}) \leq \vone
\end{alignat*}\\
    Let $\vy^{(i)}\leftarrow(1-\eps)\cdot\vy^{(i-1)}+\eps\cdot\va^{(i)}$.\\
    Let $\vz^{(i)}\leftarrow\vz^{(i-1)} + \eps\cdot(1-\vz^{(i-1)})\odot\vb^{(i)}$
}
\Return a vector maximizing $F$ among all the vectors in $\{\comb{i} \mid i \in \bZ, 0 \leq i \leq \eps^{-1}\}$.
\caption{\texttt{Frank-Wolfe/Continuous-Greedy Hybrid for Known $F(\optpone)$}\label{alg:Offline}}
\end{algorithm}

It is clear that algorithm Algorithm~\ref{alg:Offline} runs in polynomial time. Therefore, we concentrate on proving that the output vector of Algorithm~\ref{alg:Offline} obeys the properties stated in Theorem~\ref{thm:Offline}. We begin by showing that this vector belongs to $(\NDM + \DM) \cap [0, 1]^n$.
\begin{lemma} \label{lem:membership}
For every integer $0 \leq i \leq \eps^{-1}$, $\vy^{(i)} \in \NDM$ and $\vz^{(i)} \in \eps i \cdot \DM$, where $\eps i \cdot \DM \triangleq \{\eps i \cdot \vx \mid \vx \in \DM\}$. Hence, $\comb{i} \in (\NDM + \DM) \cap [0, 1]^n$.
\end{lemma}
\begin{proof}
We begin the proof by showing that the first part of the lemma implies its second part. Assume that $\vy^{(i)} \in \NDM$ and $\vz^{(i)} \in \eps i \cdot \DM \subseteq \DM$ (the inclusion holds by the down-monotonicity of $\DM$). The definition of $\psum$ guarantees that we always have $\comb{i} \in [0, 1]^n$. Thus, to prove that $\comb{i} \in (\NDM + \DM) \cap [0, 1]^n$, it suffices to show that $\comb{i}$ is the sum of a vector in $\NDM$ and a vector in $\DM$, which is the case since $\comb{i} = \vy^{(i)} + (\vone - \vy^{(i)}) \odot \vz^{(i)}$ and $(\vone - \vy^{(i)}) \odot \vz^{(i)} \in \DM$ by the down-closeness of $\DM$.

In the rest of the proof, we prove the first part of the lemma by induction. The base of the induction holds by the initializations of $\vy^{(0)}$ and $\vz^{(0)}$. Assume now that both $\vy^{(i - 1)} \in \NDM$ and $\vz^{(i - 1)} \in \eps(i - 1) \cdot \DM$ hold for some integer $1 \leq i \leq \eps^{-1}$, and let us prove that we also have $\vy^{(i)} \in \NDM$ and $\vz^{(i)} \in \DM$. For $\vy^{(i)}$ this is true since $\NDM$ is convex and $\vy^{(i)}$ is defined as a convex combination of $\vy^{(i - 1)}$ and $\va^{(i)}$, which are both vectors in $\NDM$. Additionally, since $\vz^{(i-1)} \in \eps(i - 1) \cdot \DM$, there must exist a vector $\vx \in \DM$ such that $\vz^{(i-1)} = \eps(i - 1) \cdot \vx$. Hence,
\begin{align*}
	\vz^{(i)}
	={} &
	\vz^{(i-1)} + \eps\cdot(\vone-\vz^{(i-1)})\odot\vb^{(i)}
	\leq
	\vz^{(i-1)} + \eps \cdot \vb^{(i)}\\
	={} &
	\eps (i - 1) \cdot \vx + \eps \cdot \vb^{(i)}
	=
	\eps i \cdot [(1 - i^{-1}) \cdot \vx + i^{-1} \cdot \vb^{(i)}]
	\in
	\eps i \cdot \DM
	\enspace,
\end{align*}
where the inclusion holds since the convexity of $\DM$ and the fact that both $\vx$ and $\vb^{(i)}$ are vectors in $\DM$ imply together that $(1 - i^{-1}) \cdot \vx + i^{-1} \cdot \vb^{(i)} \in \DM$. Thus, $\vz^{(i)}$ is upper bounded by a vector in $\eps i \cdot \DM$, which implies that $\vz^{(i)}$ itself also belongs to $\eps i \cdot \DM$ because the down-closeness of $\DM$ implies that $\eps i \cdot \DM$ is also down-closed.
\end{proof}

Our next goal is to lower bound the value of the output vector of Algorithm~\ref{alg:Offline}. We begin with the following lemma, which bounds the infinity norm of $\comb{i}$.
\begin{restatable}{lemma}{lemNormXY}\label{lem:normxy}
For every integer $0 \leq i \leq \eps^{-1}$, $\|\vz^{(i)}\|_\infty \leq 1 - (1 - \eps)^i$ and $\|\vy^{(i)}\|_\infty \leq \|\comb{i}\|_\infty \leq 1 - \term{1-\eps}^{i}(1 - m)$.
\end{restatable}
\begin{proof}
We prove the lemma by induction. For $i = 0$, the lemma holds since our choice of values for $\vy^{(0)}$ and $\vz^{(0)}$ guarantees that $\|\vz^{(0)}\|_\infty = \|\vzero\|_\infty = 0$ and
$
	\|\comb{0}\|_\infty
	=
	\|\vy^{(0)}\|_\infty
	=
	m
$. Let us now prove the lemma for $i \geq 1$ assuming it holds for $i - 1$. Note that
\begin{align*}
\vone - \comb{i}
={} &
(\vone-\vy^{(i)}) \odot (\vone-\vz^{(i)})\\
={} & \term{\vone-(1 - \eps) \cdot \vy^{(i-1)}-\varepsilon\va^{(i)}} \odot \term{\vone-\vz^{(i-1)}-\eps(\vone-\vz^{(i-1)}) \odot \vb^{(i)}}\\
\geq{} & (\vone-\vy^{(i-1)}) \odot (\vone-\varepsilon (\vone - \vb^{(i)})) \odot (\vone-\vz^{(i-1)}) \odot (\vone-\varepsilon\vb^{(i)})\\
\geq{} &
(\vone-\vy^{(i-1)}) \odot (\vone-\vz^{(i-1)}) \cdot (1-\varepsilon)
=
(1-\varepsilon) \cdot (\vone - \comb{i - 1})
\enspace,
\end{align*}
where the first inequality uses the fact that the inequality $\va^{(i)} + \vb^{(i)} \odot (\vone - \vy^{(i - 1)}) \leq \vone$ is one of the conditions of the linear program of Algorithm~\ref{alg:Offline}, which implies $\va^{(i)} - \vy^{(i - 1)} \leq (\vone - \vy^{(i - 1)}) \odot (\vone - \vb^{(i)})$. Hence, by the induction hypothesis,
\begin{align*}
	\|\comb{i}\|_\infty
	={} &
	\max_{j \in [n]} (\comb{i})_j
	\leq
	\max_{j \in [n]} [1 - (1-\varepsilon) \cdot (1 - (\comb{i - 1})_j)]\\
	={} &
	1 - (1 - \varepsilon) \cdot (1 - \max_{j \in [n]} (\comb{i - 1})_j)
	=
	1 - (1 - \varepsilon) \cdot (1 - \|\comb{i - 1}\|_\infty)\\
	\leq{} &
	1 - (1 - \varepsilon) \cdot [\term{1-\eps}^{i - 1}(1 - m)]
	=
	1 - \term{1-\eps}^{i}(1 - m)
	\enspace.
\end{align*}
Similarly, the induction hypothesis also implies that
\begin{align*}
	\|\vz^{(i)}\|_\infty
	={} &
	\|\vz^{(i - 1)} + \eps (\vone - \vz^{(i - 1)}) \odot \vb^{(i)}\|_\infty
	\leq
	\|\vz^{(i - 1)} + \eps (\vone - \vz^{(i - 1)})\|_\infty\\
	={} &
	\eps + (1 - \eps) \cdot \|\vz^{(i - 1)}\|_\infty
	\leq
	\eps + (1 - \eps) \cdot (1 - (1 - \eps)^{i - 1})
	=
	1 - (1 - \eps)^i
	\enspace.
	\qedhere
\end{align*}
\end{proof}

Using the last lemma, we prove two lower bounds on the optimal value of the linear program solved by Algorithm~\ref{alg:Offline}. Each one of these lower bounds is based on one possible solution for this linear program. The first such solution is given by the next lemma.

\begin{lemma} \label{lem:feasible_solution}
For every integer $1 \leq i \leq \eps^{-1}$, the assignment $\va^{(i)} = \optq$ and $\vb^{(i)} = \optpone$ is a feasible solution for the linear program solved by Algorithm~\ref{alg:Offline} in iteration number $i$.
\end{lemma}
\begin{proof}
The definitions of $\optq$ and $\optpone$ immediately implies that the first two constraints of the linear program hold for the solution stated in the lemma. Thus, we concentrate on proving that this solution obeys also the other two constraints of the linear program. The third constraint of the linear program is
\[
	\langle \vb^{(i)} \odot (\vone-\vz^{(i-1)}), \nabla F(\vz^{(i-1)}) \rangle\geq (1 - \eps)^{i - 1} \cdot F(\optpone)- F(\vz^{(i-1)})
	\enspace.
\]
To see that this constraint is satisfied by our proposed solution, notice that, by Property~\ref{prop:dr_bound2_up} of Lemma~\ref{lem:DR_properties},
\begin{align*}
	\langle \optpone \odot (\vone-\vz^{(i-1)}), \nabla F(&\vz^{(i-1)}) \rangle
	\geq
	F(\optpone \psum \vz^{(i-1)})- F(\vz^{(i-1)})\\
	\geq{} &
	(1 - \|\vz^{(i - 1)}\|_\infty) \cdot F(\optpone) - F(\vz^{(i-1)})
	\geq
	\term{1-\eps}^{i - 1} \cdot F(\optpone) - F(\vz^{(i-1)})
	\enspace,
\end{align*}
where the second inequality holds by Corollary~\ref{cor:norm_bound}, and the last inequality follows from Lemma~\ref{lem:normxy}. 
The last constraint of the linear program is
\[
	\va^{(i)} + \vb^{(i)} \odot (\vone - \vy^{(i - 1)}) \leq \vone
	\enspace.
\]
This constraint is also satisfied by our proposed solution since 
$
	\optq + \optpone \odot (\vone - \vy^{(i - 1)})
	\leq
	\optq + \optpone
	\leq
	\vone
$.
\end{proof}

As promised, we can now get a lower bound on the optimal value of the linear program solved by Algorithm~\ref{alg:Offline}. Recall that the objective function of this linear program is given (in iteration $i$ of the algorithm) by $\langle\grad{F(\comb{i-1}) \odot (\vone-\vz^{(i-1)}),\va^{(i)}+\vb^{(i)} \odot (\vone-\vy^{(i - 1)})}\rangle$.

\begin{lemma}\label{lemma:right}
For every integer $1 \leq i \leq \eps^{-1}$,
\begin{align*}
	\langle\nabla F(\comb{i-1})& \odot (\vone-\vz^{(i-1)}),\va^{(i)}+\vb^{(i)} \odot (\vone-\vy^{(i - 1)})\rangle\\
	\geq{} &
	(1 - \eps)^{i - 1}(1 - m) \cdot F(\vo) + \term{1-\eps}^{i - 1}(1 - m) \cdot F(\vz^{(i-1)}) \\&+ \langle\nabla F(\comb{i-1}), \vy^{(i - 1)} \odot (\vone-\vz^{(i-1)})\rangle - 2F(\comb{i-1})
	\enspace.
\end{align*}
\end{lemma}
\begin{proof}
Since Lemma~\ref{lem:feasible_solution} guarantees that $\va^{(i)} = \optq$ and $\vb^{(i)} = \optpone$ is one feasible solution for the linear program solved in iteration $i$ of Algorithm~\ref{alg:Offline}, we get
\begin{align} \label{eq:opt_value}
	&
	\langle\nabla F(\comb{i-1}) \odot (\vone-\vz^{(i-1)}),\va^{(i)}+\vb^{(i)} \odot (\vone-\vy^{(i - 1)})\rangle\\\nonumber
	\geq{} &
	\langle\nabla F(\comb{i-1}) \odot (\vone-\vz^{(i-1)}),\optq+\optpone \odot (\vone-\vy^{(i - 1)})\rangle\\\nonumber
	={} &
	\langle\nabla F(\comb{i-1}), \vo \odot (\vone-\vy^{(i - 1)}) \odot (\vone-\vz^{(i-1)})\rangle \\&\mspace{300mu}\nonumber+ \langle\nabla F(\comb{i-1}), \optq \odot \vy^{(i - 1)} \odot (1-\vz^{(i-1)})\rangle
	\enspace.
\end{align}
The first term on the rightmost side of this inequality can be lower bounded, by Property~\ref{prop:dr_bound2_up} of Lemma~\ref{lem:DR_properties}, as follows.
\begin{align*}
	\langle\nabla F(\comb{i-1}), \vo \odot (\vone-\vy^{(i - 1)}) \odot {}&(\vone-\vz^{(i-1)})\rangle
	\geq
	F(\comb{i-1} \psum \vo) - F(\comb{i-1})\\
	\geq{} &
	(1 - \|\comb{i - 1}\|_\infty) \cdot F(\vo) - F(\comb{i-1})\\
	\geq{} &
	(1 - \eps)^{i - 1}(1 - m) \cdot F(\vo) - F(\comb{i-1})
	\enspace,
\end{align*}
where the second inequality holds by Corollary~\ref{cor:norm_bound}, and the last inequality follows from Lemma~\ref{lem:normxy}. 

Next, we need to lower bound the second term on the rightmost side of Inequality~\eqref{eq:opt_value}.
\begin{align*}
	\langle\nabla F(\comb{i-1}), \optq \odot \vy^{(i - 1)} &{}\odot (\vone-\vz^{(i-1)})\rangle - \langle\nabla F(\comb{i-1}), \vy^{(i - 1)} \odot (\vone-\vz^{(i-1)})\rangle\\
	={} &
	- \langle\nabla F(\comb{i-1}), \comb{i-1} - \vz^{(i-1)} \psum (\optq \odot \vy^{(i-1)}) \rangle\\
	\geq{} &
	F(\vz^{(i-1)} \psum (\optq \odot \vy^{(i - 1)})) - F(\comb{i-1})\\
	\geq{} &
	(1-\|\optq \odot \vy^{i - 1}\|_\infty)\cdot F(\vz^{(i-1)}) - F(\comb{i-1})\\
	\geq{} &
	(1-\|\vy^{i - 1}\|_\infty)\cdot F(\vz^{(i-1)}) - F(\comb{i-1})\\
	\geq{} &
	\term{1-\eps}^{i - 1}(1 - m) \cdot F(\vz^{(i-1)}) - F(\comb{i-1})
	\enspace,
\end{align*}
where the first inequality holds by Property~\eqref{prop:dr_bound2_down} of Lemma~\ref{lem:DR_properties} since $\comb{i-1} - \vz^{(i-1)} \psum (\optq \odot \vy^{(i-1)} \geq \comb{i-1} - \vz^{(i-1)} \psum (\vone \odot \vy^{(i-1)}) = \vzero$, the second inequality follows from Corollary~\ref{cor:norm_bound}, and the last inequality holds due to Lemma~\ref{lem:normxy}.
\end{proof}

The lower bound given by the last lemma depends on the term $F(\vz^{(i - 1)})$. Thus, to make this lower bound useful, we need to prove also a lower bound on this term, which is done by the next lemma.
\begin{lemma} \label{lem:z_value}
For every integer $0 \leq i \leq \eps^{-1}$, $F(\vz^{(i)}) \geq \eps i \cdot (1 - \eps)^{i - 1} \cdot F(\optpone) - i \cdot \eps^2 \beta D^2/[2(1 - m)^2]$.
\end{lemma}
\begin{proof}
We prove the lemma by induction on $i$. For $i = 0$, the lemma trivially holds by the non-negativity of $F$. Assume now that the lemma holds for $i - 1$, and let us prove it for $i \geq 1$. By the chain rule,
\begin{align*}
	F(\vz^{(i)})
	={} &
	F(\vz^{(i - 1)} + \eps(\vone - \vz^{(i - 1)}) \odot \vb^{(i)})\\
	={} &
	F(\vz^{(i - 1)}) + \eps \cdot \inner{(\vone - \vz^{(i - 1)}) \odot \vb^{(i)}}{\nabla F(\vz^{(i - 1)})} \\&\mspace{90mu}+ \int_0^\eps \inner{(\vone - \vz^{(i - 1)}) \odot \vb^{(i)}}{\nabla F(\vz^{(i - 1)} + \tau(\vone - \vz^{(i - 1)}) \odot \vb^{(i)}) - \nabla F(\vz^{(i - 1)})} d\tau\\
	\geq{} &
	F(\vz^{(i - 1)}) + \eps \cdot [(1 - \eps)^{i - 1} \cdot F(\optpone)- F(\vz^{(i-1)})] \\&\mspace{50mu} - \int_0^\eps \|(\vone - \vz^{(i - 1)}) \odot \vb^{(i)}\|_2 \cdot \|\nabla F(\vz^{(i - 1)} + \tau(\vone - \vz^{(i - 1)}) \odot \vb^{(i)}) - \nabla F(\vz^{(i - 1)})\|_2 d\tau\\
	\geq{} &
	(1 - \eps) \cdot F(\vz^{(i - 1)}) + \eps(1 - \eps)^{i - 1} \cdot F(\optpone) - \int_0^\eps \tau \cdot \beta \|(\vone - \vz^{(i - 1)}) \odot \vb^{(i)}\|_2^2 d\tau\\
	\geq{}&
	(1 - \eps) \cdot F(\vz^{(i - 1)}) + \eps(1 - \eps)^{i - 1} \cdot F(\optpone) - \eps^2 \beta D^2 / [2(1 - m)^2]
	\enspace,
\end{align*}
where the first inequality follows from the Cauchy–Schwarz inequality and fact that $\vb^{(i)}$ is part of a feasible solution for the linear program that Algorithm~\ref{alg:Offline} solves at iteration $i$, the second inequality holds by the $\beta$-smoothness of $F$, and the last inequality uses the observation that since $\|\vy^{(0)}\|_\infty = m$ and $\DM$ is down-closed, both $\vy^{(0)}$ and $\vy^{(0)} + (1 - m) \cdot \vb^{(i)}$ are vectors in $(\NDM + \DM) \cap [0, 1]^n$, and thus,
\[
	\|(\vone - \vz^{(i - 1)}) \odot \vb^{(i)}\|_2
	\leq
	\|\vb^{(i)}\|_2
	= \frac{\|(\vy^{(0)} + (1 - m) \cdot \vb^{(i)}) - \vy^{(0)}\|_2}{1 - m}
	\leq
	\frac{D}{1 - m}
	\enspace.
\]

Plugging the induction hypothesis into the last inequality yields
\begin{align*}
	F(\vz^{(i)})
	\geq{} &
	(1 - \eps) \cdot [\eps (i - 1) \cdot (1 - \eps)^{i - 2} \cdot F(\optpone)] + \eps(1 - \eps)^{i - 1} \cdot F(\optpone) - i \cdot \eps^2 \beta D^2 /[2(1 - m)^2]\\
	={} &
	\eps i \cdot (1 - \eps)^{i - 1} \cdot F(\optpone)  - i \cdot \eps^2 \beta D^2 / [2(1 - m)^2]
	\enspace.
	\qedhere
\end{align*}
\end{proof}

We now present, in Lemma~\ref{lem:feasible_solution2}, another possible solution for the linear program solved by Algorithm~\ref{alg:Offline}. Corollary~\ref{cor:advance_second_half} then states the lower bound implied by this solution for the optimal value of the objective function of this linear program.
\begin{lemma} \label{lem:feasible_solution2}
For every integer $1 \leq i \leq \eps^{-1}$, the assignment $\va^{(i)} = \vy^{(i - 1)}$ and $\vb^{(i)} = \optptwo$ is a feasible solution for the linear program solved by Algorithm~\ref{alg:Offline} in iteration number $i$.
\end{lemma}
\begin{proof}
Lemma~\ref{lem:membership} shows that $\va^{(i)} = \vy^{(i - 1)} \in \NDM$, and by definition we have $\optptwo \in \DM$. Thus, the first two constraints of the linear program are satisfied by the solution stated in the lemma. 
The third constraint of this linear program is
\[
	\langle \vb^{(i)} \odot (\vone-\vz^{(i-1)}), \nabla F(\vz^{(i-1)}) \rangle\geq (1 - \eps)^{i - 1} \cdot F(\optpone)- F(\vz^{(i-1)})
	\enspace.
\]
Repeating the part of the proof of Lemma~\ref{lem:feasible_solution} related to this constraint, with $\optptwo$ taking the role of $\optpone$, we get
\begin{align*}
	\langle \optptwo \odot (\vone-\vz^{(i-1)}), \nabla F(\vz^{(i-1)}) \rangle
	\geq{} &
	\term{1-\eps}^{i - 1} \cdot F(\optptwo) - F(\vz^{(i-1)})\\
	\geq{} &
	\term{1-\eps}^{i - 1} \cdot F(\optpone) - F(\vz^{(i-1)})
	\enspace,
\end{align*}
where the second inequality holds by our assumption that $F(\optptwo) \geq F(\optpone)$.
Hence, the above stated third constraint is satisfied by our solution, and it only remains to prove that this solution also satisfies the last constraint of the linear program, which is
\[
	\va^{(i)} + \vb^{(i)} \odot (\vone - \vy^{(i - 1)}) \leq \vone
	\enspace.
\]
This is indeed the case since 
$
	\vy^{(i - 1)} + \optptwo \odot (\vone - \vy^{(i - 1)})
	\leq
	\vy^{(i - 1)} + (\vone - \vy^{(i - 1)})
	=
	\vone
$.
\end{proof}
\begin{corollary} \label{cor:advance_second_half}
For every integer $1 \leq i \leq \eps^{-1}$,
\begin{align*}
	\langle\nabla F(\comb{i-1}) \odot (\vone-{}&\vz^{(i-1)}),\va^{(i)}+\vb^{(i)} \odot (\vone-\vy^{(i - 1)})\rangle
	\geq{}
	(1 - \eps)^{i - 1}(1 - m) \cdot F(\optptwo) \\&- F(\comb{i-1}) + \langle\nabla F(\comb{i-1}), \vy^{(i - 1)} \odot (\vone-\vz^{(i-1)})\rangle
	\enspace.
\end{align*}
\end{corollary}
\begin{proof}
Recall that the left hand side of the inequality of the lemma is the objective function of the linear program solved by Algorithm~\ref{alg:Offline} in iteration $i$. Thus, its value is at least the value obtained by plugging in the feasible solution described by Lemma~\ref{lem:feasible_solution2}. Hence,
\begin{align*}
	\langle\nabla F(\comb{i-1})& \odot (\vone-\vz^{(i-1)}),\va^{(i)}+\vb^{(i)} \odot (\vone-\vy^{(i - 1)})\rangle\\
	\geq{} &
	\langle\nabla F(\comb{i-1}) \odot (\vone-\vz^{(i-1)}),\vy^{(i - 1)}+\optptwo \odot (\vone-\vy^{(i - 1)})\rangle\\
	={} &
	\langle\nabla F(\comb{i-1}),\optptwo \odot (\vone-\vy^{(i - 1)}) \odot (\vone-\vz^{(i-1)})\rangle\\
	&\mspace{200mu}+\langle\nabla F(\comb{i-1}), \vy^{(i - 1)} \odot (\vone-\vz^{(i-1)})\rangle
	\enspace.
\end{align*}
To complete the proof of the corollary, it remains to observe that
\begin{align*}
	\langle\nabla F(\comb{i-1}),\optptwo \odot (\vone-\vy^{(i - 1)}&) \odot (\vone-\vz^{(i-1)})\\
	\geq{} &
	F(\comb{i-1} \psum \optptwo) - F(\comb{i-1})\\
	\geq{} &
	(1 - \|\comb{i-1}\|_\infty) \cdot F(\optptwo) - F(\comb{i-1})\\
	\geq{} &
	(1 - \eps)^{i - 1}(1 - m) \cdot F(\optptwo) - F(\comb{i-1})
	\enspace,
\end{align*}
where the first inequality holds by Property~\ref{prop:dr_bound2_up} of Lemma~\ref{lem:DR_properties}, the second inequality holds by Corollary~\ref{cor:norm_bound}, and the last inequality follows from Lemma~\ref{lem:normxy}.
\end{proof}

Using the above results, we can now prove the following lemma about the rate in which the value of $F(y^{(i)})$ increases as a function of $i$.
\begin{lemma}\label{lem:offline}
For every integer $1\leq i\leq \eps^{-1}$, the value of $F(\comb{i})- F(\comb{i-1})$ can be lower bounded by both expressions
\begin{align*}
	\eps(1 - m) \cdot [(1 - \eps)^{i} \cdot{}& F(\vo) + \eps (1-\eps)^{2i - 2} (i - 1) \cdot F(\optpone)] \\&- 2\eps \cdot F(\comb{i-1}) - 3\eps^2 \cdot F(\comb{i-1}) - O(\tfrac{\eps^2\beta D^2}{1 - m})
	\enspace,
\end{align*}
and
\[
	\eps(1 - m) \cdot (1 - \eps)^{i} \cdot F(\optptwo) - \eps \cdot F(\comb{i-1}) - 3\eps^2 \cdot F(\comb{i-1}) - O(\tfrac{\eps^2\beta D^2}{1 - m})
	\enspace.
\]
\end{lemma}
\begin{proof} By the chain rule,
\begin{align} \label{eq:increase_comb}
F(\comb{i})&{}-F(\comb{i-1})\\\nonumber
={} &
F((\term{1-\eps}\vy^{(i-1)}+\eps\va^{(i)}) \psum (\vz^{(i-1)}+\eps(\vone-\vz^{(i-1)}) \odot \vb^{(i)}))-F(\comb{i-1})\\\nonumber
={} &
\int_0^\eps \Big\langle \frac{d[((1-\tau)\vy^{(i-1)}+\tau\va^{(i)}) \psum (\vz^{(i-1)}+\tau(\vone-\vz^{(i-1)}) \odot \vb^{(i)})]}{d\tau}, \\\nonumber&\mspace{150mu} \nabla F(((1-\tau)\vy^{(i-1)}+\tau\va^{(i)}) \psum (\vz^{(i-1)}+\tau(\vone-\vz^{(i-1)}) \odot \vb^{(i)}))\Big\rangle d\tau\\\nonumber
={} &
\int_0^\eps \langle(\vone - \vz^{(i - 1)}) \odot [(\vone - \vy^{(i - 1)}) \odot \vb^{(i)} + (\va^{(i)} - \vy^{(i - 1)}) \odot (\vone - 2\tau \cdot \vb^{(i)})], \\[-2mm]\nonumber&\mspace{150mu} \nabla F(((1-\tau)\vy^{(i-1)}+\tau\va^{(i)}) \psum (\vz^{(i-1)}+\tau(\vone-\vz^{(i-1)}) \odot \vb^{(i)}))\rangle d\tau
\enspace.
\end{align}
At this point, we need to lower bound the integrand on the rightmost side of the last equality. The first step towards obtaining this lower bound is the following inequality.
\begin{align} \label{eq:part1}
&
\langle(\vone - \vz^{(i - 1)}) \odot [(\vone - \vy^{(i - 1)}) \odot \vb^{(i)} + (\va^{(i)} - \vy^{(i - 1)}) \odot (\vone - 2\tau \cdot \vb^{(i)})], \\&\mspace{100mu}\nonumber \nabla F(((1-\tau)\vy^{(i-1)}+\tau\va^{(i)}) \psum (\vz^{(i-1)}+\tau(\vone-\vz^{(i-1)}) \odot \vb^{(i)})) - \nabla F(\comb{i - 1})\rangle\\\nonumber
\geq{}&
-\|(\vone - \vz^{(i - 1)}) \odot [(\vone - \vy^{(i - 1)}) \odot \vb^{(i)} + (\va^{(i)} - \vy^{(i - 1)}) \odot (\vone - 2\tau \cdot \vb^{(i)})]\|_2 \cdot \\&\mspace{80mu}\nonumber \|\nabla F(((1-\tau)\vy^{(i-1)}+\tau\va^{(i)}) \psum (\vz^{(i-1)}+\tau(\vone-\vz^{(i-1)}) \odot \vb^{(i)})) - \nabla F(\comb{i - 1})\|_2\\\nonumber
\geq{} &
-3D \cdot \beta \|(((1-\tau)\vy^{(i-1)}+\tau\va^{(i)}) \psum (\vz^{(i-1)}+\tau(\vone-\vz^{(i-1)}) \odot \vb^{(i)})) - (\comb{i - 1})\|_2\\\nonumber
={} &
-3\beta D \cdot \|\tau(\va^{(i)} - \vy^{(i - 1)}) \odot (\vone - \vz^{(i - 1)}) + \tau(\vone-\vz^{(i-1)}) \odot \vb^{(i)} \odot (\vone - \vy^{(i - 1)}) \\\nonumber&\mspace{300mu}- \tau^2(\vone-\vz^{(i-1)}) \odot (\va^{(i)} - \vy^{(i - 1)}) \odot \vb^{(i)} \|_2
\geq
-6\tau \beta D^2
\enspace,
\end{align}
where the first inequality follows from the Cauchy-Schwarz inequality, and the second inequality uses the $\beta$-smoothness of $F$ and the observation that since $\va^{(i)} + (\vone - \vy^{(i - 1)}) \odot \vb^{(i)}$, $\va^{(i)}$ and $\vy^{(i - 1)}$ are all vectors in $(\NDM + \DM) \cap [0, 1]^n$, it holds that
\begin{align*}
	\|(\vone - \vz^{(i - 1)}) \odot [(\vone - \vy^{(i - 1)}) \odot{}& \vb^{(i)} + (\va^{(i)} - \vy^{(i - 1)}) \odot (\vone - 2\tau \cdot \vb^{(i)})]\|_2\\
	\leq{} &
	\|(\va^{(i)} + (\vone - \vy^{(i - 1)}) \odot \vb^{(i)}) - \vy^{(i - 1)}\|_2 + 2\tau \cdot \|(\va^{(i)} - \vy^{(i - 1)}) \odot \vb^{(i)}\|_2\\
	\leq{} &
	D + 2\tau \cdot \|\va^{(i)} - \vy^{(i - 1)}\|_2
	\leq
	3D
	\enspace.
\end{align*}
Similarly, the last inequality of Inequality~\eqref{eq:part1} holds since
\begin{align*}
	\|\tau(\va^{(i)} - \vy^{(i - 1)}&) \odot (\vone - \vz^{(i - 1)}) + \tau(\vone-\vz^{(i-1)}) \odot \vb^{(i)} \odot (\vone - \vy^{(i - 1)}) \\\nonumber&\mspace{300mu}- \tau^2(\vone-\vz^{(i-1)}) \odot (\va^{(i)} - \vy^{(i - 1)}) \odot \vb^{(i)} \|_2\\
	={} &
	\|\tau(\va^{(i)} - \vy^{(i - 1)}) \odot (\vone - \vz^{(i - 1)}) \odot (1 - \tau \vb^{(i)}) + \tau(\vone-\vz^{(i-1)}) \odot \vb^{(i)} \odot (\vone - \vy^{(i - 1)})\|_2\\
	\leq{} &
	\|\va^{(i)} - \vy^{(i - 1)}\|_2 + \|(\vy^{(i - 1)} + (\vone-\vy^{(i-1)}) \odot \vb^{(i)}) - \vy^{(i - 1)}\|_2
	\leq
	2D
	\enspace.
\end{align*}

We now need another inequality.
{\allowdisplaybreaks\begin{align*}
	\langle(\vone - \vz^{(i - 1)}&) \odot [(\vone - \vy^{(i - 1)}) \odot \vb^{(i)} + (\va^{(i)} - \vy^{(i - 1)}) \odot (\vone - 2\tau \cdot \vb^{(i)})], \nabla F(\comb{i - 1})\rangle\\
	={} &
	(1 - 2\tau) \cdot \{\langle (\vone - \vy^{(i - 1)}) \odot \vb^{(i)} + \va^{(i)}, \nabla F(\comb{i - 1}) \odot (\vone - \vz^{(i - 1)})\rangle \\&- \langle \vy^{(i - 1)} \odot (\vone - \vz^{(i - 1)}), \nabla F(\comb{i - 1}) \rangle\} \\
	&+2\tau \cdot \langle(\vone - \vz^{(i - 1)}) \odot (\vone - \vy^{(i - 1)}) \odot \vb^{(i)}, \nabla F(\comb{i - 1})\rangle\\
	&+ 2\tau \cdot \langle (\vone - \vz^{(i - 1)}) \odot \va^{(i)} \odot (\vone - \vy^{(i - 1)}) \odot (\vone - \vb^{(i)}), \nabla F(\comb{i - 1})\rangle\\
	&- 2\tau \cdot \langle (\vone - \vz^{(i - 1)}) \odot \vy^{(i - 1)} \odot (\vone - \va^{(i)}) \odot (\vone - \vb^{(i)}), \nabla F(\comb{i - 1})\rangle\\
	\geq{} &
	(1 - 2\tau) \cdot \{\langle (\vone - \vy^{(i - 1)}) \odot \vb^{(i)} + \va^{(i)}, \nabla F(\comb{i - 1}) \odot (\vone - \vz^{(i - 1)})\rangle \\&- \langle \vy^{(i - 1)} \odot (\vone - \vz^{(i - 1)}), \nabla F(\comb{i - 1}) \rangle\} \\&+ 2\tau \cdot \{F(\vy^{(i - 1)} \psum \vz^{(i - 1)} \psum \vb^{(i)}) + F(\vy^{i-1} \psum \vz^{(i - 1)} \psum ((\vone - \vb^{(i)}) \odot \va^{(i)})) \\&+ F(((\va^{(i)} \psum \vb^{(i)}) \odot \vy^{(i - 1)}) \psum \vz^{(i - 1)}) - 3F(\comb{i - 1})\}\\
	\geq{} &
	(1 - 2\tau) \cdot \max\{(1 - m) \cdot [(1 - \eps)^{i - 1} F(\vo) + \term{1-\eps}^{2i - 3} \cdot \eps (i - 1) \cdot F(\optpone)] - O(\tfrac{\eps \beta D^2}{1 - m}) \\&\mspace{100mu}- 2F(\comb{i-1}), (1 - \eps)^{i - 1}(1 - m) \cdot F(\optptwo) - F(\comb{i-1})\} \\&\mspace{100mu}-6\tau \cdot F(\comb{i - 1})
	\enspace,
\end{align*}}%
where the second inequality holds by Properties~\ref{prop:dr_bound2_up} and~\ref{prop:dr_bound2_down} of Lemma~\ref{lem:DR_properties}, and the last inequality follows from Lemmata~\ref{lemma:right} and~\ref{lem:z_value}, Corollary~\ref{cor:advance_second_half} and the non-negativity of $F$.

Adding the last inequality to Inequality~\eqref{eq:part1}, we get the promised lower bound on the integrand on the rightmost side of Equality~\eqref{eq:increase_comb}; and plugging this lower bound into the equality yields
\begin{align*}
F(\comb{i})- F&(\comb{i-1})\\
\geq{} & (\eps - \eps^2) \cdot \max\{(1 - m) \cdot [(1 - \eps)^{i - 1} F(\vo) + \term{1-\eps}^{2i - 3} \cdot \eps (i - 1) \cdot F(\optpone)] \\&\mspace{10mu}  - 2F(\comb{i-1}), (1 - \eps)^{i - 1}(1 - m) \cdot F(\optptwo) - F(\comb{i-1})\} \\&- 3\eps^2 \cdot F(\comb{i - 1}) - O(\tfrac{\eps^2\beta D^2}{1 - m})
\enspace.
\qedhere
\end{align*}
\end{proof}

The last lemma implies recursive lower bounds on the values of $\{F(\comb{i}) \mid i \in \bZ, 0 \leq i \leq \eps^{-1}\}$. We need to get closed expression forms for these lower bounds, and the next two lemmata are steps towards this goal. To simplify the statements of these lemmata, it is useful to define $\alpha \triangleq (1 - m)(1 - \eps)^2$ and $g(t) \triangleq \term{e^{-t}-e^{-2t}} \cdot F(\vo) +\frac{t^2\cdot e^{-2t}}{2}\cdot F(\optpone)$.
\begin{lemma} \label{lem:bound1}
    For every integer $0\leq i\leq \eps^{-1}$, $F(\comb{i}) \geq \alpha \cdot g(\eps i) - 3\eps^2 \cdot \sum_{i' = 1}^{i} F(\comb{i' - 1}) - i \cdot O(\eps^2 (\alpha \cdot F(\vo) + \alpha \cdot F(\optpone) + \tfrac{\beta D^2}{1 - m}))$.
\end{lemma}
\begin{proof}
We prove the lemma by induction on $i$. For $i = 0$, the lemma immediately follows from the non-negativity of $F$ since $g(0) = 0$. Assume now that the lemma holds for $i - 1$, and let us prove it for $i \geq 1$.

Observe that
\[
    g'(t) = (2e^{-2t} - e^{-t}) \cdot F(\vo) + te^{-2t}(1 - t) \cdot F(\optpone)
    \enspace;
\]
and therefore,
    \begin{align} \label{eq:g_diff}
        g(\eps i)-{}&g(\eps(i-1))=\int_{\varepsilon(i-1)}^{\varepsilon i}[(2e^{-2t}-e^{-t}) \cdot F(\vo)+te^{-2t}\term{1-t} \cdot F(\optpone)]dt\\\nonumber
        \leq{} &
        \int_{\eps(i-1)}^{\eps i} [(2e^{-2\eps (i - 1)}-e^{-\eps(i - 1)}) \cdot F(\vo) +\eps(i - 1)e^{-2\eps(i - 1)}\term{1-\eps(i - 1)} \cdot F(\optpone) \\\nonumber&\mspace{500mu}+ O(\eps(F(\vo) + F(\optpone)))]dt\\\nonumber
        ={} &
        \eps[(2e^{-2\eps (i - 1)}-e^{-\eps(i - 1)}) \cdot F(\vo)+\eps(i - 1)e^{-2\eps(i - 1)}\term{1-\eps(i - 1)} \cdot F(\optpone)] \\\nonumber&\mspace{500mu}+ O(\eps^2(F(\vo) + F(\optpone)))\\\nonumber
        ={} &
        \eps[e^{-\eps(i - 1)} \cdot F(\vo) + \eps(i - 1)e^{-2\eps(i - 1)} \cdot F(\optpone) - 2g(\eps(i - 1))] + O(\eps^2(F(\vo) + F(\optpone)))
        \enspace,
    \end{align}
where the inequality holds since, for every $t \in [\eps (i - 1), \eps i]$,
\begin{align*}
    (2e^{-2t}-e^{-t}) \cdot F(\vo&) +te^{-2t}\term{1-t} \cdot F(\optpone) \\
    &\mspace{50mu}- (2e^{-2\eps (i - 1)}-e^{-\eps(i - 1)}) \cdot F(\vo) - \eps(i - 1)e^{-2\eps(i - 1)}\term{1-\eps(i - 1)} \cdot F(\optpone)\\
    ={} &
    \int_{\eps(i - 1)}^t [(e^{-\tau} - 4e^{-2\tau}) \cdot F(\vo) + e^{-2\tau}((1 - \tau) - 2\tau(1 - \tau) - \tau)\cdot F(\optpone)] d\tau\\
    ={} &
    \int_{\eps(i - 1)}^t [(e^{-\tau} - 4e^{-2\tau}) \cdot F(\vo) + e^{-2\tau}(1 - 4\tau + 2\tau^2) \cdot F(\optpone)] d\tau\\
    \leq{} &
    \int_{\eps(i - 1)}^t \left(\frac{F(\vo)}{16} + F(\optpone)\right)d\tau
    =
    O(\eps (F(\vo) + F(\optpone)))
    \enspace.
\end{align*}
Rearranging Inequality~\eqref{eq:g_diff} yields
\begin{equation} \label{eq:g_lower_bound}
    (1 - 2\eps) \cdot g(\eps(i - 1)) + \eps [e^{-\eps(i - 1)} \cdot F(\vo) + \eps(i - 1)e^{-2\eps(i - 1)} \cdot F(\optpone)]
    \geq
    g(\eps i) - O(\eps^2(F(\vo) + F(\optpone)))
    \enspace.
\end{equation}

Now, we can use Lemma~\ref{lem:offline} (and the non-negativity of $F$) to get
{\allowdisplaybreaks\begin{align*}
    F(\comb{i})
		\geq{} &
		(1 - 2\eps - 3\eps^2) \cdot F(\comb{i-1}) \\&\mspace{100mu}+ \eps\alpha [(1 - \eps)^{i - 2} \cdot F(\vo) + \eps \term{1-\eps}^{2i - 4} (i - 1) \cdot F(\optpone)] - O(\tfrac{\eps^2\beta D^2}{1 - m})\\
		\geq{} &
		(1 - 2\eps) \cdot \Big[\alpha \cdot g(\eps (i - 1)) - 3\eps^2 \cdot \sum_{i' = 1}^{i - 2} F(\comb{i'}) \\&- (i - 1) \cdot O(\eps^2 (\alpha \cdot F(\vo) + \alpha \cdot F(\optpone) + \tfrac{\beta D^2}{1 - m}))\Big] - 3\eps^2 \cdot F(\comb{i - 1}) \\& + \eps\alpha [e^{-\eps(i - 1)} \cdot F(\vo) + \eps e^{-2\eps(i - 1)} (i - 1) \cdot F(\optpone)] - O(\tfrac{\eps^2\beta D^2}{1 - m})\\
    \geq{} &
    \alpha \cdot g(\eps i) - 3\eps^2 \cdot \sum_{i' = 1}^{i - 1} F(\comb{i'}) - i \cdot O(\eps^2 (\alpha \cdot F(\vo) + \alpha \cdot F(\optpone) + \tfrac{\beta D^2}{1 - m})
    \enspace,
\end{align*}}%
where the second inequality holds by the induction hypothesis, and the last inequality follows from Inequality~\eqref{eq:g_lower_bound}.
\end{proof}

\begin{lemma} \label{lem:bound2}
For every two integers $0 \leq i_s \leq i \leq \eps^{-1}$, $F(\comb{i}) \geq \alpha \eps (i - i_s)(1 - \eps)^{i - 2} \cdot F(\optptwo) + (1 - \eps)^{i - i_s} \cdot F(\comb{i_s}) - 3\eps^2 \cdot \sum_{i' = i_s + 1}^{i} F(\comb{i' - 1}) - (i - i_s) \cdot O(\tfrac{\eps^2\beta D^2}{1 - m})$.
\end{lemma}
\begin{proof}
We prove the lemma by induction on $i$. For $i = i_s$ the lemma is trivially true. Thus, assume that the lemma holds for $i - 1$, and let us prove it for $i > i_s$. By Lemma~\ref{lem:offline},
{\allowdisplaybreaks\begin{align*}
	F(\comb{i})
	\geq{} &
	(1 - \eps) \cdot F(\comb{i-1}) + \eps(1 - m) \cdot (1 - \eps)^{i} \cdot F(\optptwo) \\&\mspace{300mu}- 3\eps^2 \cdot F(\comb{i - 1}) - O(\tfrac{\eps^2\beta D^2}{1 - m})\\
	\geq{} &
	(1 - \eps) \cdot \Big[\alpha \eps (i - 1 - i_s)(1 - \eps)^{i - 3} \cdot F(\optptwo) + (1 - \eps)^{i - i_s - 1} \cdot F(\comb{i_s}) \\&\mspace{100mu} - \sum_{i' = i_s + 1}^{i - 1} F(\comb{i' - 1})- (i - i_s - 1) \cdot O(\tfrac{\eps^2\beta D^2}{1 - m})\Big] \\&\mspace{100mu}+ \alpha\eps \cdot (1 - \eps)^{i - 2} \cdot F(\optptwo) - 3\eps^2 \cdot F(\comb{i - 1}) - O(\tfrac{\eps^2\beta D^2}{1 - m})\\
	\geq{} &
	\alpha \eps (i - i_s)(1 - \eps)^{i - 2} \cdot F(\optptwo) + (1 - \eps)^{i - i_s} \cdot F(\comb{i_s}) \\&\mspace{100mu}- 3\eps \cdot \sum_{i' = i_s + 1}^{i} F(\comb{i' - 1}) - (i - i_s) \cdot O(\tfrac{\eps^2\beta D^2}{1 - m})
	\enspace,
\end{align*}}%
where the second inequality holds by the induction hypothesis.
\end{proof}


We are ready now to prove the next corollary, which completes the proof of Theorem~\ref{thm:Offline}. Recall that the value of the output vector $\vw$ of Algorithm~\ref{alg:Offline} is $\max_{0 \leq i \leq \eps^{-1}} F(\comb{i})$.
\begin{corollary}
It holds that
\begin{align*}
	(1 + O(\eps)) \cdot {}&\max_{0 \leq i \leq \eps^{-1}T} F(\comb{i})
	\geq
	\alpha \cdot \max_{t_s \in [0, 1]} \max_{T \in [t_s, 1]} \bigg\{((T - t_s) e^{-T} - O(\eps)) \cdot F(\optptwo) \\& + \left[\frac{t_s^2\cdot e^{-t_s - T}}{2} - O(\eps)\right] \cdot F(\optpone) + (e^{-T}-e^{-t_s - T} - O(\eps)) \cdot F(\vo)\bigg\} -O(\tfrac{\eps \beta D^2}{1 - m})
	\enspace.
\end{align*}
\end{corollary}
\begin{proof}
Combining Lemmata~\ref{lem:bound1} and~\ref{lem:bound2}, we get, for every two integers $0 \leq i_s \leq i_T \leq \eps^{-1}$,
{\allowdisplaybreaks\begin{align*}
	F(\comb{i_T}&)
	\geq
	\alpha \eps (i_T - i_s)(1 - \eps)^{i_T - 2} \cdot F(\optptwo) - (i_T - i_s) \cdot O(\tfrac{\eps^2\beta D^2}{1 - m}) \\&\mspace{30mu} + (1 - \eps)^{i_T - i_s} \cdot \Big[\alpha \cdot g(\eps i_s) - i_s \cdot O(\eps^2 (\alpha \cdot F(\vo) + \alpha \cdot F(\optpone) + \tfrac{\beta D^2}{1 - m})) \\&\mspace{30mu}- 3\eps^2 \cdot \sum_{i' = 1}^{i_s} F(\comb{i' - 1}) \Big] - 3\eps^2 \cdot \sum_{i' = i_s + 1}^{i_T} F(\comb{i' - 1})\\
	\geq{} &
	\alpha \eps (i_T - i_s)(1 - \eps)^{i_T - 2} \cdot F(\optptwo) - O(\eps (\alpha \cdot F(\vo) + \alpha \cdot F(\optpone) + \tfrac{\beta D^2}{1 - m})) \\&\mspace{30mu} + (1 - \eps)^{i_T - i_s} \cdot \alpha\left[(e^{-\eps i_s}-e^{-2\eps i_s}) \cdot F(\vo) +\frac{(\eps i_s)^2\cdot e^{-2\eps i_s}}{2}\cdot F(\optpone)\right]\\
	&\mspace{30mu}-3\eps^2 \cdot \sum_{i' = 1}^{i_T} F(\comb{i' - 1})\\
	\geq{} &
	\alpha\eps (i_T - i_s) e^{-\eps i_T} \cdot F(\optptwo) + \alpha\left[\frac{(\eps i_s)^2\cdot e^{-\eps i_s - \eps i_T}}{2} - O(\eps)\right] \cdot F(\optpone) \\&\mspace{8mu}+ \alpha(e^{-\eps i_T}-e^{-\eps i_s - \eps i_T} - O(\eps)) \cdot F(\vo) - 3\eps \cdot \mspace{-18mu} \max_{1 \leq i' \leq i_T} F(\comb{i' - 1}) -O(\tfrac{\eps \beta D^2}{1 - m})
	\enspace.
\end{align*}}%

This implies that, for every $t_s \in [0, 1]$ and $T \in [t_s, 1]$,
\begin{align*}
	&(1 + 3\eps) \cdot \max_{0 \leq i \leq \eps^{-1}} F(\comb{i})
	\geq 
	F(\comb{\lfloor \eps^{-1}T \rfloor}) + 3\eps \cdot \mspace{-18mu} \max_{1 \leq i' \leq \lfloor \eps^{-1}T \rfloor} \comb{i' - 1}\\
	\geq{} &
	\alpha\eps(\lfloor \eps^{-1}T \rfloor - \lfloor \eps^{-1}t_s\rfloor) e^{-\eps\lfloor \eps^{-1}T \rfloor} \cdot F(\optptwo) + \alpha\left[\frac{(\eps \lfloor \eps^{-1}t_s\rfloor)^2\cdot e^{-\eps \lfloor \eps^{-1}t_s\rfloor - \eps\lfloor \eps^{-1}T \rfloor}}{2} - O(\eps)\right] \cdot F(\optpone) \\&\mspace{200mu}+ \alpha(e^{-\eps\lfloor \eps^{-1}T \rfloor}-e^{-\eps \lfloor \eps^{-1}t_s\rfloor - \eps \lfloor \eps^{-1}T \rfloor} - O(\eps)) \cdot F(\vo) -O(\tfrac{\eps \beta D^2}{1 - m})\\
	\geq{} &
	\alpha(T - t_s - \eps) e^{-T} \cdot F(\optptwo) + \alpha\left[\frac{t_s^2\cdot e^{-t_s - T}}{2} - O(\eps)\right] \cdot F(\optpone) \\&\mspace{200mu}+ \alpha(e^{-T}-e^{-t_s - T} - O(\eps)) \cdot F(\vo) -O(\tfrac{\eps \beta D^2}{1 - m})
	\enspace.&
	\qedhere
\end{align*}
\end{proof}

\subsection{Guess-Free Offline Maximization}  \label{ssc:unknown_optp}

In this section, we reprove Theorem~\ref{thm:Offline} without assuming knowledge of $F(\optpone)$ like in Section~\ref{ssc:known_optp}. Formally, we prove in this section the following proposition.

\begin{proposition} \label{prop:Offline_ts}
Let $\NDM \subseteq [0,1]^n$ be a general solvable convex set, $\DM \subseteq [0,1]^n$ be a down-closed solvable convex set, and $F\colon [0, 1]^n \rightarrow \nnR$ be a non-negative $\beta$-smooth DR-submodular function. Then, there exists a polynomial time algorithm that, given a time $t_s \in [0, 1]$ and an error parameter $\eps \in (0, 1)$, outputs vector $\vw \in (\NDM + \DM) \cap [0, 1]^n$ such that
\begin{align*}
	F(\vw) \geq{} & (1 - m) \cdot \max_{T \in [t_s, 1]} \bigg\{((T - t_s) e^{-T} - O(\eps)) \cdot F(\optptwo) + \left(\frac{t_s^2\cdot e^{-t_s - T}}{2} - O(\eps)\right) \cdot F(\optpone) \\&\mspace{200mu}+ (e^{-T}-e^{-t_s - T} - O(\eps)) \cdot F(\optq + \optpone)\bigg\} -O(\tfrac{\eps \beta D^2}{1 - m})\enspace,
\end{align*}
where $m = \min_{\vx\in \NDM}\norm{\vx}{}$, $D$ is the diameter of $(\NDM + \DM) \cap [0, 1]^n$, $\optq \in \NDM$ and $\optpone \in \DM$ are any vectors whose sum belongs to $(\NDM + \DM) \cap [0, 1]^n$, and $\optptwo \in \DM$.
\end{proposition}

To see why Proposition~\ref{prop:Offline_ts} implies Theorem~\ref{thm:Offline}, we observe that, by executing the algorithm from this proposition for every $t_s \in \{\eps i \mid i \in \bZ, 0 \leq i \leq \eps^{-1}\}$, and then outputting the best solution obtained, one can get (in polynomial time) a vector $\vw$ such that
\begin{align*}
	F(\vw) \geq{} & (1 - m) \cdot \max_{t_s \in \{\eps i \mid i \in \bZ, 0 \leq i \leq \eps^{-1}\}} \max_{T \in [t_s, 1]} \bigg\{((T - t_s) e^{-T} - O(\eps)) \cdot F(\optptwo) \\&\mspace{8mu}+ \left(\frac{t_s^2\cdot e^{-t_s - T}}{2} - O(\eps)\right) \cdot F(\optpone) + (e^{-T}-e^{-t_s - T} - O(\eps)) \cdot F(\optq + \optpone)\bigg\} -O(\tfrac{\eps \beta D^2}{1 - m})\enspace.
\end{align*}
Since the coefficients of $F(\optpone)$, $F(\optptwo)$ and $F(\optq + \optpone)$ in the right side of the last inequality all change by at most $O(\eps)$ when $t_s$ and $T$ change by at most $\eps$, this right side is equivalent to the right side of the inequality stated in Theorem~\ref{thm:Offline}. Thus, the vector $\vw$ obeys the properties guaranteed by this theorem.

The proof of Proposition~\ref{prop:Offline_ts} is based on Algorithm~\ref{alg:OfflineGF}. The pseudocode of this algorithm makes, without loss of generality, two assumptions, which we list below.
\begin{itemize}
	\item The pseudocode assumes that $\eps^{-1}$ is integral and $\eps \leq 1/30$. If this is not the case, then we can replace $\eps$ with $1 / \lceil 30\eps^{-1} \rceil$, which decreases $\eps$ by most a constant factor.
	\item The pseudocode assumes that $\eps^{-1} t_s$ is integral. If this is not the case, we reduce $t_s$ by at most $\eps$ to $\eps\lfloor \eps^{-1}t_s \rfloor$. Notice that such a reduction again affects the right hand side of the inequality stated in Proposition~\ref{prop:Offline_ts} only by modifying the hidden constants within the big $O$ notation.
\end{itemize}

\begin{algorithm}[ht]
\DontPrintSemicolon
Let $\vy^{(0)} \leftarrow \argmin_{\vx\in \NDM}\norm{\vx}{}$, $\vz^{(0)} \gets \vzero$ and $m \gets \|\vy^{(0)}\|_\infty$.\\
\For{$i=1$ \KwTo $\eps^{-1}$}
{
	\If{$i \leq \eps^{-1} t_s$}
	{
	Solve the following linear program. The variables in this program are the vectors $\va^{(i)}$ and $\vb^{(i)}$.
    \begin{alignat*}{2}
  & \text{maximize}   & \quad & e^{2\eps i} \cdot \inner{\nabla F(\comb{i-1}) \odot (\vone-\vz^{(i - 1)})}{(\vone-\vy^{(i - 1)})\odot \vb^{(i)} + \va^{(i)}} \\
	&										&& \mspace{140mu}+(1 - m) \cdot e^{\eps i}(t_s - \eps i) \cdot \inner{\nabla F(\vz^{(i-1)}) \odot (1-\vz^{(i-1)})}{ \vb^{(i)}}         \nonumber \\
  & \text{subject to} &       & \va^{(i)} \in \NDM \\ \nonumber
  &                   &       & \vb^{(i)} \in  \DM\\ \nonumber
&                   &       & \va^{(i)}+\vb^{(i)}\in [0, 1]^n \nonumber
\end{alignat*}
\vspace{-5mm}}
\Else
{

Solve the following linear program. The variables in this program are the vectors $\va^{(i)}$ and $\vb^{(i)}$.
    \begin{alignat*}{2}
  & \text{maximize}   & \quad & \langle\grad{F(\comb{i-1}) \odot (\vone-\vz^{(i-1)}),\vb^{(i)} \odot (\vone-\vy^{(i - 1)})}\rangle         \nonumber \\
  & \text{subject to} &       & \va^{(i)} = \vy^{(i)} \\ \nonumber
  &                   &       & \vb^{(i)} \in  \DM
\end{alignat*}
\vspace{-5mm}}
    Let $\vy^{(i)}\leftarrow(1-\eps)\cdot\vy^{(i-1)}+\eps\cdot\va^{(i)}$.\label{line:y_update}\\
    Let $\vz^{(i)}\leftarrow\vz^{(i-1)} + \eps\cdot(1-\vz^{(i-1)})\odot\vb^{(i)}$.
}
\Return a vector maximizing $F$ among all the vectors in $\{\comb{i} \mid i \in \bZ, \eps^{-1} t_s \leq i \leq \eps^{-1}\}$.
\caption{\texttt{Frank-Wolfe/Continuous-Greedy Hybrid}\label{alg:OfflineGF}}
\end{algorithm}


Intuitively, Algorithm~\ref{alg:Offline} strives to guarantee that both $F(\comb{i})$ and $F(\vz^{(i)})$ increase at appropriate rates. Algorithm~\ref{alg:OfflineGF} relaxes this goal, and its aim is to guarantee that some combination of these expressions increases at an appropriate rate. Specifically, the combination considered is given by the potential function $\phi(i)\triangleq e^{2(\eps i-t_s)}\cdot F(\comb{i}) + (1 - m)(1 - \eps) \cdot e^{\eps i-2t_s}(t_s-\eps i)\cdot F(\vz^{(i)})$. We note that the idea of using such a dynamic combination of the two functions of interest as a potential function can be traced back to the work of~\cite{feldman2021guess}.

It is also interesting to note that, for $i > \eps^{-1} t_s$, the variable $\va^{(i)}$ and $\vy^{(i)}$ of Algorithm~\ref{alg:Offline} are always set to be equal to $\vy^{(i - 1)}$. This means that one could simplify the algorithm by dropping the constraint regarding $\va^{(i)}$ from the second linear program of Algorithm~\ref{alg:Offline} and executing Line~\ref{line:y_update} only for $i \leq \eps^{-1} t_s$. This observation is essential for getting the online result of Section~\ref{sec:Online}. However, for the sake of the current section, the given description of the algorithm is more convenient as it immediately implies the following observation, which shows that the pair $(\va^{(i)}, \vb^{(i)})$ always obey all the constraints of the linear program of Algorithm~\ref{alg:Offline}, except for one.

\begin{observation} \label{obs:properties_of_solutions}
For every $1 \leq i \leq \eps^{-1}$, $\va^{(i)} \in \NDM$, $\vb^{(i)} \in \DM$ and $\va^{(i)} + (1 - \vy^{(i - 1)}) \odot \vb^{(i)} \leq \vone$.
\end{observation}

Given Observation~\ref{obs:properties_of_solutions}, the proofs of Lemmata~\ref{lem:membership} and~\ref{lem:normxy} go through without a change. Thus, we can use both lemmata in the analysis of Algorithm~\ref{alg:OfflineGF}, which in particular, implies that the output vector of this algorithm belongs to $(\NDM + \DM) \cap [0, 1]^n$.

Our objective in the rest of this section is to lower bound the value of this output vector.
%
We begin with a basic lower bound on the rate in which the potential function $\phi(i)$ increases.
\newtoggle{AppendixEmp}
\togglefalse{AppendixEmp}
\begin{restatable}{lemma}{lemBasicLowerBound} \label{lem:basic_lower_bound}
For every integer $1 \leq i \leq \eps^{-1}t_s$, the expression $\eps^{-1} e^{2t_s}[\phi(i) - \phi(i - 1)]$ can be lower bounded both by $2e^{2\eps (i - 1)} \cdot F(\comb{i-1}) - (1 - m)(1 - \eps) \cdot e^{\eps i}(1 - t_s + \eps i) \cdot F(\vz^{(i - 1)}) - \tfrac{\iftoggle{AppendixEmp}{62}{25}\eps\beta D^2}{1 - m} - \iftoggle{AppendixEmp}{30}{15}\eps\cdot F(\comb{i - 1})$ and by the sum of this expression and
\begin{align*}
	(1 - \eps) \cdot e^{2\eps i} \cdot \langle \nabla F(&\comb{i-1}) \odot (\vone - \vz^{(i - 1)}), (\vone - \vy^{(i - 1)}) \odot \optpone + \optq - \vy^{(i - 1)} \rangle \\\nonumber&+  (1 - m)(1 - \eps) \cdot e^{\eps i}(t_s-\eps i)\cdot \inner{(\vone - \vz^{(i - 1)}) \odot \optpone}{\nabla F(\vz^{(i - 1)})}
	\enspace.
\end{align*}
\end{restatable}
\begin{proof}
Observe that
\begin{align} \label{eq:first_diff}
	e^{2\eps i}\cdot {}&F(\comb{i}) - e^{2\eps (i - 1)}\cdot F(\comb{i - 1})\\\nonumber
	={} &
	[e^{2\eps i} - e^{2\eps(i - 1)}] \cdot F(\comb{i-1}) + e^{2\eps i} \cdot [F(\comb{i}) - F(\comb{i - 1})]\\\nonumber
	\geq{} &
	\int_0^\eps \{2e^{2(\eps (i - 1) + \tau)} \cdot F(\comb{i-1}) + e^{2\eps i} \cdot \langle \nabla F(\comb{i-1}) \odot (\vone - \vz^{(i - 1)}), \\[-2mm]\nonumber&\mspace{150mu} (\vone - \vy^{(i - 1)}) \odot \vb^{(i)} + (\va^{(i)} - \vy^{(i - 1)}) \odot (\vone - 2\tau \cdot \vb^{(i)}) \rangle -45\tau\beta D^2\}d\tau\\\nonumber
	\geq{} &
	2\eps e^{2\eps (i - 1)} \cdot F(\comb{i-1}) + \eps e^{2\eps i} \cdot \langle \nabla F(\comb{i-1}) \odot (\vone - \vz^{(i - 1)}), \\\nonumber&\mspace{150mu} (\vone - \vy^{(i - 1)}) \odot \vb^{(i)} + (\va^{(i)} - \vy^{(i - 1)}) \odot (\vone - \eps \cdot \vb^{(i)}) \rangle -23\eps^2\beta D^2
	\enspace,
\end{align}
where the first inequality follows from the calculations done in the first half of the proof of Lemma~\ref{lem:offline}, and the second inequality uses the non-negativity of $F$. Similarly,
\begin{align} \label{eq:second_diff}
	e^{\eps i}(&t_s-\eps i)\cdot F(\vz^{(i)}) - e^{\eps (i - 1)}(t_s-\eps (i - 1))\cdot F(\vz^{(i - 1)})\\\nonumber
	={} &
	[e^{\eps i}(t_s-\eps i) - e^{\eps (i - 1)}(t_s-\eps (i - 1))] \cdot F(\vz^{(i - 1)}) + e^{\eps i}(t_s-\eps i)\cdot [F(\vz^{(i)}) - F(\vz^{(i - 1)})]\\\nonumber
	\geq{} &
	\int_0^\eps \{e^{\eps (i - 1) + \tau}(t_s - 1 - \eps (i - 1) - \tau) \cdot F(\vz^{(i - 1)}) + e^{\eps i}(t_s-\eps i)\cdot \inner{(\vone - \vz^{(i - 1)}) \odot \vb^{(i)}}{\nabla F(\vz^{(i - 1)})} \\[-2mm]\nonumber&\mspace{550mu}- 4\tau\beta D^2 / (1 - m)^2\}d\tau\\\nonumber
	\geq{} &
	-\eps e^{\eps i}(1 - t_s + \eps i) \cdot F(\vz^{(i - 1)}) + \eps e^{\eps i}(t_s-\eps i)\cdot \inner{(\vone - \vz^{(i - 1)}) \odot \vb^{(i)}}{\nabla F(\vz^{(i - 1)})} \\\nonumber&\mspace{550mu}- 2\eps^2\beta D^2 / (1 - m)^2
	\enspace,
\end{align}
where the first inequality is based on calculations from the proof of Lemma~\ref{lem:z_value}, and the second inequality uses $F$'s non-negativity. Adding up Inequality~\eqref{eq:first_diff} and $(1 - m)(1 - \eps)$ times Inequality~\eqref{eq:second_diff}, we get
\begin{align} \label{eq:diff_first_bound}
	\eps^{-1}e^{2t_s}[&\phi(i) - \phi(i - 1)]\\\nonumber
	\geq{}&
	2 e^{2\eps (i - 1)} \cdot F(\comb{i-1}) - (1 - m)(1 - \eps) \cdot e^{\eps i}(1 - t_s + \eps i) \cdot F(\vz^{(i - 1)}) - \tfrac{25\eps\beta D^2}{1 - m}\\ \nonumber
	&+ (1 - \eps) \cdot e^{2\eps i} \cdot \langle \nabla F(\comb{i-1}) \odot (\vone - \vz^{(i - 1)}), (\vone - \vy^{(i - 1)}) \odot \vb^{(i)} + \va^{(i)} - \vy^{(i - 1)} \rangle \\\nonumber&+  (1 - m)(1 - \eps) \cdot e^{\eps i}(t_s-\eps i)\cdot \inner{(\vone - \vz^{(i - 1)}) \odot \vb^{(i)}}{\nabla F(\vz^{(i - 1)})} \\\nonumber&+ \eps e^{2\eps i} \cdot \langle \nabla F(\comb{i-1}) \odot (\vone - \vz^{(i - 1)}), (\va^{(i)} - \vy^{(i - 1)}) \odot (\vone - \vb^{(i)}) \rangle
	\enspace.
\end{align}

Note now that Properties~\ref{prop:dr_bound2_up} and~\ref{prop:dr_bound2_down} of Lemma~\ref{lem:DR_properties} imply together
\begin{align*}
	\langle \nabla F(\comb{i-1}&) \odot (\vone - \vz^{(i - 1)}), (\va^{(i)} - \vy^{(i - 1)}) \odot (\vone - \vb^{(i)}) \rangle\\
	={} &
	\langle \nabla F(\comb{i-1}) \odot (\vone - \vz^{(i - 1)}), \va^{(i)} \odot (\vone - \vy^{(i - 1)}) \odot (\vone - \vb^{(i)}) \rangle \\&- \langle \nabla F(\comb{i-1}) \odot (\vone - \vz^{(i - 1)}), \vy^{(i - 1)} \odot (\vone - \va^{(i)}) \odot (\vone - \vb^{(i)}) \rangle\\
	\geq{} &
	F(\vy^{(i - 1)} \psum \vz^{(i - 1)} \psum (\va^{(i)} \odot (\vone - \vb^{(i)}))) + F((\vy^{(i - 1)} \odot (\va^{(i)} \psum \vb^{(i)})) \psum \vz^{(i - 1)})\\&- 2 \cdot F(\comb{i - 1})
	\geq
	- 2 \cdot F(\comb{i - 1})
	\enspace,
\end{align*}
where the last inequality holds by the non-negativity of $F$. Plugging this inequality into Inequality~\eqref{eq:diff_first_bound} yields
\begin{align*}
	\eps^{-1}e^{2t_s}[&\phi(i) - \phi(i - 1)]\\\nonumber
	\geq{}&
	2 e^{2\eps (i - 1)} \cdot F(\comb{i-1}) - (1 - m)(1 - \eps) \cdot e^{\eps i}(1 - t_s + \eps i) \cdot F(\vz^{(i - 1)}) - \tfrac{25\eps\beta D^2}{1 - m}\\ \nonumber
	&+ (1 - \eps) \cdot e^{2\eps i} \cdot \langle \nabla F(\comb{i-1}) \odot (\vone - \vz^{(i - 1)}), (\vone - \vy^{(i - 1)}) \odot \vb^{(i)} + \va^{(i)} - \vy^{(i - 1)} \rangle \\\nonumber&+  (1 - m)(1 - \eps) \cdot e^{\eps i}(t_s-\eps i)\cdot \inner{(\vone - \vz^{(i - 1)}) \odot \vb^{(i)}}{\nabla F(\vz^{(i - 1)})}  - 15\eps\cdot F(\comb{i - 1})
	\enspace.
\end{align*}
The last two terms on the right hand side of the last inequality are related to the objective function of the first linear program of Algorithm~\ref{alg:OfflineGF}. Specifically, to get from them to the objective function of this linear program, it is necessary to multiply by $(1 - \eps)^{-1}$, and then remove the additive term $e^{2\eps i} \cdot \langle \nabla F(\comb{i-1}) \odot (\vone - \vz^{(i - 1)}), - \vy^{(i - 1)} \rangle$, which does not depend on $\va^{(i)}$ and $\vb^{(i)}$. Therefore, we can lower bound the right hand side of the last inequality by plugging in feasible solutions for the first linear program of Algorithm~\ref{alg:OfflineGF}. Specifically, we plug in the solutions $(\va^{(i)}, \vb^{(i)}) = (\optq, \optpone)$ and $(\vy^{(i - 1)}, \vzero)$, which implies the two lower bounds stated in the lemma. Notice that both these solutions are guaranteed to be feasible since $\DM$ is down-closed and $\vy^{(i - 1)} + \vzero = \vy^{(i - 1)} \leq \vone$.
\end{proof}

We now need to develop the basic lower bound given by Lemma~\ref{lem:basic_lower_bound} in two ways. First, we show in the next corollary that the potential cannot decrease significantly. Then, in Lemma~\ref{lem:increase_bound}, we show a lower bound on the increase of the potential in terms of $F(\vo)$ and $F(\optpone)$. 
\begin{corollary} \label{cor:not_decreasing}
For every integer $1 \leq i \leq \eps^{-1}t_s$,
\[
	\phi(i) - \phi(i - 1)
	\geq
	- \tfrac{25\eps^2\beta D^2}{1 - m} - 15\eps^2\cdot \phi(i - 1)
	\enspace.
\]
\end{corollary}
\begin{proof}
By Corollary~\ref{cor:norm_bound} and Lemma~\ref{lem:normxy},
\begin{align*}
	2e^{2\eps (i - 1)} \cdot F(&\comb{i-1})
	\geq
	2(1 - m) \cdot e^{2\eps (i - 1)} \cdot (1 - \eps)^{i - 1} \cdot F(\vz^{(i - 1)})\\
	\geq{} &
	(1 - m) \cdot e^{\eps (i - 1)} \cdot F(\vz^{(i - 1)})
	\geq
	(1 - m)(1 - \eps) \cdot e^{\eps i} (1 - t_s + \eps i) \cdot F(\vz^{(i - 1)})
	\enspace,
\end{align*}
where the second inequality follows from our assumption that $\eps \leq 1/30$. By plugging the last inequality into Lemma~\ref{lem:basic_lower_bound}, we get
\[
	\eps^{-1} e^{2t_s}[\phi(i) - \phi(i - 1)]
	\geq
	- \tfrac{25\eps\beta D^2}{1 - m} - 15\eps\cdot F(\comb{i - 1})
	\enspace.
\]
The corollary follows from this inequality since the non-negativity of $F$ implies that \[\phi(i - 1) \geq e^{2(\eps (i - 1)-t_s)}\cdot F(\comb{i - 1}) \geq e^{-2t_s} \cdot F(\comb{i - 1}) \enspace. \qedhere\]
\end{proof}

\begin{lemma} \label{lem:increase_bound}
For every integer $1 \leq i \leq \eps^{-1}t_s$,
\[
	\phi(i) - \phi(i - 1)
	\geq
	\eps(1 - m)(1 - \eps) \cdot [e^{\eps i - 2t_s} \cdot F(\vo) + e^{-2t_s}(t_s-\eps i) \cdot F(\optpone)] - \tfrac{25\eps^2\beta D^2}{1 - m} - 62\eps^2\cdot \phi(i - 1)
	\enspace.
\]
\end{lemma}
\begin{proof}
Repeating some of the calculations from the proof of Lemma~\ref{lemma:right}, we get
\begin{align*}
	\langle \nabla F(\comb{i-1}&) \odot (\vone - \vz^{(i - 1)}), (\vone - \vy^{(i - 1)}) \odot \optpone + (\optq - \vy^{(i - 1)}) \rangle\\
	\geq{} &
	(1 - \eps)^{i - 1}(1 - m) \cdot F(\vo) + \term{1-\eps}^{i - 1}(1 - m) \cdot F(\vz^{(i-1)}) - 2F(\comb{i-1})
	\enspace.
\end{align*}
Additionally, by repeating some of the calculations from the proof of Lemma~\ref{lem:feasible_solution}, we get
\[
	\langle \optpone \odot (\vone-\vz^{(i-1)}), \nabla F(\vz^{(i-1)}) \rangle
	\geq
	\term{1-\eps}^{i - 1} \cdot F(\optpone) - F(\vz^{(i-1)})
	\enspace.
\]
Plugging both these inequalities into the guarantee of Lemma~\ref{lem:basic_lower_bound}, we get that $\eps^{-1} e^{2t_s}[\phi(i) - \phi(i - 1)]$ is at least
\begin{align*}
	&2e^{2\eps (i - 1)} \cdot F(\comb{i-1}) - (1 - m)(1 - \eps) \cdot e^{\eps i}(1 - t_s + \eps i) \cdot F(\vz^{(i - 1)}) - \tfrac{25\eps\beta D^2}{1 - m}\\
	&+(1 - \eps) \cdot e^{2\eps i} \cdot [(1 - \eps)^{i - 1}(1 - m) \cdot F(\vo) + \term{1-\eps}^{i - 1}(1 - m) \cdot F(\vz^{(i-1)}) - 2F(\comb{i-1})] \\\nonumber&+ (1 - m)(1 - \eps) \cdot e^{\eps i}(t_s-\eps i)\cdot [\term{1-\eps}^{i - 1} \cdot F(\optpone) - F(\vz^{(i-1)})] - 15\eps\cdot F(\comb{i - 1})\\
	\geq{} &
	(1 - m)(1 - \eps) \cdot [e^{\eps i} \cdot F(\vo) + (t_s-\eps i) \cdot F(\optpone)] - \tfrac{25\eps\beta D^2}{1 - m} - 62\eps\cdot F(\comb{i - 1})\\
	\geq{} &
	(1 - m)(1 - \eps) \cdot [e^{\eps i} \cdot F(\vo) + (t_s-\eps i) \cdot F(\optpone)] - \tfrac{25\eps\beta D^2}{1 - m} - 62\eps\cdot e^{2t_s} \cdot \phi(i - 1)
	\enspace.
	\qedhere
\end{align*}
\end{proof}

The last two claims provide lower bounds on the change in $\phi$ as a function of $i$. We now use these bounds to get a lower bound on $F(\vy^{(\eps^{-1}t_s)} \psum \vz^{(\eps^{-1}t_s)}) = \phi(\eps^{-1} t_s)$.
\begin{lemma} \label{lem:t_s_value}
It holds that
\begin{align*}
	F(\vy^{(\eps^{-1}t_s)} \psum \vz^{(\eps^{-1}t_s)}\mspace{-34mu}&\mspace{34mu})
	=
	\phi(\eps^{-1} t_s)\\
	\geq{} &
	\frac{(1 - m)(1 - \eps) \cdot \{(e^{-t_s} - e^{-2t_s}) \cdot F(\vo) + e^{-2t_s} \cdot \frac{t_s^2 - \eps}{2} \cdot F(\optpone)\} - t_s \cdot O(\tfrac{\eps\beta D^2}{1 - m})}{1 + O(\eps)}
	\enspace.
\end{align*}
\end{lemma}
\begin{proof}
Summing up the guarantee of Lemma~\ref{lem:increase_bound} for every integer $1 \leq i \leq \eps^{-1} t_s$ yields
\begin{align} \label{eq:telescopic_sum}
	\phi(\eps^{-1} t_s)
	\geq{} &
	\phi(0) + \eps(1 - m)(1 - \eps) \cdot \sum_{i = 1}^{\eps^{-1} t_s}[e^{\eps i - 2t_s} \cdot F(\vo) + e^{-2t_s}(t_s-\eps i) \cdot F(\optpone)] \\\nonumber&- \tfrac{25\eps t_s \beta D^2}{1 - m} - 62\eps^2\cdot \sum_{i = 1}^{\eps^{-1} t_s} \phi(i - 1)
	\enspace.
\end{align}
Let us now bound some of the terms in the last inequality. 
First,
\[
	\eps \cdot \sum_{i = 1}^{\eps^{-1} t_s} e^{\eps i - 2t_s} \cdot F(\vo)
	\geq
	\int_0^{t_s} e^{\tau - 2t_s} d\tau \cdot  F(\vo)\\
	=
	e^{\tau - 2t_s}|_0^{t_s} \cdot  F(\vo)\\
	=
	(e^{-t_s} - e^{-2t_s}) \cdot F(\vo)
	\enspace.
\]
Second,
\[
	\eps \cdot \sum_{i = 1}^{\eps^{-1} t_s} e^{-2t_s}(t_s-\eps i) \cdot F(\optpone)
	=
	e^{-2t_s}\left[t^2_s - \eps^2 \cdot \frac{(\eps^{-1}t_s)(\eps^{-1}t_s + 1)}{2}\right] \cdot F(\optpone)
	\geq
	e^{-2t_s} \cdot \frac{t_s^2 - \eps}{2} \cdot F(\optpone)
	\enspace.
\]
Finally, if we denote by $i^*$ the value of $i$ for which the maximum of $\max_{0 \leq i \leq \eps^{-1} t_s} \phi(i)$ is obtained, then, by Corollary~\ref{cor:not_decreasing},
\[
	\phi(i^*)
	\leq
	\phi(\eps^{-1} t_s) + \tfrac{25t_s \eps \beta D^2}{1 - m} + 15\eps^2 \cdot \sum_{i = i^* + 1}^{\eps^{-1} t_s} \phi(i - 1)
	\leq
	\phi(\eps^{-1} t_s) + \tfrac{25t_s \eps \beta D^2}{1 - m} + 15\eps \cdot \phi(i^*)
	\enspace,
\]
and thus,
\begin{align*}
	\eps \cdot \sum_{i = 1}^{\eps^{-1} t_s} \phi(i - 1) - \tfrac{50t_s \eps \beta D^2}{1 - m}
	\leq{} &
	\phi(i^*) - \tfrac{50t_s \eps \beta D^2}{1 - m}\\
	\leq{} &
	\frac{(1 - 15 \eps) \cdot \phi(i^*) - \tfrac{25t_s \eps \beta D^2}{1 - m}}{1 - 15 \eps}
	\leq
	\frac{\phi(\eps^{-1} t_s)}{1 - 15\eps}
	\leq
	2 \cdot \phi(\eps^{-1} t_s)
	\enspace,
\end{align*}
where the second and last inequalities follows from our assumption that $\eps \leq 1/30$. The lemma now follows by plugging all the above bounds into Inequality~\eqref{eq:telescopic_sum}, and observing that $\phi(0) \geq 0$ due to the non-negativity of $F$.
\end{proof}

Up to this point we have only considered the first $\eps^{-1} t_s$ iterations of Algorithm~\ref{alg:OfflineGF}. We now need to lower bound the rate in which $F(\comb{i})$ increases in the remaining iterations of the algorithm.
\begin{restatable}{lemma}{lemIncreaseSecondPart} \label{lem:increase_second_part}
For every integer $\eps^{-1} t_s < i \leq \eps^{-1}$,
\[
	F(\comb{i}) - F(\comb{i - 1})
	\geq
	\eps \cdot [(1 - m)(1 - \eps)^{i-1} \cdot F(\optptwo) - F(\comb{i - 1})] - \eps^2 \beta D^2\iftoggle{AppendixEmp}{}{/2}
	\enspace.
\]
\end{restatable}
\begin{proof}
By the chain rule,
\begin{align} \label{eq:chain_rule}
	F(\comb{i}\mspace{-30mu}&\mspace{30mu}) - F(\comb{i - 1})
	=
	\int_0^\eps \frac{dF(\vy^{(i - 1)} \psum (\vz^{(i - 1)} + \tau \cdot \vb^{(i)} \odot (\vone - \vz^{(i - 1)})))}{d\tau} d\tau\\\nonumber
	={} &
	\int_0^\eps \inner{\vb^{(i)} \odot (\vone - \vz^{(i - 1)}) \odot (\vone - \vy^{(i - 1)})}{\nabla F(\vy^{(i - 1)} \psum (\vz^{(i - 1)} + \tau \cdot \vb^{(i)} \odot (\vone - \vz^{(i - 1)})))} d\tau
	\enspace.
\end{align}
We would like to lower bound the integrand on the rightmost side of the last equality. We do that by lower bounding two expressions whose sum is equal to this integrand. The first expression is
\begin{align*}
	&
	\langle \vb^{(i)} \odot (\vone - \vz^{(i - 1)}) \odot (\vone - \vy^{(i - 1)}),\\&\mspace{180mu}\nabla F(\vy^{(i - 1)} \psum (\vz^{(i - 1)} + \tau \cdot \vb^{(i)} \odot (\vone - \vz^{(i - 1)}))) - \nabla F(\comb{i-1})\rangle\\
	\geq{} &
	-\|\vb^{(i)} \odot (\vone - \vz^{(i - 1)}) \odot (\vone - \vy^{(i - 1)})\|_2 \\&\mspace{180mu}\cdot \|\nabla F(\vy^{(i - 1)} \psum (\vz^{(i - 1)} + \tau \cdot \vb^{(i)} \odot (\vone - \vz^{(i - 1)}))) - \nabla F(\comb{i-1})\|_2\\
	\geq{} &
	-\tau\beta\|\vb^{(i)} \odot (\vone - \vz^{(i - 1)}) \odot (\vone - \vy^{(i - 1)})\|_2^2
	\geq
	-\tau\beta D^2
	\enspace,
\end{align*}
where the first inequality holds by the Cauchy–Schwarz inequality, the second inequality holds due to the $\beta$-smoothness of $F$, and the last inequality holds since
\[
	\|\vb^{(i)} \odot (\vone - \vz^{(i - 1)}) \odot (\vone - \vy^{(i - 1)})\|_2
	\leq \|\vb^{(i)} \odot (\vone - \vy^{(i - 1)})\|_2
	=
	\|(\vy^{(i - 1)} + \vb^{(i)} \odot (\vone - \vy^{(i - 1)})) - \vy^{(i - 1)}\|_2
	\leq
	D
\]
because both $\vy^{(i - 1)} + \vb^{(i)} \odot (\vone - \vy^{(i - 1)})$ and $\vy^{(i - 1)}$ are vectors of $(\NDM + \DM) \cap [0, 1]^n$. The second expression is
\begin{align*}
	\langle \vb^{(i)} \odot (\vone - \vz^{(i - 1)}) \odot (\vone - \vy^{(i - 1)}), & \nabla F(\vy^{(i - 1)} \psum \vz^{(i - 1)})\rangle\\
	={} &
	\inner{\vb^{(i)} \odot (\vone - \vy^{(i - 1)})}{\nabla F(\vy^{(i - 1)} \psum \vz^{(i - 1)})\odot (\vone - \vz^{(i - 1)})}\\
	\geq{} &
	\inner{\optptwo \odot (\vone - \vy^{(i - 1)})}{\nabla F(\vy^{(i - 1)} \psum \vz^{(i - 1)})\odot (\vone - \vz^{(i - 1)})}\\
	\geq{} &
	F(\optptwo \psum \comb{i - 1}) - F(\comb{i - 1})\\
	\geq{} &
	(1 - \eps)^{i-1}(1 - m) \cdot F(\optptwo) - F(\comb{i - 1})
	\enspace,
\end{align*}
where the first inequality holds since $\vb^{(i)} = \optptwo$ is a feasible solution for the second linear program of Algorithm~\ref{alg:OfflineGF}, the second inequality follows from Property~\eqref{prop:dr_bound2_up} of Lemma~\ref{lem:DR_properties}, and the last inequality follows from Corollary~\ref{cor:norm_bound} and Lemma~\ref{lem:normxy}. The lemma now follows by plugging the lower bounds we have proved into the rightmost side of Equality~\eqref{eq:chain_rule}, and then solving the integral obtained.
\end{proof}

Using the lower bound proved in the last lemma on the rate in which $F(\comb{i})$ increases as a function of $i$ for $i \geq \eps^{-1} t_s$, we derive in the next lemma a lower bound on the value of this expression for particular values of $i$. The proof of this lower bound is very similar to the proof of Lemma~\ref{lem:bound2}.

\begin{lemma} \label{lem:bound2GF}
For every integer $\eps^{-1} t_s \leq i \leq \eps^{-1}$, $F(\comb{i}) \geq (1 - m)(\eps i - t_s)(1 - \eps)^{i - 1} \cdot F(\optptwo) + (1 - \eps)^{i - \eps^{-1}t_s} \cdot F(\vy^{(\eps^{-1}t_s)} \psum \vz^{(\eps^{-1}t_s)}) - (i - \eps^{-1}t_s) \cdot O(\eps^2\beta D^2)$.
\end{lemma}
\begin{proof}
We prove the lemma by induction on $i$. For $i = \eps^{-1} t_s$ the lemma trivially holds. Thus, assume that the lemma holds for $i - 1$, and let us prove it for $i > \eps t_s$. By Lemma~\ref{lem:increase_second_part},
\begin{align*}
	F(\comb{i})
	\geq{} &
	(1 - \eps) \cdot F(\comb{i-1}) + \eps(1 - m)(1 - \eps)^{i - 1} \cdot F(\optptwo) - O(\eps^2\beta D^2)\\
	\geq{} &
	(1 - \eps) \cdot [(1 - m)(\eps i - \eps - t_s)(1 - \eps)^{i - 2} \cdot F(\optptwo) \\&\mspace{80mu}+ (1 - \eps)^{i - 1 - \eps^{-1}t_s} \cdot F(\vy^{(\eps^{-1}t_s)} \psum \vz^{(\eps^{-1}t_s)}) - (i - 1 - \eps^{-1}t_s) \cdot O(\eps^2\beta D^2)] \\&\mspace{80mu} + \eps(1 - m)(1 - \eps)^{i - 1} \cdot F(\optptwo) - O(\eps^2\beta D^2)\\
	\geq{} &
	(1 - m)(\eps i - t_s)(1 - \eps)^{i - 1} \cdot F(\optptwo) + (1 - \eps)^{i - \eps^{-1}t_s} \cdot F(\vy^{(\eps^{-1}t_s)} \psum \vz^{(\eps^{-1}t_s)}) \\&\mspace{80mu}- (i - \eps^{-1}t_s) \cdot O(\eps^2\beta D^2)
	\enspace,
\end{align*}
where the second inequality holds by the induction hypothesis.
\end{proof}

We are now ready to complete the proof of the approximation guarantee of Algorithm~\ref{alg:OfflineGF}, which completes the proof of Proposition~\ref{prop:Offline_ts} (recall that the output of Algorithm~\ref{alg:OfflineGF} has a value of $\max_{\eps^{-1} t_s \leq i \leq \eps^{-1}} F(\comb{i})$).
\begin{lemma} \label{lem:approximation_with_potential}
For every value $T \in [t_s, 1]$, it holds that
\begin{align*}
	\max_{\eps^{-1} t_s \leq i \leq \eps^{-1}} \mspace{-17mu}F(\comb{i})
	\geq{} &
	F(\vy^{(\lfloor\eps^{-1}T\rfloor)} \psum \vz^{(\lfloor\eps^{-1}T\rfloor)})
	\geq
	\frac{(1 - m)}{1 + 42\eps} \cdot \bigg\{(e^{-T} - e^{-t_s - T} - 2\eps) \cdot F(\vo) \\&+ (T - t_s - \eps) \cdot e^{-T} \cdot F(\optptwo) + \bigg[\frac{t_s^2 \cdot e^{-t_s - T}}{2} - \eps\bigg] \cdot F(\optpone) \bigg\} - O(\tfrac{\eps\beta D^2}{1 - m})
	\enspace.
\end{align*}
\end{lemma}
\begin{proof}
Plugging the guarantee of Lemma~\ref{lem:t_s_value} into the guarantee of Lemma~\ref{lem:bound2GF} yields
\begin{align*}
	F(&\vy^{(\lfloor\eps^{-1}T\rfloor)} \psum \vz^{(\lfloor\eps^{-1}T\rfloor)})
	\geq
	(1 - m)(\eps\lfloor\eps^{-1}T\rfloor - t_s)(1 - \eps)^{\eps\lfloor\eps^{-1}T\rfloor - 1} \cdot F(\optptwo) \\&- (\eps\lfloor\eps^{-1}T\rfloor - t_s) \cdot O(\eps\beta D^2) + (1 - \eps)^{\eps^{-1}\lfloor\eps^{-1}T\rfloor - \eps^{-1}t_s} \cdot F(\vy^{(\eps^{-1}t_s)} \psum \vz^{(\eps^{-1}t_s)})\\
	\geq{} &
	(1 - m)(T - t_s - \eps) \cdot e^{-T} \cdot F(\optptwo) - (T - t_s) \cdot O(\eps\beta D^2) \\&+ e^{t_s - T}(1 - \eps) \left[\frac{(1 - m)(1 - \eps) \{(e^{-t_s} - e^{-2t_s}) \cdot F(\vo) + e^{-2t_s} \cdot \frac{t_s^2 - \eps}{2} \cdot F(\optpone)\} - t_s \cdot O(\tfrac{\eps \beta D^2}{1 - m})}{1 + O(\eps)}\right]\\
	\geq{} &
	\frac{(1 - m)(1 - 2\eps)}{1 + O(\eps)} \cdot \bigg\{(e^{-T} - e^{-t_s - T}) \cdot F(\vo) \\&\mspace{100mu}+ (T - t_s - \eps) \cdot e^{-T} \cdot F(\optptwo) + \bigg[\frac{t_s^2 \cdot e^{-t_s - T}}{2} - \eps\bigg] \cdot F(\optpone) \bigg\} - O(\tfrac{\eps\beta D^2}{1 - m})
	\enspace.&
	\qedhere
\end{align*}
\end{proof}
\section{Online Maximization}
\label{sec:Online}

In this section, we consider the online (regret minimization) setting described in Section~\ref{ssc:online}. 
The algorithm we present and analyze in this section has the guarantee given by the next theorem.
\begin{theorem} \label{thm:Online}
Let $\NDM \subseteq [0,1]^n$ be a general solvable convex set, and $\DM \subseteq [0,1]^n$ be a down-closed solvable convex set. If the adversary chooses only non-negative $\beta$-smooth DR-submodular functions $F_\ell \colon [0, 1]^n \rightarrow [0, 1]$, then there exists a polynomial time online algorithm that, given an error parameter $\eps \in (0, 1)$, outputs at every time step $\ell \in [L]$ a distribution $P_\ell$ over vectors in $(\NDM + \DM) \cap [0, 1]^n$ guaranteeing that
\begin{align*}
	\sum_{\ell = 1}^L \bE_{\vx \sim P_\ell}\mspace{-12mu}&\mspace{12mu}[F_\ell(\vx)] \geq (1 - m) \cdot \max_{t_s \in [0, 1]} \max_{T \in [t_s, 1]} \bigg\{\left((T - t_s) e^{-T}  + \frac{t_s^2\cdot e^{-t_s - T}}{2} - O(\eps)\right) \cdot \sum_{\ell = 1}^L F_\ell(\optp) \\&+ (e^{-T}-e^{-t_s - T} - O(\eps)) \cdot \sum_{\ell = 1}^L F_\ell(\optq + \optp)\bigg\} - O(\eps\beta L D^2) - O(DG\sqrt{L}) - O(\sqrt{L \ln \eps^{-1}})\enspace,
\end{align*}
where $m = \min_{\vx\in \NDM}\norm{\vx}{}$, $D$ is the diameter of $(\NDM + \DM) \cap [0, 1]^n$, $G = \max\{\max_{\vx \in \DM} \|\nabla F_t(\vx)\|_2,\allowbreak \max_{\vx \in (\NDM + \DM) \cap [0, 1]^n} \|\nabla F_t(\vx)\|_2\}$, and $\optq \in \NDM$ and $\optp \in \DM$ are any vectors whose sum belongs to $(\NDM + \DM) \cap [0, 1]^n$.
\end{theorem}

We would like to draw attention to two remarks about Theorem~\ref{thm:Online}.
\begin{itemize}
	\item In Section~\ref{ssc:online}, we state that the regret of our online algorithm (compared to the offline algorithm) asymptotically grows as $\sqrt{L}$. At first glance, this might seem to contradict the presence of the error term $O(\eps\beta L D^2)$ in the guarantee of Theorem~\ref{thm:Online}. However, this is not really the case since this error term is part of the $\psi$ functions mentioned in Section~\ref{ssc:online} (because the term $O(\tfrac{\eps \beta D^2}{1 - m})$ appears in the guarantee of Theorem~\ref{thm:Offline}). On a more intuitive level, the error term $O(\eps\beta L D^2)$ is not significant as it diminishes when $\eps$ approaches $0$.
	\item Theorem~\ref{thm:Offline} considers two vectors $\optpone$ and $\optptwo$ in $\DM$, where $\optptwo$ can be an arbitrary vector of $\DM$, and $\optpone$ is required to obey also the inequality $\optq + \optpone \leq \vone$. Technically, a similar result could be proved in Theorem~\ref{thm:Online}. However, in the online setting, the algorithm is required to be competitive against any fixed solution, and therefore, we felt that it makes more sense to present Theorem~\ref{thm:Online} for the case in which $\optpone$ and $\optptwo$ are the same vector (denoted by $\optp$).
\end{itemize}

Similarly to Section~\ref{sec:Offline}, we use $\vo$ below as a shorthand for the sum $\optq + \optp$. The majority of this section is devoted to proving the following proposition, which is a variant of Theorem~\ref{thm:Online} in which (i) $t_s$ is a parameter of the algorithm, and (ii) the algorithm outputs in each time step a single vector rather than a distribution over vectors.
\begin{proposition} \label{prop:Online}
Let $\NDM \subseteq [0,1]^n$ be a general solvable convex set, and $\DM \subseteq [0,1]^n$ be a down-closed solvable convex set. If the adversary chooses only non-negative $\beta$-smooth DR-submodular functions $F_\ell \colon [0, 1]^n \rightarrow [0, 1]$, then there exists a polynomial time online algorithm that, given parameters $t_s \in [0, 1]$ and $\eps \in (0, 1)$, outputs at every time step $\ell \in [L]$ a vector $\vw^{(\ell)} \in (\NDM + \DM) \cap [0, 1]^n$ guaranteeing that
\begin{align*}
	\sum_{\ell = 1}^L F_\ell(\vw^{(\ell)}) \geq{} & (1 - m) \cdot \max_{T \in [t_s, 1]} \bigg\{\left((T - t_s) e^{-T}  + \frac{t_s^2\cdot e^{-t_s - T}}{2} - O(\eps)\right) \cdot \sum_{\ell = 1}^L F_\ell(\optp) \\&\mspace{72mu}+ (e^{-T}-e^{-t_s - T} - O(\eps)) \cdot \sum_{\ell = 1}^L F_\ell(\optq + \optp)\bigg\} - O(\eps\beta L D^2) - O(DG\sqrt{L})\enspace,
\end{align*}
where $m = \min_{\vx\in \NDM}\norm{\vx}{}$, $D$ is the diameter of $(\NDM + \DM) \cap [0, 1]^n$, $G = \max\{\max_{\vx \in \DM} \|\nabla F_t(\vx)\|_2,\allowbreak \max_{\vx \in (\NDM + \DM) \cap [0, 1]^n} \|\nabla F_t(\vx)\|_2\}$, and $\optq \in \NDM$ and $\optp \in \DM$ are any vectors whose sum belongs to $(\NDM + \DM) \cap [0, 1]^n$.
\end{proposition}

Before getting to the proof of Proposition~\ref{prop:Online}, we first explain why it implies Theorem~\ref{thm:Online}. As explained in Section~\ref{ssc:unknown_optp}, the guarantee given in Theorem~\ref{thm:Online} remains unchanged (up to the constants hidden by the big $O$ notation) if the maximum over $t_s \in [0, 1]$ is restricted to the set of $O(\eps^{-1})$ values that are integer multiples of $\eps$ between $0$ and $1$. 
If we knew upfront what value from this set leads to the best guarantee for Proposition~\ref{prop:Online}, then we could use the algorithm of this proposition to get the guarantee of Theorem~\ref{thm:Online}.
Unfortunately, in reality, we do not usually know upfront the best value for $t_s$. Nevertheless, since there are only $O(\eps^{-1})$ such values that need to be considered, we can use a regret minimization algorithm (such as the one of~\cite{cesa-bianchi2007improved}) to get in each time step a distribution over solutions whose expected value is at least as good as the guarantee of Proposition~\ref{prop:Online} for the best value of $t_s$, up to an error term of $O(\sqrt{L \ln \eps^{-1}})$. Thus, Theorem~\ref{thm:Online} indeed follows from Proposition~\ref{prop:Online}.

At this point, we would like to describe the algorithm that we use to prove Proposition~\ref{prop:Online}, which is given as Algorithm~\ref{alg:Online}. Similarly to the Meta-Frank-Wolfe algorithm of~\cite{chen2018online}, Algorithm~\ref{alg:Online} uses multiple instances of an online algorithm for linear optimization. Specifically, we use the algorithm Regularized-Follow-the-Leader due to~\cite{abernethy2008competing}, which has the following behavior. There are $L$ time steps. In every time step $\ell \in [L]$, the algorithm selects a vector $\vu^{(\ell)} \in \cK$ for some given convex body $\cK$, and then an adversary reveals to the algorithm a vector $\vd^{(\ell)}$ that was chosen independently of $\vu^{(\ell)}$. Regularized-Follow-the-Leader guarantees that
\[
	\sum_{t = 1}^L \langle \vu^{(\ell)},\vd^{(\ell)} \rangle
	\geq
	\max_{\vx \in \cK} \sum_{\ell = 1}^L \langle \vx,\vd^{(\ell)} \rangle - D'G'\sqrt{2L}
	\enspace,
\]
where $D'$ is the diameter of $\cK$ and $G' = \max_{1 \leq \ell \leq L} \norm{\vd^{(\ell)}}{2}$.

Algorithm~\ref{alg:Online} follows the same general structure of Algorithm~\ref{alg:OfflineGF} in every time step, with two modifications: the linear programs of Algorithm~\ref{alg:OfflineGF} are now replaced by instances of the online linear optimization algorithm that we use, and the algorithm uses the vector corresponding to $i = T^{-1}$ as the output instead of the best vector among multiple options. 
In the pseudocode of Algorithm~\ref{alg:Online}, we implicitly assume that $\eps \leq 1/70$ and $\eps t_s$ is integral. These assumptions are without loss of generality due to the arguments used in Section~\ref{ssc:unknown_optp} to justify similar assumptions in Algorithm~\ref{alg:OfflineGF}.

\begin{algorithm}[t!h!]
\DontPrintSemicolon
\For{$i = 1$ \KwTo $\eps^{-1}$}
{
	Initialize an instance $\cE_i$ of Regularized-Follow-the-Leader for the convex body $\{(\va, \vb) \mid \va \in \NDM, \vb \in \DM, \va + \vb \in [0, 1]^n\}$.
}
Let $m \gets \min_{\vx\in \NDM}\norm{\vx}{}$.\\
\For{$\ell=1$ \KwTo $L$}
{
	Let $\vy^{(0, \ell)} \leftarrow \argmin_{\vx\in \NDM}\norm{\vx}{}$, $\vz^{(0, \ell)} \gets \vzero$.\\
	\For{$i = 1$ \KwTo $\eps^{-1} t_s$}
	{
		Let $(\va^{(i, \ell)}, \vb^{(i, \ell)})$ be the vector picked by $\mathcal{E}_i$ in time step $\ell$.\\
		Let $\vy^{(i, \ell)}\leftarrow(1-\eps) \cdot \vy^{(i-1, \ell)} + \eps \cdot \va^{(i, \ell)}$.\\
		Let $\vz^{(i,\ell)}\leftarrow\vz^{(i-1, \ell)} + \eps\cdot(1-\vz^{(i-1, \ell)})\odot\vb^{(i, \ell)}$.\\
	}
	\For{$i = \eps^{-1} t_s + 1$ \KwTo $\eps^{-1}$}
	{
		Let $\va^{(i, \ell)} \gets \vy^{(i - 1, \ell)}$.\tcp*{This vector is used for analysis purposes only.}
		Let $\vb^{(i, \ell)}$ be a vector consisting of the last $n$ coordinates of the vector picked by $\mathcal{E}_i$ in time step $\ell$.\\
		Let $\vy^{(i, \ell)}\leftarrow \vy^{(i-1, \ell)}$.\\
		Let $\vz^{(i,\ell)}\leftarrow\vz^{(i-1, \ell)} + \eps\cdot(1-\vz^{(i-1, \ell)})\odot\vb^{(i, \ell)}$.\\
	}
	Set the vector $\vy^{(\eps^{-1}, \ell)} \psum \vz^{(\eps^{-1}, \ell)}$ as the output for time step $\ell$.\\
	\For{$i = 1$ \KwTo $\eps^{-1} t_s$}
	{
		Let $\vg^{(i, \ell)}$ be a vector in $\bR^{2n}$ obtained as follows. The first $n$ coordinates of this vector are given by $e^{2\eps i} \cdot \nabla F_\ell(\comb{i-1,\ell}) \odot (\vone-\vz^{(i - 1,\ell)})$, and the other $n$ coordinates of $\vg^{(i, \ell)}$ are equal to $e^{2\eps i} \cdot \nabla F_\ell(\comb{i-1,\ell}) \odot (\vone-\vz^{(i - 1,\ell)}) \odot (\vone-\vy^{(i - 1,\ell)}) + (1 - m) \cdot e^{\eps i}(t_s - \eps i) \cdot \nabla F_\ell(\vz^{(i-1,\ell)}) \odot (1-\vz^{(i-1,\ell)})$.\\
		Pass $\vg^{(i, \ell)}$ as the adversarially chosen vector $\vd^{(\ell)}$ for $\cE_i$.
	}
	\For{$i = \eps^{-1} t_s + 1$ \KwTo $\eps^{-1}$}
	{
		Let $\vg^{(i, \ell)}$ be a vector in $\bR^{2n}$ obtained as follows. The first $n$ coordinates of this vector are all zeros, and the other $n$ coordinates of $\vg^{(i, \ell)}$ are equal to $\nabla F_\ell(\comb{i-1,\ell}) \odot (\vone-\vz^{(i-1,\ell)}) \odot (\vone-\vy^{(i - 1,\ell)})$.\\
		Pass $\vg^{(i, \ell)}$ as the adversarially chosen vector $\vd^{(\ell)}$ for $\cE_i$.
	}
}
\caption{\texttt{Online Frank-Wolfe/Continuous-Greedy Hybrid\label{alg:Online}}}
\end{algorithm}

One can observe that, for every fixed time step $\ell \in [L]$, Algorithm~\ref{alg:Online} sets the vectors $\{\vy^{(i, \ell)}, \vz^{(i, \ell)} \mid i \in \bZ, 0 \leq i \leq \eps^{-1}\}$ in the same way in which Algorithm~\ref{alg:OfflineGF} sets the vectors $\{\vy^{(i)}, \vz^{(i)} \mid i \in \bZ, 0 \leq i \leq \eps^{-1}\}$, with the sole difference that the vectors $\{\va^{(i, \ell)}, \vb^{(i, \ell)} \mid i \in [\eps^{-1}]\}$ used for this purpose may not optimize the linear programs of Algorithm~\ref{alg:OfflineGF} (but they are feasible solutions for these linear programs). 
%
Since the property that these vectors optimize the linear programs of Algorithm~\ref{alg:OfflineGF} is not used in the proof that Lemmata~\ref{lem:membership} and~\ref{lem:normxy} apply to this algorithm, we immediately get that both these lemmata apply also to Algorithm~\ref{alg:Online} for any fixed $\ell$. In other words, we get the following lemma.
\begin{lemma} \label{lem:online_imports}
For every two integers $0 \leq i \leq \eps^{-1}$ and $\ell \in [L]$,
\begin{itemize}
	\item $\vy^{(i,\ell)} \in \NDM$ and $\vz^{(i,\ell)} \in \eps i \cdot \DM$, and thus, $\comb{i,\ell} \in (\NDM + \DM) \cap [0, 1]^n$.
	\item $\|\vz^{(i,\ell)}\|_\infty \leq 1 - (1 - \eps)^i$ and $\|\vy^{(i,\ell)}\|_\infty \leq \|\comb{i,\ell}\|_\infty \leq 1 - \term{1-\eps}^{i}(1 - m)$.
\end{itemize}
\end{lemma}

The last lemma implies, in particular, that Algorithm~\ref{alg:Online} outputs feasible vectors.
The rest of this section is devoted to showing that the total value of these vectors is at least as large as is guaranteed by Proposition~\ref{prop:Online}. For that purpose, we need to define a new potential function $\phi(i) \triangleq \sum_{\ell = 1}^L \phi(i, \ell)$, where
\[
	\phi(i, \ell)
	\triangleq
	e^{2(\eps i-t_s)}\cdot F_\ell(\comb{i,\ell}) + (1 - m)(1 - \eps) \cdot e^{\eps i-2t_s}(t_s-\eps i)\cdot F_\ell(\vz^{(i, \ell)})
	\enspace.
\]
Notice that the definition $\phi(i, \ell)$ is identical to the definition of the potential function from Section~\ref{ssc:unknown_optp} up to the addition of the index $\ell$ to the various vectors and the function $F$.

The next lemma corresponds to Lemma~\ref{lem:basic_lower_bound} from the analysis of Algorithm~\ref{alg:OfflineGF}. Notice that this lemma has only one guarantee, in contrast to Lemma~\ref{lem:basic_lower_bound} that has two guarantees (because the proof of the other guarantee does not generalize to the online setting). Due to the lack of this addition guarantee, there is no result in this section corresponding to Corollary~\ref{cor:not_decreasing} from the analysis of Algorithm~\ref{alg:OfflineGF}, and the proof of Lemma~\ref{lem:t_s_value_online} is more involved than the proof of the corresponding Lemma~\ref{lem:t_s_value}.
\begin{lemma} \label{lem:basic_lower_bound_online}
For every integer $1 \leq i \leq \eps^{-1}t_s$,
\begin{align*}
	\mspace{20mu}&\mspace{-20mu}\eps^{-1}e^{2t_s}[\phi(i) - \phi(i - 1)]
	\geq
	2e^{2\eps (i - 1)} \cdot \sum_{\ell = 1}^L F_\ell(\comb{i-1, \ell}) \\&- (1 - m)(1 - \eps) \cdot e^{\eps i}(1 - t_s + \eps i) \cdot \sum_{\ell = 1}^L F_\ell(\vz^{(i - 1, \ell)}) - 25\eps\beta LD^2 - 15\eps\cdot \sum_{\ell = 1}^L F_\ell(\comb{i - 1,\ell})\\
	&+
	(1 - \eps) \cdot e^{2\eps i} \cdot \sum_{\ell = 1}^L \langle \nabla F_\ell(\comb{i-1, \ell}) \odot (\vone - \vz^{(i - 1)}), (\vone - \vy^{(i - 1, \ell)}) \odot \optp + \optq - \vy^{(i - 1, \ell)} \rangle \\\nonumber&+  (1 - m)(1 - \eps) \cdot e^{\eps i}(t_s-\eps i)\cdot \sum_{\ell = 1}^L \inner{(\vone - \vz^{(i - 1, \ell)}) \odot \optp}{\nabla F_\ell(\vz^{(i - 1, \ell)})} - 24DG\sqrt{L}
	\enspace.
\end{align*}
\end{lemma}
\begin{proof}
The definition of the convex-body assigned to $\cE_i$ guarantees that, for every $\ell \in [L]$, there exists a vector $\vx^{(\ell)} \in \NDM$ such that $\vx^{(\ell)} + \vb^{(i, \ell)} \in (\NDM + \DM) \cap [0, 1]^n$. Thus, $\|\vb^{(i, \ell)}\|_2 = \|(\vx^{(\ell)} + \vb^{(i, \ell)}) - \vx^{(\ell)}\|_2 \leq D$ because the down-closeness of $\DM$ implies that $\vx^{(\ell)}$ also belongs to $(\NDM + \DM) \cap [0, 1]^n$. Plugging this observation into the first part of the proof of Lemma~\ref{lem:basic_lower_bound}, we get, for every $\ell \in [L]$, 
\begin{align*}
	\mspace{15mu}&\mspace{-15mu}\eps^{-1}e^{2t_s}[\phi(i, \ell) - \phi(i - 1, \ell)]
	\geq
	2 e^{2\eps (i - 1)} \cdot F_\ell(\comb{i-1,\ell}) \\&- (1 - m)(1 - \eps) \cdot e^{\eps i}(1 - t_s + \eps i) \cdot F_\ell(\vz^{(i - 1,\ell)}) - 25\eps\beta D^2 - 15\eps\cdot F_\ell(\comb{i - 1,\ell})\\ \nonumber
	&+ (1 - \eps) \cdot e^{2\eps i} \cdot \langle \nabla F_\ell(\comb{i-1,\ell}) \odot (\vone - \vz^{(i - 1,\ell)}), (\vone - \vy^{(i - 1,\ell)}) \odot \vb^{(i,\ell)} + \va^{(i,\ell)} - \vy^{(i - 1,\ell)} \rangle \\\nonumber&+  (1 - m)(1 - \eps) \cdot e^{\eps i}(t_s-\eps i)\cdot \inner{(\vone - \vz^{(i - 1,\ell)}) \odot \vb^{(i,\ell)}}{\nabla F_\ell(\vz^{(i - 1,\ell)})}
	\enspace.
\end{align*}
Adding up this inequality for all $\ell \in [L]$ yields
{\allowdisplaybreaks
\begin{align} \label{eq:sum_basic}
	&\eps^{-1}e^{2t_s}[\phi(i) - \phi(i - 1)]
	\geq
	2 e^{2\eps (i - 1)} \cdot \sum_{\ell = 1}^L F_\ell(\comb{i-1,\ell})  \\\nonumber&- (1 - m)(1 - \eps) \cdot e^{\eps i}(1 - t_s + \eps i) \cdot \sum_{\ell = 1}^L F_\ell(\vz^{(i - 1,\ell)})- 25\eps\beta LD^2 - 15\eps\cdot \sum_{\ell = 1}^L F_\ell(\comb{i - 1,\ell})\\ \nonumber
	&+ (1 - \eps) e^{2\eps i} \cdot \sum_{\ell = 1}^L \langle \nabla F_\ell(\comb{i-1,\ell}) \odot (\vone - \vz^{(i - 1,\ell)}), (\vone - \vy^{(i - 1,\ell)}) \odot \vb^{(i,\ell)} + \va^{(i,\ell)} - \vy^{(i - 1,\ell)} \rangle \\\nonumber&+  (1 - m)(1 - \eps) \cdot e^{\eps i}(t_s-\eps i)\cdot \sum_{\ell = 1}^L \inner{(\vone - \vz^{(i - 1,\ell)}) \odot \vb^{(i,\ell)}}{\nabla F_\ell(\vz^{(i - 1,\ell)})}
	\enspace.
\end{align}}%

One can observe that the last two terms of on the right hand side of the last inequality are equal to the expression $(1 - \eps) \cdot \sum_{\ell = 1}^L \inner{\vg^{(i, \ell)}}{(\va^{(i, \ell)}, \vb^{(i, \ell)})}$ up to an additive term of $(1 - \eps) e^{2\eps i} \cdot \sum_{\ell = 1}^L \langle \nabla F_\ell(\comb{i-1,\ell}) \odot (\vone - \vz^{(i - 1,\ell)}), - \vy^{(i - 1,\ell)}\rangle$. Therefore, since $(\optq, \optp)$ is one possible solution that $\cE_i$ can return, the guarantee of $\cE_i$ implies
\begin{align} \label{eq:regret_bound}
	&(1 - \eps) e^{2\eps i} \mspace{-3mu}\cdot\mspace{-3mu} \sum_{\ell = 1}^L \langle \nabla F_\ell(\comb{i-1,\ell}) \odot (\vone \mspace{-2mu}-\mspace{-2mu} \vz^{(i - 1,\ell)}), \mspace{-2mu}(\vone \mspace{-3mu}-\mspace{-3mu} \vy^{(i - 1,\ell)}) \odot \vb^{(i,\ell)} \mspace{-3mu}+\mspace{-3mu} \va^{(i,\ell)} \mspace{-3mu}-\mspace{-3mu} \vy^{(i - 1,\ell)} \rangle \\\nonumber&+  (1 - m)(1 - \eps) \cdot e^{\eps i}(t_s-\eps i)\cdot \sum_{\ell = 1}^L \inner{(\vone - \vz^{(i - 1,\ell)}) \odot \vb^{(i,\ell)}}{\nabla F_\ell(\vz^{(i - 1,\ell)})}\\\nonumber
	&\geq (1 - \eps) e^{2\eps i} \cdot \sum_{\ell = 1}^L \langle \nabla F_\ell(\comb{i-1,\ell}) \odot (\vone - \vz^{(i - 1,\ell)}), (\vone - \vy^{(i - 1,\ell)}) \odot \optp + \optq - \vy^{(i - 1,\ell)} \rangle \\\nonumber&+ \mspace{-2mu}(1 \mspace{-2mu}-\mspace{-2mu} m)(1 \mspace{-2mu}-\mspace{-2mu} \eps) \mspace{-2mu}\cdot\mspace{-2mu} e^{\eps i}(t_s\mspace{-2mu}-\mspace{-2mu}\eps i)\cdot \sum_{\ell = 1}^L \inner{(\vone \mspace{-2mu}-\mspace{-2mu} \vz^{(i - 1,\ell)}) \odot \optq}{\nabla F_\ell(\vz^{(i \mspace{-2mu}-\mspace{-2mu} 1,\ell)})} \mspace{-2mu}-\mspace{-2mu} D'\sqrt{2L} \cdot \max\nolimits_{1 \leq i \leq L} \|\vg^{(i, \ell)}\|_2
	\enspace,
\end{align}
where $D'$ is the diameter of the convex body assigned to $\cE_i$.

We need to upper bound the terms $D'$ and $\max\nolimits_{1 \leq i \leq L} \|\vg^{(i, \ell)}\|_2$ appearing in the last inequality. First,
\begin{align*}
	D'
	={} &
	\max_{\substack{\vx^{(1)}, \vx^{(2)} \in \NDM, \vw^{(1)}, \vw^{(2)} \in \DM \\ \vx^{(1)} + \vw^{(1)}, \vx^{(2)} + \vw^{(2)} \in [0, 1]^n}} \|(\vx^{(1)}, \vw^{(1)}) - (\vx^{(2)}, \vw^{(2)})\|_2\\
	\leq{} &
	\max_{\substack{\vx^{(1)}, \vx^{(2)} \in \NDM, \vw^{(1)}, \vw^{(2)} \in \DM \\ \vx^{(1)} + \vw^{(1)}, \vx^{(2)} + \vw^{(2)} \in [0, 1]^n}} \|(\vx^{(1)} + \vw^{(1)}) - (\vx^{(2)} + \vw^{(2)})\|_2
	\leq
	D
	\enspace.
\end{align*}
Second, for every $\ell \in [L]$, the $\ell_2$-norm of the first $n$ coordinates of $\vg^{(i, \ell)}$ is upper bounded by $e^{2\eps i} \cdot \|\nabla F_\ell(\comb{i-1,\ell})\| \leq e^2G$, and the $\ell_2$-norm of the other $n$ coordinates of $\vg^{(i, \ell)}$ is upper bounded by $\|e^{2\eps i} \cdot \nabla F_\ell(\comb{i-1,\ell}) + e^{\eps i} \cdot \nabla F_\ell(\vz^{(i-1,\ell)})\|_2 \leq 2e^2G$. Thus,
\[
	\|\vg^{(i, \ell)}\|_2
	\leq
	\sqrt{(e^2G)^2 + (2e^2G)^2}
	=
	\sqrt{5} \cdot e^2G
	\leq
	\frac{24}{\sqrt{2}} \cdot G
	\enspace.
\]
The lemma now follows by plugging the above upper bounds into Inequality~\eqref{eq:regret_bound}, and then combining the obtained inequality with Inequality~\eqref{eq:sum_basic}.
\end{proof}

To get the following corollary, we repeat the proof of Lemma~\ref{lem:increase_bound} from the analysis of Algorithm~\ref{alg:OfflineGF} with Lemma~\ref{lem:basic_lower_bound_online} taking the role of  Lemma~\ref{lem:basic_lower_bound} and $\optp$ taking the role of $\optpone$.

\begin{corollary} \label{cor:increase_bound_online}
For every integer $1 \leq i \leq \eps^{-1}t_s$,
\begin{align*}
	\phi(i) - \phi(i - 1)
	\geq{} &
	\eps(1 - m)(1 - \eps) \cdot \bigg[e^{\eps i - 2t_s} \cdot \sum_{\ell = 1}^L F_\ell(\vo) + e^{-2t_s}(t_s-\eps i) \cdot \sum_{\ell = 1}^L F_\ell(\optp)\bigg] \\&- 25\eps^2\beta LD^2 - 62\eps^2\cdot \phi(i - 1) - 24\eps DG\sqrt{L}
	\enspace.
\end{align*}
\end{corollary}

We now use the last corollary to get a lower bound on $\sum_{\ell = 1}^L F_\ell(\comb{\eps^{-1}t_s, \ell}) = \phi(\eps^{-1} t_s)$. As mentioned above, Lemma~\ref{lem:t_s_value_online} corresponds to Lemma~\ref{lem:t_s_value} from the analysis of Algorithm~\ref{alg:OfflineGF}, but its analysis is somewhat more involved.
\begin{lemma} \label{lem:t_s_value_online}
It holds that
\begin{align*}
	\sum_{\ell = 1}^L F_\ell(\mspace{50mu}&\mspace{-50mu}\comb{\eps^{-1}t_s, \ell})
	=
	\phi(\eps^{-1} t_s)
	\geq
	- 63t_s \eps\beta LD^2 - 60t_s DG\sqrt{L}\\
	&+(1 - m) \cdot \bigg[(e^{-t_s} - e^{- 2t_s} - O(\eps)) \cdot \sum_{\ell = 1}^L F_\ell(\vo) + \left(\frac{e^{-2t_s} \cdot t_s^2}{2} - O(\eps)\right) \cdot \sum_{\ell = 1}^L F_\ell(\optp)\bigg]
	\enspace.
\end{align*}
\end{lemma}
\begin{proof}
We prove by induction on $i$ that
\begin{align} \label{eq:induction_claim}
	(1 - 62\eps^2)^{-i} \cdot \phi(i)
	\geq{} &
	\eps(1 - m)(1 - \eps) \cdot \sum_{i' = 1}^i \bigg[e^{\eps i' - 2t_s} \cdot \sum_{\ell = 1}^L F_\ell(\vo) + e^{-2t_s}(t_s-\eps i') \cdot \sum_{\ell = 1}^L F_\ell(\optp)\bigg] \\\nonumber& - 63i\eps^2\beta LD^2 - 60i\eps DG\sqrt{L}
	\enspace.
\end{align}
Inequality~\eqref{eq:induction_claim} holds for $i = 0$ since the non-negativity of $F$ guarantees that $\phi(0) \geq 0$. Assume now that Inequality~\eqref{eq:induction_claim} holds for $i - 1$, and let us prove it for $i$. By Corollary~\ref{cor:increase_bound_online} and the induction hypothesis, 
\begin{align*}
	(1 - 62\eps^2)^{-i} &{}\cdot \phi(i)
	=
	(1 - 62\eps^2)^{-(i - 1)} \cdot \phi(i - 1) + (1 - 62\eps^2)^{-i} \cdot[\phi(i) - (1 - 62\eps^2)\phi(i - 1)]\\
	\geq{}&\eps(1 - m)(1 - \eps) \cdot \sum_{i' = 1}^{i - 1} \bigg[e^{\eps i' - 2t_s} \cdot \sum_{\ell = 1}^L F_\ell(\vo) + e^{-2t_s}(t_s-\eps i') \cdot \sum_{\ell = 1}^L F_\ell(\optp)\bigg] \\&- 63(i - 1)\eps^2\beta LD^2 - 60(i - 1)\eps DG\sqrt{L}\\
	&+(1 - 62\eps^2)^{-i} \cdot \bigg\{\eps(1 - m)(1 - \eps) \cdot \bigg[e^{\eps i - 2t_s} \cdot \sum_{\ell = 1}^L F_\ell(\vo) + e^{-2t_s}(t_s-\eps i) \cdot \sum_{\ell = 1}^L F_\ell(\optp)\bigg] \\&- 25\eps^2\beta LD^2 - 24\eps DG\sqrt{L}\bigg\}\\
	\geq{} &
	\eps(1 - m)(1 - \eps) \cdot \sum_{i' = 1}^i \bigg[e^{\eps i' - 2t_s} \cdot \sum_{\ell = 1}^L F_\ell(\vo) + e^{-2t_s}(t_s-\eps i') \cdot \sum_{\ell = 1}^L F_\ell(\optp)\bigg] \\\nonumber&- 63i\eps^2\beta LD^2 - 60i\eps DG\sqrt{L}
	\enspace,
\end{align*}
where the second inequality uses the non-negativity of $F$ and the fact that our assumption that $\eps \leq 1/70$ implies that $(1 - 62\eps^2)^{-i} \leq (1 - 62\eps^2)^{-1/\eps} \leq e^{62\eps} \leq 2.5$. This completes the proof of Inequality~\eqref{eq:induction_claim}. Plugging $i = \eps^{-1} t_s$ into this inequality, we get
\begin{align*}
	\phi(\eps^{-1} t_s)
	\geq{} &
	(1 \mspace{-2mu}-\mspace{-2mu} 62\eps^2)^{\eps^{-1}t_s} \mspace{-2mu}\cdot\mspace{-2mu} \bigg\{\eps(1 \mspace{-2mu}-\mspace{-2mu} m)(1 \mspace{-2mu}-\mspace{-2mu} \eps) \mspace{-2mu}\cdot\mspace{-2mu} \sum_{i' = 1}^{\eps^{-1} t_s} \bigg[e^{\eps i' \mspace{-2mu}-\mspace{-2mu} 2t_s} \cdot \sum_{\ell = 1}^L F_\ell(\vo) \mspace{-2mu}+\mspace{-2mu} e^{-2t_s}(t_s\mspace{-2mu}-\mspace{-2mu}\eps i') \cdot \sum_{\ell = 1}^L F_\ell(\optp)\bigg] \\\nonumber&- 63t_s \eps\beta LD^2 - 60t_s DG\sqrt{L}\bigg\}\\
	\geq{} &
	(1 - m) \cdot \bigg[(e^{-t_s} - e^{- 2t_s} - O(\eps)) \cdot \sum_{\ell = 1}^L F_\ell(\vo) + \left(\frac{e^{-2t_s} \cdot t_s^2}{2} - O(\eps)\right) \cdot \sum_{\ell = 1}^L F_\ell(\optp)\bigg] \\\nonumber&- 63t_s \eps\beta LD^2 - 60t_s DG\sqrt{L}
	\enspace,
\end{align*}
where the second inequality uses two lower bounds on sums that are justified in the proof of Lemma~\ref{lem:t_s_value}.
\end{proof}

Up to this point we have only studied vectors in $\{\vy^{(i, \ell)}, \vz^{(i, \ell)} \mid i \in \bZ, 0 \leq i \leq \eps^{-1} t_s, \ell \in [L]\}$. We now need to consider vectors corresponding to larger values of $i$. We begin with the following lemma, which corresponds to Lemma~\ref{lem:increase_second_part} from the analysis of Algorithm~\ref{alg:OfflineGF}.
\begin{lemma} \label{lem:increase_second_part_online}
For every integer $\eps^{-1} t_s < i \leq \eps^{-1}$,
\begin{align*}
	\sum_{\ell = 1}^\ell [F_\ell(\comb{i,\ell}) -{}& F_\ell(\comb{i - 1,\ell})]\\
	\geq{} &
	\max\bigg\{\eps \cdot \bigg[(1 - m)(1 - \eps)^{i-1} \cdot \sum_{\ell = 1}^L F_\ell(\optp) - \sum_{\ell = 1}^L F_\ell(\comb{i - 1, \ell})\bigg], 0 \bigg\} \\&- \eps^2 \beta LD^2/2 - \eps DG\sqrt{2L}
	\enspace.
\end{align*}
\end{lemma}
\begin{proof}
Repeating the first part of the proof of Lemma~\ref{lem:increase_second_part} implies that, for every $\ell \in [L]$,
\begin{align*}
	F_\ell(\comb{i,\ell}) - {}&F_\ell(\comb{i - 1,\ell})\\
	\geq{} &
	\eps \cdot \inner{\vb^{(i)} \odot (\vone - \vy^{(i - 1, \ell)})}{\nabla F_\ell(\vy^{(i - 1, \ell)} \psum \vz^{(i - 1, \ell)})\odot (\vone - \vz^{(i - 1, \ell)})} - \eps^2\beta D^2/2
	\enspace.
\end{align*}
Summing up this inequality for all $\ell \in [L]$ yields
\begin{align*}
	\sum_{\ell = 1}^\ell [F_\ell(\comb{i,\ell}\mspace{-20mu}&\mspace{20mu}) - F_\ell(\comb{i - 1,\ell})]\\
	\geq{} &
	\eps \cdot \sum_{\ell = 1}^L \inner{\vb^{(i)} \odot (\vone - \vy^{(i - 1, \ell)})}{\nabla F_\ell(\vy^{(i - 1, \ell)} \psum \vz^{(i - 1, \ell)})\odot (\vone - \vz^{(i - 1, \ell)})} - \eps^2\beta LD^2/2\\
	\geq{} &
	\max\bigg\{\eps \cdot \sum_{\ell = 1}^L \inner{\optp \odot (\vone - \vy^{(i - 1, \ell)})}{\nabla F_\ell(\vy^{(i - 1, \ell)} \psum \vz^{(i - 1, \ell)})\odot (\vone - \vz^{(i - 1, \ell)})}, 0\bigg\} \\&- \eps^2\beta LD^2/2 - \eps DG\sqrt{2L}
	\enspace,
\end{align*}
where the second inequality holds by the properties of $\cE_i$ since (i) the convex body assigned to $\cE_i$ is of diameter at most $D$ (see the proof of Lemma~\ref{lem:basic_lower_bound_online}), (ii) both $(\optq, \optp)$ and $(\optq, \vzero)$ are vectors in this convex body, and (iii) for every $\ell \in [L]$, $\|g^{(i, \ell)}\|_2 \leq \|\nabla F(\comb{i-1, \ell})\|_2 \leq G$.

To complete the proof of the lemma, it remains to observe that, for every $\ell \in [L]$,
\begin{align*}
	\inner{\optp \odot (\vone - \vy^{(i - 1, \ell)})}{\nabla F_\ell(\vy^{(i - 1, \ell)} \psum {}\vz^{(i - 1, \ell)})\odot (\vone - \vz^{(i - 1, \ell)})}\mspace{-100mu}&\mspace{100mu}\\
	\geq{} &
	F_\ell(\vy^{(i - 1, \ell)} \psum \vz^{(i - 1, \ell)} \psum \optp) - F_\ell(\vy^{(i - 1, \ell)} \psum \vz^{(i - 1, \ell)})\\
	\geq{} &
	(1 - m)(1 - \eps)^{i-1} \cdot F_\ell(\optp) - F_\ell(\vy^{(i - 1, \ell)} \psum \vz^{(i - 1, \ell)})
	\enspace,
\end{align*}
where the first inequality holds by Property~\ref{prop:dr_bound2_up} of Lemma~\ref{lem:DR_properties}, and the second inequality follows from Corollary~\ref{cor:norm_bound} and Lemma~\ref{lem:online_imports}.
\end{proof}

\begin{corollary} \label{cor:bound2online}
For every integer $\eps^{-1} t_s \leq i \leq \eps^{-1}$, $\sum_{\ell = 1}^L F_\ell(\vy^{(\eps^{-1}, \ell)} \psum \vz^{(\eps^{-1}, \ell)}) \geq (1 - m)(\eps i - t_s)(1 - \eps)^{i - 1} \cdot \sum_{\ell = 1}^L F_\ell(\optp) + (1 - \eps)^{i - \eps^{-1}t_s} \cdot \sum_{i = 1}^L F_\ell(\vy^{(\eps^{-1}t_s, \ell)} \psum \vz^{(\eps^{-1}t_s, \ell)}) - (1 - t_s) \cdot O(\eps\beta L D^2) - (1 - t_s) DG\sqrt{2L}$.
\end{corollary}
\begin{proof}
By repeating the proof of Lemma~\ref{lem:bound2GF} with Lemma~\ref{lem:increase_second_part_online} taking the role of Lemma~\ref{lem:increase_second_part}, one can prove that
\begin{align*}
	\sum_{\ell = 1}^L F_\ell(&\comb{i, \ell})\\
	\geq{} &
	(1 - m)(\eps i - t_s)(1 - \eps)^{i - 1} \cdot \sum_{\ell = 1}^L F_\ell(\optp) + (1 - \eps)^{i - \eps^{-1}t_s} \cdot \sum_{i = 1}^L F_\ell(\vy^{(\eps^{-1}t_s, \ell)} \psum \vz^{(\eps^{-1}t_s, \ell)}) \\&- (\eps i - t_s) \cdot O(\eps\beta L D^2) - (\eps i - t_s) DG\sqrt{2L}
	\enspace.
\end{align*}
By adding up the inequalities guaranteed by Lemma~\ref{lem:increase_second_part_online} for all $i < i' \leq \eps^{-1}$, we can also get
\begin{align*}
	\sum_{\ell = 1}^L F_\ell(\mspace{80mu}&\mspace{-80mu}\vy^{(\eps^{-1}, \ell)} \psum \vz^{(\eps^{-1}, \ell)}) - \sum_{\ell = 1}^L F_\ell(\comb{i, \ell})\\
	={} &
	\sum_{i' = i + 1}^{\eps^{-1}} \sum_{\ell = 1}^\ell [F_\ell(\comb{i,\ell}) - F_\ell(\comb{i - 1,\ell})]\\
	\geq{} &
	\sum_{i' = i + 1}^{\eps^{-1}}\{- \eps^2 \beta LD^2/2 - \eps DG\sqrt{2L}\}
	=
	-(1 - \eps i)\eps \beta LD^2/2 - (1 - \eps i) \cdot DG\sqrt{2L}
	\enspace.
\end{align*}
The corollary now follows by adding up the two above inequalities.
\end{proof}

We are now ready to complete the proof of the approximation guarantee of Algorithm~\ref{alg:Online}, which completes the proof of Proposition~\ref{prop:Online}.
\begin{lemma}
It holds that
\begin{align*}
	\sum_{\ell = 1}^L F_\ell(\vy^{(\eps^{-1}, \ell)} \psum \vz^{(\eps^{-1}, \ell)})
	\geq{} &
	(1 - m) \cdot \max_{T \in [t_s, 1]}\bigg\{\bigg[(T - t_s)e^{-T} + \frac{e^{-t_s-T} \cdot t_s^2}{2} - O(\eps) \bigg] \cdot \sum_{\ell = 1}^L F_\ell(\optp) \\&+ (e^{-T} - e^{- t_s-T} - O(\eps)) \cdot \sum_{\ell = 1}^L F_\ell(\vo) \bigg\} - O(\eps\beta L D^2) - O(DG\sqrt{L})
	\enspace.
\end{align*}
\end{lemma}
\begin{proof}
We prove below that the inequality stated in the lemma holds for any fixed $T \in [t_s, 1]$. Plugging the guarantee of Lemma~\ref{lem:t_s_value_online} into the guarantee of Corollary~\ref{cor:bound2online} for $i = \lfloor \eps^{-1} T \rfloor$ yields
{\allowdisplaybreaks
\begin{align*}
	&\sum_{\ell = 1}^L F_\ell(\vy^{(\eps^{-1}, \ell)} \psum \vz^{(\eps^{-1}, \ell)})\\
	\geq{} &
	(1 - m)(\eps\lfloor \eps^{-1} T \rfloor - t_s)(1 - \eps)^{\lfloor \eps^{-1} T \rfloor - 1} \cdot \sum_{\ell = 1}^L F_\ell(\optp)  - (\eps\lfloor \eps^{-1} T \rfloor - t_s) \cdot O(\eps\beta L D^2) \\&- (1 - t_s) DG\sqrt{2L}
	+ (1 - \eps)^{\lfloor \eps^{-1} T \rfloor - \eps^{-1}t_s} \cdot \bigg\{(1 - m) \cdot \bigg[(e^{-t_s} - e^{- 2t_s} - O(\eps)) \cdot \sum_{\ell = 1}^L F_\ell(\vo)\\
	&+ \left(\frac{e^{-2t_s} \cdot t_s^2}{2} - O(\eps)\right) \cdot \sum_{\ell = 1}^L F_\ell(\optp)\bigg] - 63t_s \eps\beta LD^2 - 60t_s DG\sqrt{L}\bigg\}\\
	\geq{} &
	(1 - m)(T - t_s - \eps)e^{-T} \cdot \sum_{\ell = 1}^L F_\ell(\optp) - O(\eps\beta L D^2) - O(DG\sqrt{L})\\
	&+\mspace{-2mu} e^{t_s - T}(1 \mspace{-2mu}-\mspace{-2mu} m)(1 \mspace{-2mu}-\mspace{-2mu} \eps)\mspace{-2mu} \bigg[(e^{-t_s} \mspace{-2mu}-\mspace{-2mu} e^{- 2t_s} \mspace{-2mu}-\mspace{-2mu} O(\eps)) \cdot \sum_{\ell = 1}^L F_\ell(\vo)
	\mspace{-2mu}+\mspace{-2mu} \left(\frac{e^{-2t_s} \mspace{-2mu}\cdot\mspace{-2mu} t_s^2}{2} \mspace{-2mu}-\mspace{-2mu} O(\eps)\right) \cdot \sum_{\ell = 1}^L F_\ell(\optp)\bigg]\\
	={} &
	(1 - m) \cdot \bigg\{\bigg[(T - t_s)e^{-T} + \frac{e^{-t_s-T} \cdot t_s^2}{2} - O(\eps) \bigg] \cdot \sum_{\ell = 1}^L F_\ell(\optp) \\&+ (e^{-T} - e^{- t_s-T} - O(\eps)) \cdot \sum_{\ell = 1}^L F_\ell(\vo) \bigg\} - O(\eps\beta L D^2) - O(DG\sqrt{L})
	\enspace.&
	\qedhere
\end{align*}}%
\end{proof}
\section{Applications and Experimental Results}\label{sec:Experiments}

%

In this section, we study the empirical performance of the offline and online algorithms described in the previous sections on various machine learning tasks.

\subsection{Revenue Maximization}

As our first experimental setting, we consider the following revenue maximization setting, which was also considered by~\cite{mualem2023resolving,soma2017nonmonotone,thang2021online}. The objective of some company is to promote a product to users with the aim of boosting revenue through the \say{word-of-mouth} effect. The problem of optimizing this objective can be formalized as follows. The input is a weighted undirected graph $G = (V, E)$ representing a social network, where $w_{ij}$ represents the weight of the edge between vertex $i$ and vertex $j$ (with $w_{ij} = 0$ if the edge $(i, j)$ is absent from the graph). If the company allocates a cost of $x_i$ units to a user $i\in V$, then that user becomes an advocate of the product with a probability of $1-(1-p)^{x_i}$, where $p\in(0,1)$ is a parameter. Note that this formula implies that each $\eps$ unit of cost invested in the user independently contributes to the chance of making the user an advocate. Furthermore, by investing a full unit in the user, the user becomes an advocate with a probability of $p$~\cite{soma2017nonmonotone}.

Given the set $S\subseteq V$ of users who have become advocates for the product, the expected revenue generated is proportional to the total influence of the users in $S$ on non-advocate users, formally expressed as $\sum_{i\in S}\sum_{j\in{V \setminus S}}w_{ij}$. Hence, the objective function $F\colon [0, 1]^V \rightarrow \nnR$ in this setting is defined as the expectation of the aforementioned expression, i.e.,
\begin{align}
    F(\vx)=\mathbb{E}_S\left[\sum_{i\in S}\sum_{j\in{V\setminus S}}w_{ij}\right]=\sum_{i \in V}\sum_{\substack{j \in V \\ i\neq j}}w_{ij}(1-(1-p)^{x_i})(1-p)^{x_j}
		\enspace.
\end{align}
It has been demonstrated that $F$ is a non-monotone DR-submodular function~\cite{soma2017nonmonotone}.

We conducted experiments in both online and offline scenarios based on instances of the aforementioned setting derived from two distinct datasets. The first dataset is sourced from a Facebook network~\cite{viswanath2009evolution}, encompassing $64K$ users (vertices) and $1M$ unweighted relationships (edges). The second dataset is based on the Advogato network~\cite{massa2009dowling}, comprising $6.5K$ users (vertices) and $61K$ weighted relationships (edges).

\subsubsection{Online Setting} \label{ssc:revenue_online}

In our online experiments, inspired by~\cite{mualem2023resolving}, we set the number of time steps to $L = 1000$, with $p=0.0001$. At each time step $\ell$, the objective function is defined by selecting a uniformly random subset $V_\ell \subseteq V$ of a given size, and then retaining only edges connecting two vertices of $V_\ell$. For the Advogato network, $V_\ell$ is of size $200$, and for the larger Facebook network, $V_\ell$ is of size $\numprint{15,000}$. The above objective functions are optimized subject to the constraint $0.1\leq\sum_ix_i\leq 1$, which represents both minimum and maximum investment requirements. Notably, the intersection of this constraint with the implicit box constraint forms a non-down-monotone feasibility polytope. However, this polytope can be decomposed into two polytopes: (i) $\NDM$, a polytope defined by the equality $\sum_{i = 1}^n x_i =0.1$, and (ii) $\DM$, a down-closed polytope defined by the inequality $\sum_{i = 1}^n x_i\leq 0.9$. Observe that $(\DM + \NDM) \cap [0, 1]^n$ is indeed the original polytope, and thus, this is a valid decomposition of this polytope.

\begin{figure}[!t]
\begin{subfigure}[t]{0.24\textwidth}
    \begin{tikzpicture}[scale=0.45]
         \begin{axis}[
				x label style={font=\large},
				y label style={font=\large},
				xlabel=Timestep, 
				ylabel=Function Value,
       legend cell align=left,
		legend style={font=\large,at={(0.68,1)}},
		y tick label style={
        /pgf/number format/.cd,
        fixed,
        fixed zerofill,
        precision=2,
        /tikz/.cd}]
\addplot[line width=1pt,solid,color=blue,mark=*,  mark repeat=100] %
	table[x=Iteration,y=Ours,col sep=comma]{CSV/revenue/advogatoComparisonOnline.csv};
\addlegendentry{Ours}
\addplot[line width=1pt,solid,color=red,mark=triangle*, mark size=3pt,  mark repeat=100] %
	table[x=Iteration,y=Mualem,col sep=comma]{CSV/revenue/advogatoComparisonOnline.csv};
\addlegendentry{Mualem et. al~\cite{mualem2023resolving}}
\end{axis}
\end{tikzpicture}
\caption{\centering Online Algorithms on the Advogato network.}
\label{fig:onAdv} 
    \end{subfigure}
     \begin{subfigure}[t]{0.24\textwidth}
    \begin{tikzpicture}[scale=0.45]
         \begin{axis}[
				x label style={font=\large},
				y label style={font=\large},
	xlabel=Timestep, 
	ylabel=Function Value,
       legend cell align=left,
		legend style={font=\large,at={(0.68,1)}},
		y tick label style={
        /pgf/number format/.cd,
        fixed,
        fixed zerofill,
        precision=2,
        /tikz/.cd}]
\addplot[line width=1pt,solid,color=blue,mark=*,  mark repeat=100] %
	table[x=Iteration,y=Ours,col sep=comma]{CSV/revenue/facebookComparisonOnline.csv};
\addlegendentry{Ours}
\addplot[line width=1pt,solid,color=red,mark=triangle*, mark size=3pt,  mark repeat=100] %
	table[x=Iteration,y=Mualem,col sep=comma]{CSV/revenue/facebookComparisonOnline.csv};
\addlegendentry{Mualem et al.~\cite{mualem2023resolving}}
\end{axis}
\end{tikzpicture}
\caption{\centering Online Algorithms on the Facebook network.} \label{fig:onFace}
    \end{subfigure}
		\hfill
    \begin{subfigure}[t]{0.24\textwidth}
    \begin{tikzpicture}[scale=0.45]
         \begin{axis}[
				x label style={font=\large},
				y label style={font=\large},
	xlabel=Timestep, 
	ylabel=Function Value,
       legend cell align=left,
		legend style={font=\large,at={(0.68,1)}},
		y tick label style={
        /pgf/number format/.cd,
        fixed,
        fixed zerofill,
        precision=2,
        /tikz/.cd}]
\addplot[line width=1pt,solid,color=blue,mark=*,  mark repeat=10] %
	table[x=Iteration,y=Ours,col sep=comma]{CSV/revenue/advogato_comparison_offline_Empirical.csv};
\addlegendentry{Ours (Algorithm~\ref{alg:OfflineGF})}
\addplot[line width=1pt,solid,color=green,mark=*,  mark repeat=10] %
	table[x=Iteration,y=Empirical Ours,col sep=comma]{CSV/revenue/advogato_comparison_offline_Empirical.csv};
\addlegendentry{Ours (Algorithm~\ref{alg:OfflineEmp})}
\addplot[line width=1pt,solid,color=red,mark=triangle*, mark size=3pt,  mark repeat=10] %
	table[x=Iteration,y=Mualem,col sep=comma]{CSV/revenue/advogato_comparison_offline_Empirical.csv};
\addlegendentry{Mualem et al.~\cite{mualem2023resolving}}
\end{axis}
\end{tikzpicture}
\caption{\centering Offline Algorithms on the Advogato network.}
\label{fig:offAdv} 
    \end{subfigure}\hfill
    \begin{subfigure}[t]{0.24\textwidth}
    \begin{tikzpicture}[scale=0.45]
         \begin{axis}[
				x label style={font=\large},
				y label style={font=\large},
	xlabel=Timestep, 
	ylabel=Function Value,
       legend cell align=left,
		legend style={font=\large,at={(0.68,1)}},
		y tick label style={
        /pgf/number format/.cd,
        fixed,
        fixed zerofill,
        precision=2,
        /tikz/.cd}]
\addplot[line width=1pt,solid,color=blue,mark=*,  mark repeat=10] %
	table[x=Iteration,y=Ours,col sep=comma]{CSV/revenue/facebook_comparison_offline_Empirical.csv};
\addlegendentry{Ours (Algorithm~\ref{alg:OfflineGF})}
\addplot[line width=1pt,solid,color=green,mark=*,  mark repeat=10] %
	table[x=Iteration,y=Empirical Ours,col sep=comma]{CSV/revenue/facebook_comparison_offline_Empirical.csv};
\addlegendentry{Ours (Algorithm~\ref{alg:OfflineEmp})}
\addplot[line width=1pt,solid,color=red,mark=triangle*, mark size=3pt,  mark repeat=10] %
	table[x=Iteration,y=Mualem,col sep=comma]{CSV/revenue/facebook_comparison_offline_Empirical.csv};
\addlegendentry{Mualem et al.~\cite{mualem2023resolving}}
\end{axis}
\end{tikzpicture}
\caption{\centering Offline Algorithms on the Facebook network.}
\label{fig:offFacebook} 
    \end{subfigure}
    \caption{Results of the Revenue Maximization Experiments} \label{fig:revenue}
\end{figure}
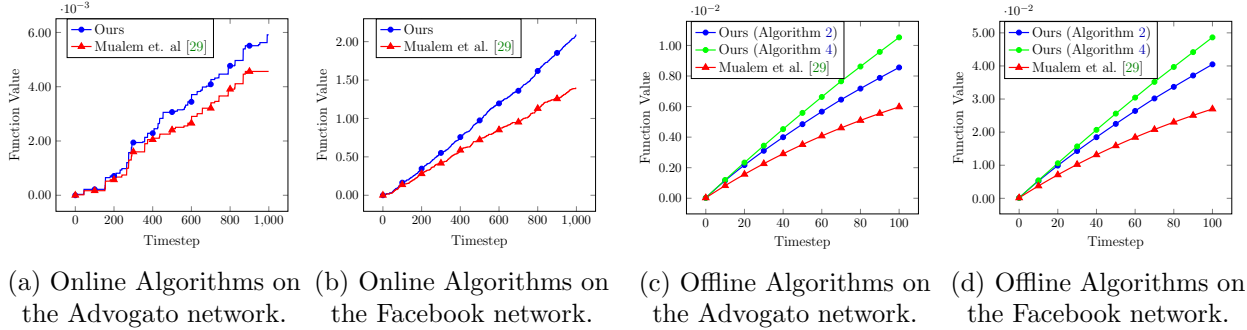

In our experiments, we have compared the perfromance of our algorithm from Section~\ref{sec:Online} with the online algorithm of Mualem \& Feldman~\cite{mualem2023resolving}, which is the current state-of-the-art algorithm for the online setting. In both algorithms, we have set the number of online linear optimizers used to be $100$ (which corresponding to setting the error parameter $\eps$ to $0.01$ in our algorithm and to $\nicefrac{\ln{2}}{100}$ in the algorithm of Mualem \& Feldman). The results of our experiments on the Advogato and Facebook networks can be found in Figures~\ref{fig:onAdv} and~\ref{fig:onFace}, respectively. One can observe that, in both experiments, our algorithm significantly outperforms the state-of-the-art algorithm, especially as the number of time steps grows.

\subsubsection{Offline Setting}

Our offline experiments are similar to their online counterparts. However, since there is only one objective function in this setting, we run the experiment on the entire network graph instead of subsets. In this setting, we compared our offline algorithm (specifically, its versions given as Algorithm~\ref{alg:OfflineGF} in Section~\ref{ssc:unknown_optp} and Algorithm~\ref{alg:OfflineEmp} in Appendix~\ref{app:practical_offline}) with the current state-of-the-art algorithm from~\cite{mualem2023resolving} (which is an explicit version of the algorithm of Du~\cite{du2022lyapunov}). Both algorithms have been executed for $T = 100$ iterations, and the error parameter $\eps$ was set accordingly (which again means $0.01$ in our algorithm and $\nicefrac{\ln{2}}{100}$ in the algorithm of Mualem \& Feldman). The results of the offline experiments on the Advogato and Facebook networks can be found in Figures~\ref{fig:offAdv} and~\ref{fig:offFacebook}, respectively. One can observe that our method consistently outperforms the previous state-of-the-art.

\subsection{Location Summarization}

In this experimental setting, our objective is to create a summary of the locations around the current user location based on the Yelp dataset, a subset of Yelp's businesses, reviews, and user data covering information about local businesses in 11 metropolitan areas~\cite{yelp}. We employ the formalization of this task used by Mualem \& Feldman~\cite{mualem2023resolving}. Specifically, symmetry scores between the various locations have been generated using the methodology introduced by Kazemi et al.~\cite{kazemi2021regularized} based on features extracted from location descriptions and associated user reviews, encompassing aspects like parking availability, WiFi access, vegan menu offerings, delivery options, suitability for outdoor seating, and being conducive to group gatherings. Then, assuming the set of locations is denoted by $[n]$, $M_{i,j}$ represents the similarity score between locations $i$ and $j$, and $d_i$ is the distance of location $i$ from the user (measured in units of 200KM), the objective function for every set $S \subseteq [n]$ is given by
\[ f(S) = \tfrac{1}{n} \sum_{i = 1}^n \max_{j \in S} M_{i, j} - \sum_{i \in S} d_i \enspace. \]
Intuitively, this objective function favors sets $S$ of locations that effectively summarizes the existing locations while remaining close to the user's current location.

Since the methods developed in this work are tailored for continuous functions, they cannot be used to directly optimize $f$. Thus, we need to invoke the multilinear extension\footnote{See, for example,~\cite{buchbinder2018submodular} for the definition of the multi-linear extension.} $F$ of $f$ defined as follows. For integers $i, j, j' \in [n]$, we write $M_{i,j} \prec M_{i, j'}$ if $M_{i, j} < M_{i, j'}$ or $M_{i, j} = M_{i, j'}$ and $j < j'$. Then, for every vector $\vx \in [0, 1]^n$,
\[
    F(\vx)
    =
    \tfrac{1}{n} \sum_{i = 1}^n \sum_{j = 1}^n \left[x_j M_{i, j} \cdot \prod_{j' | M_{i,j} \prec M_{i, j'}} \mspace{-27mu} (1 - x_{j'}) \right] - \sum_{i = 1}^n x_i d_i
    \enspace.
\]
The multilinear extension $F$ is amenable to optimization using our methods since the submodularity of $f$ implies that its multi-linear extension $F$ is DR-submodular~\cite{bian2017guaranteed}. Furthermore, the solution obtained in this way for $F$ can be rounded into a discrete solution with the same approximation for $f$ using either pipage or swap rounding~\cite{calinescu2011maximizing,chekuri2010dependent}.

Our experiment closely follows the one of~\cite{mualem2023resolving} by focusing on a single metropolitan area (Charlotte) and considering a time horizon of $100$ steps. Each time step is associated with a different user $u$ whose location is chosen uniformly at random within the rectangle encompassing the metropolitan area. If we denote by $F_u$ the function $F$ computed based on the location of user $u$, then in the time step associated with $u$, our objective is to select a vector $\vx^{(u)}$ that maximizes $F_u$ among all vectors satisfying $\left\|\vx\right\|_1 \in [1, 2]$ (we seek solutions incorporating either $1$ or $2$ locations). As this optimization should be done before learning the location of $u$ for prompt responses and privacy reasons, online optimization algorithms are used for the task. Akin to Section~\ref{ssc:revenue_online}, we compare the performance in this context of our algorithm from Section~\ref{sec:Online} with the algorithm proposed by Mualem \& Feldman~\cite{mualem2023resolving} when the number of online linear optimizers is set to $L = 100$ for both algorithms. As our algorithm requires a decomposition of the feasible polytope into a down-closed polytope $\DM$ and a general polytope $\NDM$, we chose the decomposition $\DM = \{\vx \in [0, 1]^n \mid \|\vx\|_1 \leq 1\}$ and $\NDM = \{\vx \in [0, 1]^n \mid \|\vx\|_1 = 1\}$ (note that $(\DM + \NDM) \cap [0, 1]^n$ is indeed the original feasible polytope). The results of our experiments are illustrated in Figure~\ref{fig:location}, and demonstrate that our algorithm consistently outperforms the state-of-the-art algorithm of~\cite{mualem2023resolving}.


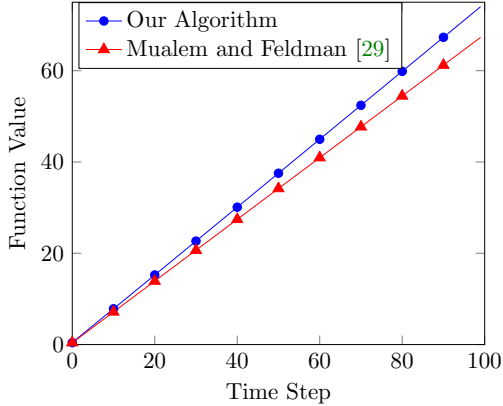
\begin{SCfigure}[][tb]
  \begin{tikzpicture}[scale=0.8] \begin{axis}[
    xlabel = {Time Step},
    ylabel = {Function Value},
    xmin=0, xmax=100,
    ymin=0, ymax=75,
		legend cell align=left,
		legend style={at={(0.8,1)}}]
		\addplot [name path = our, blue, mark = *, mark repeat=10] table [x expr=\coordindex, y index=0] {CSV/location/OnlineYelpMurad.csv};
		\addlegendentry{Our Algorithm}
		\addplot [name path = theirs, red, mark = triangle*, mark size=3pt, mark repeat=10] table [x expr=\coordindex, y index=0] {CSV/location/OnlineYelpLoay.csv};
		\addlegendentry{Mualem and Feldman~\cite{mualem2023resolving}}
	\end{axis}\end{tikzpicture}
\caption{Results of the Location Summarization Experiment} \label{fig:location}
\end{SCfigure}

\subsection{Quadratic Programming}

In this section, we showcase the ability of our algorithm to interpolate between down-closed and general convex bodies. To demonstrate this, we consider a setting with a down-closed polytope constraint, and compare the performance of our method with to two other algorithms: the non-monotone Franke-Wolfe algorithm proposed by Bian et al.~\cite{bian2017nonmonotone}, which was designed for down-closed convex bodies and achieves $e^{-1}$-approximation for maximizing DR-submodular functions over such constraints (but is not well-defined when the constraint is not down-closed), and the non-monotone Franke-Wolfe algorithm suggested by Mualem \& Feldman~\cite{mualem2023resolving} for general convex-body constraints. In a nutshell, the experiments we discuss below show that, despite the ability of our algorithm to optimize over general convex-body constraint, it is able to outperform the algorithm of~\cite{mualem2023resolving} and recover the performance of the algorithm of~\cite{bian2017nonmonotone} (and sometimes even slightly outperform it).

Every instance of the setting we consider is defined by two matrices $\mA \in\nnRE{m\times n}$ and $\mH \in \npRE{m\times n}$ (the details of the setting closely follow settings considered by~\cite{bian2017nonmonotone,mualem2023resolving}). These matrices are chosen at random according to distributions described in Sections~\ref{ssc:uniform} and~\ref{ssc:exponential}. However, before getting to the description of these distributions, let us explain how the instance is constructed based on the matrices $\mA$ and $\mH$ obtained. First, the down-closed polytope constraint is given by
\[
    \mathcal{K}=\{\vx\in \nnRE{n} \mid \mA\vx\leq \vb, \vx\leq \vu\}
		\enspace,
\]
where $\vb$ is the all-ones vector, and the upper bound vector $\vu$ is given by $u_j = \min_{i\in[m]} b_i / A_{i,j}$ for each $j\in[n]$. The function $F$ to be maximized subject to $\cK$ is given, for every vector $\vzero \leq \vx \leq \vu$, by
\[
    F(\vx)=\frac{1}{2}\vx^T\mH\vx+\vh^T\vx+c
		\enspace,
\]
where $\vh$ is a vector and $c$ is a scalar. The non-positivity of the matrix $\mH$ ensures that $F$ is DR-submodular. Additionally, we set $\vh = -0.1 \cdot \mH^T\vu$. Finally, to make $F$ non-negative, it is necessary to set the value of $c$ to be at least $M = -\min_{\vzero \leq \vx \leq \vu} \left(\frac{1}{2}\vx^T\mH\vx + \vh^T\vx\right)$. The value of $M$ can be approximately calculated via \quadprogIP\footnote{We used IBM CPLEX optimization studio \url{https://www.ibm.com/products/ilog-cplex-optimization-studio}.}~\cite{xia2020globally}, and we set $c$ to be $M+0.1|M|$, which is a bit larger than the necessary minimum.


\subsubsection{Uniform Distribution} \label{ssc:uniform}

In this section, we we use a method for choosing the matrices $\mH$ and $\mA$ that utilizes uniform distributions. In this method, the matrix $\mH \in \mathbb{R}^{n \times n}$ is a symmetric matrix randomly generated by drawing each entry independently and uniformly from the interval $[-1, 0]$. Similarly, the matrix $\mA \in \mathbb{R}^{m \times n}$ is generated with entries randomly drawn from the interval $[v, v+1]$, where $v=0.01$. Note that, by using a positive value for $v$, we ensure that the entries of $\mA$ are all strictly positive.

Our experiments based on the above distributions vary in the values chosen for the dimensions $n$ and $m$. For each choice, we generated $100$ instances and executed on them our offline algorithm (the version given as Algorithm~\ref{alg:OfflineEmp} in Appendix~\ref{app:practical_offline}) as well as the algorithms of Mualem \& Feldman~\cite{mualem2023resolving} and Bian et al.~\cite{bian2017nonmonotone}. The number of iterations was set to $100$ in all algorithms, which forces the the error control parameter $\eps$ to be set to $0.01$ in our algorithm and the algorithm of Bian et al.~\cite{bian2017nonmonotone}, and to $\nicefrac{\ln 2}{100}$ in the remaining algorithm. Additionally, since $\cK$ is down-closed, the decomposition used by our algorithm is simply $\DM = \cK$ and $\NDM = \{0\}$. Figure~\ref{fig:quad11} depicts the results obtained by the three algorithms, averaged over the $100$ instances generated. The $x$-axis in each plot represents the value of $n$, and the value of $m$ was derived based on $n$ as specified by each plot caption. The $y$-axis illustrates the approximation ratios of the various algorithms in comparison to the optimum value calculated using a quadratic programming solver. Notably, our proposed algorithm demonstrates near-identical performance to the Frank-Wolfe algorithm suggested by Bian et al.~\cite{bian2017nonmonotone}, and clearly surpasses the algorithm suggested by Mualem \& Feldman~\cite{mualem2023resolving}.

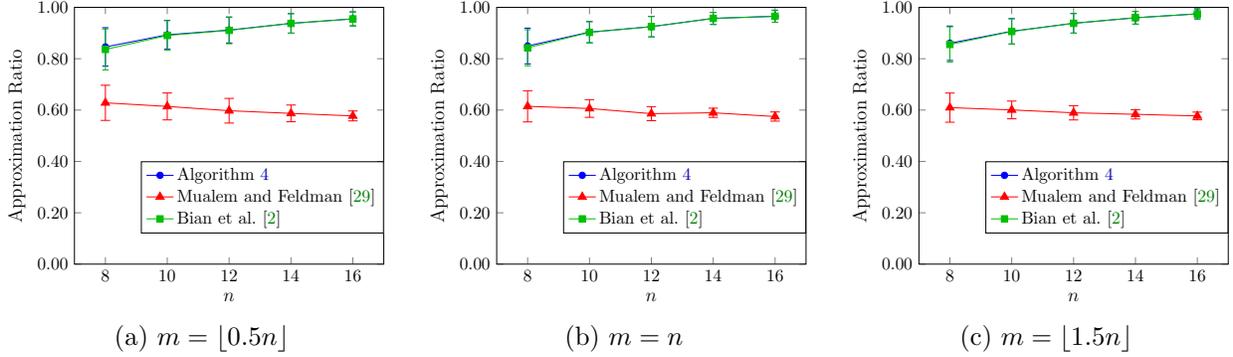
\begin{figure}[tb]
\begin{subfigure}[t]{0.32\textwidth}
  \begin{tikzpicture}[scale=0.6] \begin{axis}[
		x label style = {font=\large},
		y label style = {font=\large},
    xlabel = {$n$},
    ylabel = {Approximation Ratio},
    xmin=7, xmax=17,
    ymin=0, ymax=1,
		legend cell align=left,
		legend style={at={(1,0.4)}},
		y tick label style={
        /pgf/number format/.cd,
        fixed,
        fixed zerofill,
        precision=2,
        /tikz/.cd
    },
		error bars/y dir=both,
    error bars/y explicit]
		\pgfplotstableread{CSV/quadratic/m=0.5n/uniform2_quadratic_ours3__.csv}\ours
		\pgfplotstablecreatecol[copy column from table={CSV/quadratic/m=0.5n/uniform2_quadratic_ours3_std__.csv}{[index] 0}] {error} {\ours}
		\addplot [blue, mark = *] table [x expr=8+2*\coordindex, y index=0, y error index=1] {\ours};
		\addlegendentry{Algorithm~\ref{alg:OfflineEmp}}
		\pgfplotstableread{CSV/quadratic/m=0.5n/uniform2_quadratic_theirs__.csv}\theirs
		\pgfplotstablecreatecol[copy column from table={CSV/quadratic/m=0.5n/uniform2_quadratic_theirs_std__.csv}{[index] 0}] {error} {\theirs}
		\addplot [red, mark = triangle*, mark size=3pt] table [x expr=8+2*\coordindex, y index=0, y error index = 1] {\theirs};
		\addlegendentry{Mualem and Feldman~\cite{mualem2023resolving}}
		\pgfplotstableread{CSV/quadratic/m=0.5n/uniform2_quadratic_theirs2__.csv}\theirstwo
		\pgfplotstablecreatecol[copy column from table={CSV/quadratic/m=0.5n/uniform2_quadratic_theirs2_std__.csv}{[index] 0}] {error} {\theirstwo}
		\addplot [green!75!black, mark = square*] table [x expr=8+2*\coordindex, y index=0, y error index = 1] {\theirstwo};
		\addlegendentry{Bian et al.~\cite{bian2017nonmonotone}}
	\end{axis}\end{tikzpicture}
  \caption{$m=\lfloor0.5n\rfloor$}\label{fig:quad2}
\end{subfigure}\hfill
\begin{subfigure}[t]{0.32\textwidth}
  \begin{tikzpicture}[scale=0.6] \begin{axis}[
		x label style = {font=\large},
		y label style = {font=\large},
    xlabel = {$n$},
    ylabel = {Approximation Ratio},
    xmin=7, xmax=17,
    ymin=0, ymax=1,
		legend cell align=left,
		legend style={at={(1,0.4)}},
		y tick label style={
        /pgf/number format/.cd,
        fixed,
        fixed zerofill,
        precision=2,
        /tikz/.cd
    },
		error bars/y dir=both,
    error bars/y explicit]
		\pgfplotstableread{CSV/quadratic/m=n/uniform2_quadratic_ours3__.csv}\ours
		\pgfplotstablecreatecol[copy column from table={CSV/quadratic/m=n/uniform2_quadratic_ours3_std__.csv}{[index] 0}] {error} {\ours}
		\addplot [blue, mark = *] table [x expr=8+2*\coordindex, y index=0, y error index=1] {\ours};
		\addlegendentry{Algorithm~\ref{alg:OfflineEmp}}
		\pgfplotstableread{CSV/quadratic/m=n/uniform2_quadratic_theirs__.csv}\theirs
		\pgfplotstablecreatecol[copy column from table={CSV/quadratic/m=n/uniform2_quadratic_theirs_std__.csv}{[index] 0}] {error} {\theirs}
		\addplot [red, mark = triangle*, mark size=3pt] table [x expr=8+2*\coordindex, y index=0, y error index = 1] {\theirs};
		\addlegendentry{Mualem and Feldman~\cite{mualem2023resolving}}
		\pgfplotstableread{CSV/quadratic/m=n/uniform2_quadratic_theirs2__.csv}\theirstwo
		\pgfplotstablecreatecol[copy column from table={CSV/quadratic/m=n/uniform2_quadratic_theirs2_std__.csv}{[index] 0}] {error} {\theirstwo}
		\addplot [green!75!black, mark = square*] table [x expr=8+2*\coordindex, y index=0, y error index = 1] {\theirstwo};
		\addlegendentry{Bian et al.~\cite{bian2017nonmonotone}}
	\end{axis}\end{tikzpicture}
  \caption{$m=n$}\label{fig:quad1}
\end{subfigure}\hfill
\begin{subfigure}[t]{0.32\textwidth}
  \begin{tikzpicture}[scale=0.6] \begin{axis}[
		x label style = {font=\large},
		y label style = {font=\large},
    xlabel = {$n$},
    ylabel = {Approximation Ratio},
    xmin=7, xmax=17,
    ymin=0, ymax=1,
		legend cell align=left,
		legend style={at={(1,0.4)}},
		y tick label style={
        /pgf/number format/.cd,
        fixed,
        fixed zerofill,
        precision=2,
        /tikz/.cd
    },
		error bars/y dir=both,
    error bars/y explicit]
		\pgfplotstableread{CSV/quadratic/m=1.5n/uniform2_quadratic_ours3__.csv}\ours
		\pgfplotstablecreatecol[copy column from table={CSV/quadratic/m=1.5n/uniform2_quadratic_ours3_std__.csv}{[index] 0}] {error} {\ours}
		\addplot [blue, mark = *] table [x expr=8+2*\coordindex, y index=0, y error index=1] {\ours};
		\addlegendentry{Algorithm~\ref{alg:OfflineEmp}}
		\pgfplotstableread{CSV/quadratic/m=1.5n/uniform2_quadratic_theirs__.csv}\theirs
		\pgfplotstablecreatecol[copy column from table={CSV/quadratic/m=1.5n/uniform2_quadratic_theirs_std__.csv}{[index] 0}] {error} {\theirs}
		\addplot [red, mark = triangle*, mark size=3pt] table [x expr=8+2*\coordindex, y index=0, y error index = 1] {\theirs};
		\addlegendentry{Mualem and Feldman~\cite{mualem2023resolving}}
		\pgfplotstableread{CSV/quadratic/m=1.5n/uniform2_quadratic_theirs2__.csv}\theirstwo
		\pgfplotstablecreatecol[copy column from table={CSV/quadratic/m=1.5n/uniform2_quadratic_theirs2_std__.csv}{[index] 0}] {error} {\theirstwo}
		\addplot [green!75!black, mark = square*] table [x expr=8+2*\coordindex, y index=0, y error index = 1] {\theirstwo};
		\addlegendentry{Bian et al.~\cite{bian2017nonmonotone}}
	\end{axis}\end{tikzpicture}
  \caption{$m=\lfloor1.5n\rfloor$}\label{fig:quad3}
\end{subfigure}
\caption{Quadratic Programming with Uniform Distribution} \label{fig:quad11}
\end{figure}

\subsubsection{Exponential Distribution} \label{ssc:exponential}

In this section, we use a different method for choosing the matrices $\mH$ and $\mA$ that utilizes exponential distributions. For every value $\lambda > 0$, the exponential distribution $\exp(\lambda)$ is defined by a density function assigning a density of $\lambda e^{-\lambda y}$ for $y \geq 0$, and a density of $0$ for $y < 0$. Given this definition, $\mH \in \npRE{n\times n}$ is a symmetric matrix randomly generated with entries drawn independently from $-\exp(1)$, while $\mA \in \nnRE{m \times n}$ is a randomly generated matrix with entries drawn independently from $\exp(0.25) + 0.01$.

For this method of generating $\mH$ and $\mA$, we conducted the same set of experiments as for the previous approach. The results of these experiments, again averaged over $100$ independently chosen instances, are illustrated in Figure~\ref{fig:quad12}. Our proposed algorithm demonstrates slightly better performance than the non-monotone Frank-Wolfe algorithm suggested by Bian et al.~\cite{bian2017nonmonotone}, and surpasses the algorithm suggested by Mualem \& Feldman~\cite{mualem2023resolving}.

\begin{figure}[tb]
\begin{subfigure}[t]{0.32\textwidth}
  \begin{tikzpicture}[scale=0.6] \begin{axis}[
		x label style = {font=\large},
		y label style = {font=\large},
    xlabel = {$n$},
    ylabel = {Approximation Ratio},
    xmin=7, xmax=17,
    ymin=0, ymax=1,
		legend cell align=left,
		legend style={at={(1,0.4)}},
		y tick label style={
        /pgf/number format/.cd,
        fixed,
        fixed zerofill,
        precision=2,
        /tikz/.cd
    },
		error bars/y dir=both,
    error bars/y explicit]
		\pgfplotstableread{CSV/quadratic/m=0.5n/exponential2_quadratic_ours3__.csv}\ours
		\pgfplotstablecreatecol[copy column from table={CSV/quadratic/m=0.5n/exponential2_quadratic_ours3_std__.csv}{[index] 0}] {error} {\ours}
		\addplot [blue, mark = *] table [x expr=8+2*\coordindex, y index=0, y error index=1] {\ours};
		\addlegendentry{Algorithm~\ref{alg:OfflineEmp}}
		\pgfplotstableread{CSV/quadratic/m=0.5n/exponential2_quadratic_theirs__.csv}\theirs
		\pgfplotstablecreatecol[copy column from table={CSV/quadratic/m=0.5n/exponential2_quadratic_theirs_std__.csv}{[index] 0}] {error} {\theirs}
		\addplot [red, mark = triangle*, mark size=3pt] table [x expr=8+2*\coordindex, y index=0, y error index = 1] {\theirs};
		\addlegendentry{Mualem and Feldman~\cite{mualem2023resolving}}
		\pgfplotstableread{CSV/quadratic/m=0.5n/exponential2_quadratic_theirs2__.csv}\theirstwo
		\pgfplotstablecreatecol[copy column from table={CSV/quadratic/m=0.5n/exponential2_quadratic_theirs2_std__.csv}{[index] 0}] {error} {\theirstwo}
		\addplot [green!75!black, mark = square*] table [x expr=8+2*\coordindex, y index=0, y error index = 1] {\theirstwo};
		\addlegendentry{Bian et al.~\cite{bian2017nonmonotone}}
	\end{axis}\end{tikzpicture}
  \caption{$m=\lfloor0.5n\rfloor$}
\end{subfigure}\hfill
\begin{subfigure}[t]{0.32\textwidth}
  \begin{tikzpicture}[scale=0.6] \begin{axis}[
		x label style = {font=\large},
		y label style = {font=\large},
    xlabel = {$n$},
    ylabel = {Approximation Ratio},
    xmin=7, xmax=17,
    ymin=0, ymax=1,
		legend cell align=left,
		legend style={at={(1,0.4)}},
		y tick label style={
        /pgf/number format/.cd,
        fixed,
        fixed zerofill,
        precision=2,
        /tikz/.cd
    },
		error bars/y dir=both,
    error bars/y explicit]
		\pgfplotstableread{CSV/quadratic/m=n/exponential2_quadratic_ours3__.csv}\ours
		\pgfplotstablecreatecol[copy column from table={CSV/quadratic/m=n/exponential2_quadratic_ours3_std__.csv}{[index] 0}] {error} {\ours}
		\addplot [blue, mark = *] table [x expr=8+2*\coordindex, y index=0, y error index=1] {\ours};
		\addlegendentry{Algorithm~\ref{alg:OfflineEmp}}
		\pgfplotstableread{CSV/quadratic/m=n/exponential2_quadratic_theirs__.csv}\theirs
		\pgfplotstablecreatecol[copy column from table={CSV/quadratic/m=n/exponential2_quadratic_theirs_std__.csv}{[index] 0}] {error} {\theirs}
		\addplot [red, mark = triangle*, mark size=3pt] table [x expr=8+2*\coordindex, y index=0, y error index = 1] {\theirs};
		\addlegendentry{Mualem and Feldman~\cite{mualem2023resolving}}
		\pgfplotstableread{CSV/quadratic/m=n/exponential2_quadratic_theirs2__.csv}\theirstwo
		\pgfplotstablecreatecol[copy column from table={CSV/quadratic/m=n/exponential2_quadratic_theirs2_std__.csv}{[index] 0}] {error} {\theirstwo}
		\addplot [green!75!black, mark = square*] table [x expr=8+2*\coordindex, y index=0, y error index = 1] {\theirstwo};
		\addlegendentry{Bian et al.~\cite{bian2017nonmonotone}}
	\end{axis}\end{tikzpicture}
  \caption{$m=n$}
\end{subfigure}\hfill
\begin{subfigure}[t]{0.32\textwidth}
  \begin{tikzpicture}[scale=0.6] \begin{axis}[
		x label style = {font=\large},
		y label style = {font=\large},
    xlabel = {$n$},
    ylabel = {Approximation Ratio},
    xmin=7, xmax=17,
    ymin=0, ymax=1,
		legend cell align=left,
		legend style={at={(1,0.4)}},
		y tick label style={
        /pgf/number format/.cd,
        fixed,
        fixed zerofill,
        precision=2,
        /tikz/.cd
    },
		error bars/y dir=both,
    error bars/y explicit]
		\pgfplotstableread{CSV/quadratic/m=1.5n/exponential2_quadratic_ours3__.csv}\ours
		\pgfplotstablecreatecol[copy column from table={CSV/quadratic/m=1.5n/exponential2_quadratic_ours3_std__.csv}{[index] 0}] {error} {\ours}
		\addplot [ blue, mark = *] table [x expr=8+2*\coordindex, y index=0, y error index=1] {\ours};
		\addlegendentry{Algorithm~\ref{alg:OfflineEmp}}
		\pgfplotstableread{CSV/quadratic/m=1.5n/exponential2_quadratic_theirs__.csv}\theirs
		\pgfplotstablecreatecol[copy column from table={CSV/quadratic/m=1.5n/exponential2_quadratic_theirs_std__.csv}{[index] 0}] {error} {\theirs}
		\addplot [red, mark = triangle*, mark size=3pt] table [x expr=8+2*\coordindex, y index=0, y error index = 1] {\theirs};
		\addlegendentry{Mualem and Feldman~\cite{mualem2023resolving}}
		\pgfplotstableread{CSV/quadratic/m=1.5n/exponential2_quadratic_theirs2__.csv}\theirstwo
		\pgfplotstablecreatecol[copy column from table={CSV/quadratic/m=1.5n/exponential2_quadratic_theirs2_std__.csv}{[index] 0}] {error} {\theirstwo}
		\addplot [green!75!black, mark = square*] table [x expr=8+2*\coordindex, y index=0, y error index = 1] {\theirstwo};
		\addlegendentry{Bian et al.~\cite{bian2017nonmonotone}}
	\end{axis}\end{tikzpicture}
  \caption{$m=\lfloor1.5n\rfloor$}
\end{subfigure}
\caption{Quadratic Programming with Exponential Distribution} \label{fig:quad12}
\end{figure}
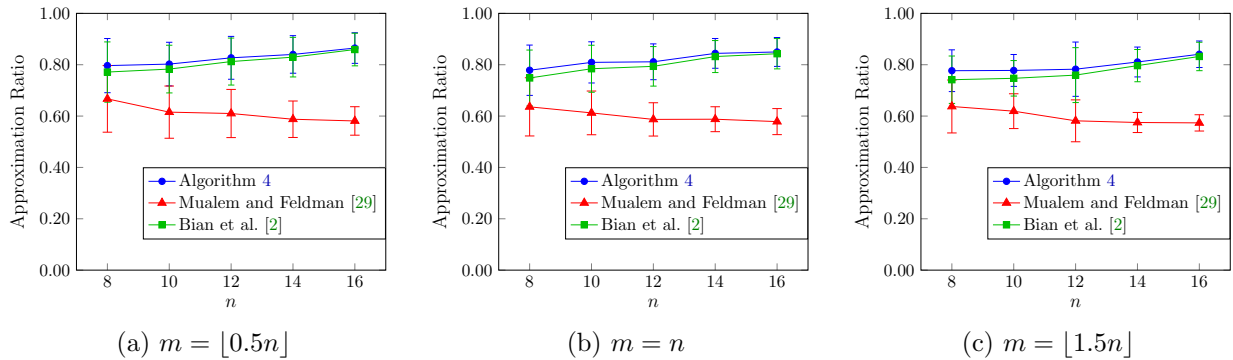

\section{Conclusion}

In this paper, we have presented novel offline and online algorithms for DR-submodular maximization subject to a general convex body constraint. Our algorithms are able to provide a smooth interpolation between the approximability of general and down-closed convex bodies by considering a decomposition of the convex body constraint into a down-closed convex body and a general convex body. In addition to giving a theoretical analysis of our algorithms, we have demonstrated their empirical superiority (compared to state-of-the-art methods) in various online and offline machine learning applications.

\appendix

\section{Implication for a Fairness Setting} \label{app:fairness}

El~Halabi et al.~\cite{el2023bfairness} considered the following fairness setting. The input for this setting is a ground set $\cN$, a non-negative discrete submodular function $f\colon 2^\cN \to \nnR$, a matroid $\cM = (\cN, \cI)$,\footnote{We refer the reader to~\cite{el2023bfairness} for the definition of matroids and the other terms used in this section.} a partition of the elements in $\cN$ into $k$ disjoint classes $C_1, C_2, \dotsc, C_k$ and integral lower and upper bounds $\ell_i$ and $u_i$ for each class $C_i$, respectively. The objective is to find a set $S \in \cI$ maximizing $f$ subject to fairness constraints requiring that $\ell_i \leq |S \cap C_i| \leq u_i$ for every class $C_i$. If one only requires the fairness constraints to hold in expectation, then El~Halabi et al.~\cite{el2023bfairness} showed that is possible to reduce their setting to the following continues settings.\footnote{If the classes are large, then, in addition to guaranteeing that the fairness constraints hold in expectation, the reduction also guarantees that, with high probability, each fairness constraint is violated by at most a small amount.} In this continuous settings, one has to find a vector $\vx$ in the matroid polytope $P_\cM$ that maximizing the multilinear extension $F$ of $f$ subject to fairness constraints requiring that $\ell_i \leq \|\vx \cap \chi_{C_i}\|_1 \leq u_i$ for every class $C_i$, where $\chi_{C_i}$ is the characteristic vector of the set $C_i$. Since multilinear extensions of discrete submodular functions are DR-submodular~\cite{bian2017guaranteed}, our offline results can be applied to this continues setting.

Some of the results of El~Halabi et al.~\cite{el2023bfairness} are bi-criteria approximation algorithms that are allowed to output solutions $\vx \in P_\cM$ that, for some parameter $\beta \in (0, 1)$, only obey $\beta \ell_i \leq \|\vx \cap \chi_{C_i}\|_1 \leq u_i$ for every class $C_i$. In our terminology, allowing such solutions implies two things.
\begin{itemize}
	\item $F(\optpone)$ and $F(\optptwo)$ can be both made to be at least $(1 - \beta) \cdot F(\vo)$, where $\vo$ is the optimal solution; and
	\item $m = \beta r$, where $r = \min_{\vx \in P_\cM} \|\vx\|_\infty$.
\end{itemize}
Thus, our offline algorithms can be used to get a solution whose approximation ratio is at least\footnote{We omitted the error term in the approximation ratio. Since the smoothness of multilinear extensions of discrete submodular functions is polynomial, this error term can be made an arbitrarily small constant. For details, see, for example, Appendix~A of~\cite{buchbinder2023constrained}.}
\begin{align*}
	(1 - \beta r) \cdot \max_{t_s \in [0, 1]} \max_{T \in [t_s, 1]} &\bigg\{(T - t_s) e^{-T} \cdot (1 - \beta) + \frac{t_s^2\cdot e^{-t_s - T}}{2} \cdot (1 - \beta) + e^{-T}-e^{-t_s - T}\bigg\}
	\enspace.
\end{align*}

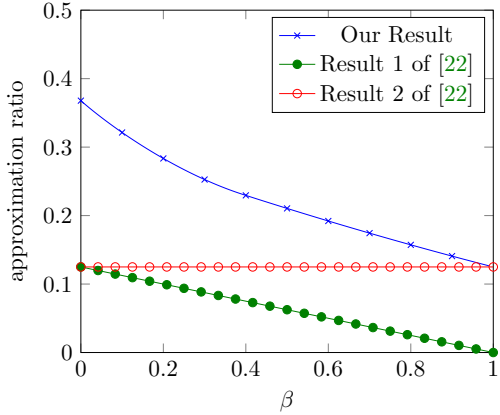
\begin{SCfigure}[][tb]
\begin{tikzpicture}[scale=0.8]
\begin{axis}[
    xlabel = {$\beta$},
    ylabel = {approximation ratio},
    xmin=0, xmax=1,
    ymin=0, ymax=0.5]

\addplot [name path = our, blue, mark = x, mark repeat=10] table [x index=0, y index=1, col sep=comma] {CSV/Fairness.csv};
\addlegendentry{Our Result}

\addplot [name path = A, domain = 0:1, darkgreen, mark = *] {(1 - x)/8};
\addlegendentry{Result $1$ of~\cite{el2023bfairness}}
 
\addplot [name path = B, domain = 0:1, red, mark = o] {1/8};
\addlegendentry{Result $2$ of~\cite{el2023bfairness}}

\end{axis}
\end{tikzpicture}
\caption{Results for the Fairness Setting of~\cite{el2023bfairness} (for $r = 1/2$)} \label{fig:results_fairness}
\end{SCfigure}

In Figure~\ref{fig:results_fairness}, we compare the above approximation ratio with two results of~\cite{el2023bfairness} (for $r = 1/2$). The first result is an approximation ratio of $(1 - \beta)/8$, which is worse than the approximation ratio that we obtain for any $\beta$, but is not based on the above mentioned reduction to the continuous setting (and thus, guarantees that the fairness constraints $\beta \ell_i \leq \|\vx \cap \chi_{C_i}\|_1 \leq u_i$ hold always, and not just in expectation). The second result of~\cite{el2023bfairness} used in our comparison is an approximation of $(1 - r)/4$, which applies even for $\beta = 1$ (but like our result, satisfies the fairness constraints only in expectation). Our result can be viewed as a generalization of this second result.
\section{Empirical Improvements of Algorithm~\ref{alg:OfflineGF}} \label{app:practical_offline}

In this section, we present a version of Algorithm~\ref{alg:OfflineGF} (appears as Algorithm~\ref{alg:OfflineEmp}) that includes modifications designed to improve the empirical performance of the algorithm. These modifications were not included in the original Algorithm~\ref{alg:OfflineGF} because they do not affect the algorithm's theoretical guarantee, and furthermore, they cannot be applied to the online counterpart of Algorithm~\ref{alg:OfflineGF} (i.e., Algorithm~\ref{alg:Online}).

\begin{algorithm}[ht]
\DontPrintSemicolon
Let $\vy^{(0)} \leftarrow \argmin_{\vx\in \NDM}\norm{\vx}{}$, $\vz^{(0)} \gets \vzero$ and $m \gets \|\vy^{(0)}\|_\infty$.\\
\For{$i=1$ \KwTo $\eps^{-1}$}
{
	\If{$i \leq \eps^{-1} t_s$}
	{
	Solve the following linear program. The variables in this program are the vectors $\va^{(i)}$, $\vb^{(i)}$ and $\vc^{(i)}$.
    \begin{alignat*}{2}
  & \text{maximize}   & \quad & e^{2\eps i} \cdot \inner{\nabla F(\comb{i-1})}{(\vone-\vy^{(i - 1)})\odot (\vb^{(i)} - \vc^{(i)}) + (\vone-\vz^{(i - 1)}) \odot \va^{(i)}} \\
	&										&& \mspace{140mu}+(1 - m) \cdot e^{\eps i}(t_s - \eps i) \cdot \inner{\nabla F(\vz^{(i-1)})}{ \vb^{(i)} - \vc^{(i)}}         \nonumber \\
  & \text{subject to} &       & \va^{(i)} \in \NDM \\ \nonumber
  &                   &       & \vb^{(i)} \in  \DM\\ \nonumber
  &										&				& \vb^{(i)} \leq (\vone - \vz^{(i - 1)}) \odot (\vone - \va^{(i)})\\
	&										&				& \vzero \leq \vc^{(i)} \leq \vz^{(i - 1)}
\end{alignat*}
\vspace{-5mm}}
\Else
{

Solve the following linear program. The variables in this program are the vectors $\va^{(i)}$, $\vb^{(i)}$ and $\vc^{(i)}$.
    \begin{alignat*}{2}
  & \text{maximize}   & \quad & \langle\grad{F(\comb{i-1}) \odot (\vone-\vy^{(i - 1)}),\vb^{(i)} -\vc^{(i)}}\rangle         \nonumber \\
  & \text{subject to} &       & \va^{(i)} = \vy^{(i)} \\ \nonumber
  &                   &       & \vb^{(i)} \in  \DM\\
	&										&				& \vb^{(i)} \leq \vone - \vz^{(i - 1)}\\
	&										&				& \vzero \leq \vc^{(i)} \leq \vz^{(i - 1)}
\end{alignat*}
\vspace{-5mm}}
    Let $\vy^{(i)}\leftarrow(1-\eps)\cdot\vy^{(i-1)}+\eps\cdot\va^{(i)}$.\label{line:y_update_emp}\\
    Let $\vz^{(i)}\leftarrow\vz^{(i-1)} + \eps\cdot(\vb^{(i)} - \vc^{(i)})$.
}
\Return a vector maximizing $F$ among all the vectors in $\{\comb{i} \mid i \in \bZ, 0 \leq i \leq \eps^{-1}\}$.
\caption{\texttt{Frank-Wolfe/Continuous-Greedy Hybrid with Empirical Improvements}\label{alg:OfflineEmp}}
\end{algorithm}

To be more concrete, the modifications done in Algorithm~\ref{alg:OfflineEmp} compared to Algorithm~\ref{alg:OfflineGF} are the following.
\begin{itemize}
	\item Algorithm~\ref{alg:OfflineGF} was designed with the view that $\optpone$ and $\optptwo$ should be a feasible assignments to $\vb^{(i)}$, and then the vector $\vz^{(i)}$ is increased proportionally to $(\vone - \vz^{(i - 1)}) \odot \vb^{(i)}$. The multiplication by $(1 - \vz^{(i - 1)})$ was done to reduce the speed in which $\|\vz^{(i)}\|_\infty$ increases as a function of $i$. Following the work of Bian et al.~\cite{bian2017nonmonotone}, Algorithm~\ref{alg:OfflineEmp} was modified to make $(\vone - \vz^{(i - 1)}) \cdot \optpone$ and $(\vone - \vz^{(i - 1)}) \cdot \optptwo$ natural assignments for $\vb^{(i)}$. This was done by allowing $\vb^{(i)}$ to be any vector in $\DM$ which is upper bounded by $(\vone - \vz^{(i - 1)}) \odot (\vone - \va^{(i)})$, and making the vector $\vz^{(i)}$ increase proportionally to the vector $\vb^{(i)}$ itself. Note that these modifications preserve the bound on the rate in which $\|\vz^{(i)}\|_\infty$ increases, but give the algorithm more flexability as the increase in $\vz^{(i)}$ no longer has to be equal to $\vone - \vz^{(i - 1)}$ times some vector in $\DM$.
	\item Algorithm~\ref{alg:OfflineGF} can both increase and decrease $\vy^{(i)}$. However, a similar flexibility does not exist for the vector $\vz^{(i)}$, which is maintained in a continuous-greedy like fashion. Algorithm~\ref{alg:OfflineEmp} introduces a new vector $\vc^{(i)}$ that is used to decrease coordinates of $\vz^{(i)}$ when this helps the objective (or the potential function when $i \leq \eps^{-1}t_s$).
	\item Algorithm~\ref{alg:OfflineGF} returns the best solution for any $\eps^{-1} \cdot t_s \leq i \leq \eps^{-1}$. Algorithm~\ref{alg:OfflineEmp} returns the best solution obtained after any number of iterations because that is always at least as good.
\end{itemize}

The rest of this section is devoted to proving that, like Algorithm~\ref{alg:OfflineGF}, Algorithm~\ref{alg:OfflineEmp} has the properties guaranteed by Theorem~\ref{thm:Offline}. We begin with the following lemma, which is a counterpart of Lemma~\ref{lem:membership}, and proves that Algorithm~\ref{alg:OfflineEmp} returns a feasible solution.

\begin{lemma} \label{lem:membership_emp}
For every integer $0 \leq i \leq \eps^{-1}$, $\vy^{(i)} \in \NDM$ and $\vz^{(i)} \in \eps i \cdot \DM$. Hence, $\comb{i} \in (\NDM + \DM) \cap [0, 1]^n$.
\end{lemma}
\begin{proof}
Given the proof of Lemma~\ref{lem:membership}, to prove the current lemma we only need to show that $\vz^{(i)} \in \eps i \cdot \DM$, which we do by induction on $i$. Clearly, $\vz^{(0)} = \vzero \in 0 \cdot \DM$. Assume now that $\vz^{(i - 1)} \in \eps (i - 1) \cdot \DM$ for some $i \in [\eps^{-1}]$, and let us prove that $\vz^{(i)} \in \eps i \cdot \DM$.

Since $\vz^{(i - 1)} \in \eps (i - 1) \cdot \DM$, there must exist a vector $\vx \in \DM$ such that $\vz^{(i - 1)} = \eps(i - 1) \cdot \vx$. Therefore,
\begin{equation} \label{eq:z_development}
	\vz^{(i)}
	=
	\vz^{(i - 1)} + \eps (\vb^{(i)} - \vc^{(i)})
	=
	\eps i \cdot \left(\frac{i - 1}{i} \cdot \vx + \frac{1}{i} \cdot \vb^{(i)} - \frac{1}{i} \cdot \vc^{(i)} \right)
	\enspace.
\end{equation}
The convexity of $\DM$ implies that $\frac{i - 1}{i} \cdot \vx + \frac{1}{i} \cdot \vb^{(i)}$ is a vector in $\DM$. This vector upper bounds $\frac{1}{i} \cdot \vc^{(i)}$ because
\[
	\frac{1}{i} \cdot \vc^{(i)}
	\leq
	\frac{1}{i} \cdot \vz^{(i - 1)}
	=
	\frac{\eps(i - 1)}{i} \cdot \vx
	\leq
	\frac{i - 1}{i} \cdot \vx
	\enspace.
\]
Therefore, Equation~\eqref{eq:z_development} and the down-closeness of $\DM$ imply together that $\vz^{(i)} \in \eps i \cdot \DM$.
\end{proof}

We now divert our attention to proving the approximation guarantee of Algorithm~\ref{alg:OfflineEmp}. It turns out that most of the analysis of the approximation guarantee of Algorithm~\ref{alg:OfflineGF} from Section~\ref{ssc:unknown_optp} can be reused to prove the same guarantee for Algorithm~\ref{alg:OfflineEmp}, except for the proofs of three lemmata: Lemma~\ref{lem:normxy}, Lemma~\ref{lem:basic_lower_bound} and Lemma~\ref{lem:increase_second_part}, which we reprove below in the context of this algorithm. We begin by proving that Lemma~\ref{lem:normxy} still holds in the context of Algorithm~\ref{alg:OfflineEmp}.
\lemNormXY*
\begin{proof}
We prove the lemma by induction on $i$. For $i = 0$, we have $\|\vz^{(0)}\|_\infty = \|\vzero\|_\infty = 0 = 1 - (1 - \eps)^0$ and $\|\comb{0}\|_\infty = \|\vy^{(0)}\|_\infty = m = 1 - (1 - \eps)^0(1 - m)$. Assume now that $\|\vz^{(i - 1)}\|_\infty \leq 1 - (1 - \eps)^{i - 1}$ and $\|\comb{i - 1}\|_\infty \leq 1 - (1 - \eps)^{i - 1}(1 - m)$ for some $i \in [\eps^{-1}]$, and let us prove the corresponding claim for $i$.

Observe that
\begin{align*}
	\|\vz^{(i)}\|_\infty
	={} &
	\|\vz^{(i - 1)} + \eps \cdot (\vb^{(i)} - \vc^{(i)})\|_\infty
	\leq
	\|\vz^{(i - 1)} + \eps \cdot \vb^{(i)}\|_\infty\\
	\leq{} &
	\|\vz^{(i - 1)} + \eps \cdot (\vone - \vz^{(i - 1)})\|_\infty
	\leq
	\|(1 - \eps) \cdot \vz^{(i - 1)}\|_\infty + \eps \cdot \|\vone\|_\infty
	=
	(1 - \eps) \cdot \|\vz^{(i - 1)}\|_\infty + \eps\\
	\leq{} &
	(1 - \eps) \cdot [1 - (1 - \eps)^{i - 1}] + \eps
	=
	1 - (1 - \eps)^i
	\enspace,
\end{align*}
where the first inequality holds since $\vc^{(i)}$ is non-negative, the second inequality holds since both linear programs of Algorithm~\ref{alg:OfflineEmp} require $\vb^{(i)}$ to be coordinate-wise upper bounded by $(\vone - \vz^{(i - 1)}) \odot (\vone - \va^{(i)}) \leq \vone - \vz^{(i - 1)}$, and the last inequality follows from the induction hypothesis.

Similarly,
\begin{align*}
	\|\comb{i}\|_\infty
	={} &
	\|((1 - \eps) \cdot \vy^{(i - 1)} + \eps \cdot \va^{(i)}) \psum (\vz^{(i - 1)} + \eps \cdot (\vb^{(i)} - \vc^{(i)}))\|_\infty\\
	\leq{} &
	\|((1 - \eps) \cdot \vy^{(i - 1)} + \eps \cdot \va^{(i)}) \psum (\vz^{(i - 1)} + \eps \cdot \vb^{(i)})\|_\infty\\
	\leq{} &
	\|((1 - \eps) \cdot \vy^{(i - 1)} + \eps \cdot \va^{(i)}) \psum (\vz^{(i - 1)} + \eps \cdot (\vone - \vz^{(i - 1)}) \odot (\vone - \va^{(i)}))\|_\infty\\
	\leq{} &
	\|(\vy^{(i - 1)} \psum (\eps \cdot \va^{(i)})) \psum (\vz^{(i - 1)} \psum (\eps(\vone - \va^{(i)})))\|_\infty
	\leq
	\|\vy^{(i - 1)} \psum \vz^{(i - 1)} \psum (\eps \cdot \vone)\|_\infty\\
	={} &
	\|\eps \cdot \vone\|_\infty + (1 - \eps) \cdot \|\vy^{(i - 1)} \psum \vz^{(i - 1)}\|_\infty\\
	\leq{} &
	\eps + (1 - \eps) \cdot [1 - (1 - \eps)^{i - 1}(1 - m)]
	=
	1 - (1 - \eps)^i(1 - m)
	\enspace.
	\qedhere
\end{align*}
\end{proof}

As the proof that Lemma~\ref{lem:basic_lower_bound} applies to Algorithm~\ref{alg:OfflineEmp} is somewhat long, we defer it a bit, and prove first that Lemma~\ref{lem:increase_second_part} applies in the context of this algorithm. To be completely honest, we prove here a modified version of Lemma~\ref{lem:increase_second_part} that has a slightly larger error term ($\eps^2 \beta D^2$ instead of $\eps^2 \beta D^2/2$). However, this affects the proof of Theorem~\ref{thm:Offline} only by changing the constants hidden by the big $O$ notation.
\toggletrue{AppendixEmp}
\lemIncreaseSecondPart*
\begin{proof}
By the chain rule,
\begin{align} \label{eq:chain_rule_emp}
	F(\comb{i}\mspace{-30mu}&\mspace{30mu}) - F(\comb{i - 1})
	=
	\int_0^\eps \frac{dF(\vy^{(i - 1)} \psum (\vz^{(i - 1)} + \tau \cdot (\vb^{(i)} - \vc^{(i)})))}{d\tau} d\tau\\\nonumber
	={} &
	\int_0^\eps \inner{(\vb^{(i)} - \vc^{(i)}) \odot (\vone - \vy^{(i - 1)})}{\nabla F(\vy^{(i - 1)} \psum (\vz^{(i - 1)} + \tau \cdot (\vb^{(i)} - \vc^{(i)})))} d\tau
	\enspace.
\end{align}
We would like to lower bound the integrand on the rightmost side of the last equality. We do that by lower bounding two expressions whose sum is equal to this integrand. The first expression is
\begin{align*}
	&
	\langle (\vb^{(i)} - \vc^{(i)}) \odot (\vone - \vy^{(i - 1)}),\\&\mspace{180mu}\nabla F(\vy^{(i - 1)} \psum (\vz^{(i - 1)} + \tau \cdot (\vb^{(i)} - \vc^{(i)}))) - \nabla F(\comb{i-1})\rangle\\
	\geq{} &
	-\|(\vb^{(i)} - \vc^{(i)}) \odot (\vone - \vy^{(i - 1)})\|_2 \\&\mspace{180mu}\cdot \|\nabla F(\vy^{(i - 1)} \psum (\vz^{(i - 1)} + \tau \cdot (\vb^{(i)} - \vc^{(i)}))) - \nabla F(\comb{i-1})\|_2\\
	\geq{} &
	-\tau\beta\|(\vb^{(i)} - \vc^{(i)}) \odot (\vone - \vy^{(i - 1)})\|_2^2
	\geq
	-2\tau\beta D^2
	\enspace,
\end{align*}
where the first inequality holds by the Cauchy–Schwarz inequality, the second inequality holds due to the $\beta$-smoothness of $F$, and the last inequality holds since
\begin{align*}
	\|(\vb^{(i)}&{} - \vc^{(i)}) \odot (\vone - \vy^{(i - 1)})\|_2
	\leq
	\|\vb^{(i)} \odot (\vone - \vy^{(i - 1)})\|_2 + \|\vc^{(i)} \odot (\vone - \vy^{(i - 1)})\|_2\\
	={} &
	\|(\vy^{(i - 1)} + \vb^{(i)} \odot (\vone - \vy^{(i - 1)})) - \vy^{(i - 1)}\|_2 + \|(\vy^{(i - 1)} + \vc^{(i)} \odot (\vone - \vy^{(i - 1)})) - \vy^{(i - 1)}\|_2
	\leq
	2D
	\enspace.
\end{align*}
Note that the last inequality holds because the fact that $\vc^{(i)} \leq \vz^{(i - 1)} \in \DM$ implies that $\vy^{(i - 1)} + \vb^{(i)} \odot (\vone - \vy^{(i - 1)})$, $\vy^{(i - 1)} + \vc^{(i)} \odot (\vone - \vy^{(i - 1)})$ and $\vy^{(i - 1)}$ are all vectors of $(\NDM + \DM) \cap [0, 1]^n$. The second expression (of the two mentioned above) is
\begin{align*}
	\langle (\vb^{(i)} - \vc^{(i)}) \odot (\vone - \vy^{(i - 1)}), \nabla F(\vy^{(i - 1)} \psum{}& \vz^{(i - 1)})\rangle
	=
	\inner{\vb^{(i)} - \vc^{(i)}}{\nabla F(\vy^{(i - 1)} \psum \vz^{(i - 1)}) \odot (\vone - \vy^{(i - 1)})}\\
	\geq{} &
	\inner{\optptwo \odot (\vone - \vz^{(i - 1)})}{\nabla F(\vy^{(i - 1)} \psum \vz^{(i - 1)})\odot (\vone - \vy^{(i - 1)})}\\
	\geq{} &
	F(\optptwo \psum \comb{i - 1}) - F(\comb{i - 1})\\
	\geq{} &
	(1 - \eps)^{i-1}(1 - m) \cdot F(\optptwo) - F(\comb{i - 1})
	\enspace,
\end{align*}
where the first inequality holds since the second side of the above inequality is identical to the objective function of the second linear program of Algorithm~\ref{alg:OfflineEmp}, and one feasible solution for this linear program is $\vb^{(i)} = \optptwo \odot (\vone - \vz^{(i - 1)})$ and $\vc^{(i)} = \vzero$. The second inequality follows from Property~\eqref{prop:dr_bound2_up} of Lemma~\ref{lem:DR_properties}, and the last inequality follows from Corollary~\ref{cor:norm_bound} and Lemma~\ref{lem:normxy}. The lemma now follows by plugging the lower bounds we have proved into the rightmost side of Equality~\eqref{eq:chain_rule_emp}, and then solving the integral obtained.
\end{proof}

It remains to show that Lemma~\ref{lem:basic_lower_bound} applies to Algorithm~\ref{alg:OfflineEmp}. The next two lemmata are steps toward this goal.

\begin{lemma} \label{lem:comb_diff_bound}
For every $i \in [\eps^{-1} t_s]$,
\begin{align*}
	e^{2\eps i}\cdot F(&\comb{i}) - e^{2\eps (i - 1)}\cdot F(\comb{i - 1})\\
	\geq{} &
	2\eps e^{2\eps (i - 1)} \cdot F(\comb{i-1}) + \eps e^{2\eps i} \cdot \langle \nabla F(\comb{i-1}),  (\vone - \vy^{(i - 1)}) \odot (\vb^{(i)} - \vc^{(i)}) \\\nonumber&\mspace{150mu} + (\va^{(i)} - \vy^{(i - 1)}) \odot (\vone - \vz^{(i - 1)} - \eps \cdot (\vb^{(i)} - \vc^{(i)})) \rangle -56\eps^2\beta D^2
	\enspace.
\end{align*}
\end{lemma}
\begin{proof}
By the chain rule,
\begin{align} \label{eq:increase_comb_emp}
F(\comb{i})&{}-F(\comb{i-1})\\\nonumber
={} &
F((\term{1-\eps}\vy^{(i-1)}+\eps\va^{(i)}) \psum (\vz^{(i-1)}+\eps(\vb^{(i)} - \vc^{(i)})))-F(\comb{i-1})\\\nonumber
={} &
\int_0^\eps \Big\langle \frac{d[((1-\tau)\vy^{(i-1)}+\tau\va^{(i)}) \psum (\vz^{(i-1)}+\tau(\vb^{(i)} - \vc^{(i)}))]}{d\tau}, \\\nonumber&\mspace{150mu} \nabla F(((1-\tau)\vy^{(i-1)}+\tau\va^{(i)}) \psum (\vz^{(i-1)}+\tau(\vb^{(i)} - \vc^{(i)})))\Big\rangle d\tau\\\nonumber
={} &
\int_0^\eps \langle(\vone - \vy^{(i - 1)}) \odot (\vb^{(i)} - \vc^{(i)}) + (\va^{(i)} - \vy^{(i - 1)}) \odot (\vone - \vz^{(i - 1)} - 2\tau \cdot (\vb^{(i)} - \vc^{(i)})), \\[-2mm]\nonumber&\mspace{150mu} \nabla F(((1-\tau)\vy^{(i-1)}+\tau\va^{(i)}) \psum (\vz^{(i-1)}+\tau(\vb^{(i)} - \vc^{(i)})))\rangle d\tau
\enspace.
\end{align}
The following inequality is used later to lower bound the integrand on the rightmost side of the last equality.
\begin{align} \label{eq:part1_emp}
&
\langle(\vone - \vy^{(i - 1)}) \odot (\vb^{(i)} - \vc^{(i)}) + (\va^{(i)} - \vy^{(i - 1)}) \odot (\vone - \vz^{(i - 1)} - 2\tau \cdot (\vb^{(i)} - \vc^{(i)})), \\&\mspace{100mu}\nonumber \nabla F(((1-\tau)\vy^{(i-1)}+\tau\va^{(i)}) \psum (\vz^{(i-1)}+\tau(\vb^{(i)} - \vc^{(i)}))) - \nabla F(\comb{i - 1})\rangle\\\nonumber
\geq{}&
-\|(\vone - \vy^{(i - 1)}) \odot (\vb^{(i)} - \vc^{(i)}) + (\va^{(i)} - \vy^{(i - 1)}) \odot (\vone - \vz^{(i - 1)} - 2\tau \cdot (\vb^{(i)} - \vc^{(i)}))\|_2 \cdot \\&\mspace{100mu}\nonumber \|\nabla F(((1-\tau)\vy^{(i-1)}+\tau\va^{(i)}) \psum (\vz^{(i-1)}+\tau(\vb^{(i)} - \vc^{(i)}))) - \nabla F(\comb{i - 1})\|_2\\\nonumber
\geq{} &
-5D \cdot \beta \|((1-\tau)\vy^{(i-1)}+\tau\va^{(i)}) \psum (\vz^{(i-1)}+\tau(\vb^{(i)} - \vc^{(i)})) - \comb{i - 1}\|_2\\\nonumber
={} &
-5\beta D \cdot \|\tau(\va^{(i)} - \vy^{(i - 1)}) \odot (\vone - \vz^{(i - 1)}) + \tau (\vb^{(i)} - \vc^{(i)}) \odot (\vone - \vy^{(i - 1)}) \\\nonumber&\mspace{300mu}- \tau^2(\va^{(i)} - \vy^{(i - 1)}) \odot (\vb^{(i)} - \vc^{(i)}) \|_2
\geq
-15\tau \beta D^2
\enspace,
\end{align}
where the first inequality follows from the Cauchy-Schwarz inequality, and the second inequality uses the $\beta$-smoothness of $F$ and the observation that since $\vy^{(i)} + (\vone - \vy^{(i - 1)}) \odot \vb^{(i)}$, $\vy^{(i)} + (\vone - \vy^{(i - 1)}) \odot \vc^{(i)}$, $\va^{(i)}$ and $\vy^{(i - 1)}$ are all vectors in $(\NDM + \DM) \cap [0, 1]^n$ (because $\vc^{(i)} \leq \vz^{(i - 1)}$ and $\DM$ is down-closed), it holds that
\begin{align*}
	\|(\vone - \vy^{(i - 1)}&) \odot (\vb^{(i)} - \vc^{(i)}) + (\va^{(i)} - \vy^{(i - 1)}) \odot (\vone - \vz^{(i - 1)} - 2\tau \cdot (\vb^{(i)} - \vc^{(i)}))\|_2\\
	\leq{} &
	\|(\vy^{(i-1)} + (\vone - \vy^{(i - 1)}) \odot \vb^{(i)}) - \vy^{(i - 1)}\|_2 + \|(\vy^{(i-1)} + (\vone - \vy^{(i - 1)}) \odot \vc^{(i)}) - \vy^{(i - 1)}\|_2 \\& - \|\va^{(i)} - \vy^{(i - 1)}\|_2 \odot \|\vone - \vz^{(i - 1)} - 2\tau \cdot (\vb^{(i)} - \vc^{(i)})\|_\infty
	\leq
	5D
	\enspace.
\end{align*}
Similarly, the last inequality of Inequality~\eqref{eq:part1_emp} holds since
\begin{align*}
	\|\tau(\va^{(i)} - {}&\vy^{(i - 1)}) \odot (\vone - \vz^{(i - 1)}) + \tau (\vb^{(i)} - \vc^{(i)}) \odot (\vone - \vy^{(i - 1)}) - \tau^2(\va^{(i)} - \vy^{(i - 1)}) \odot (\vb^{(i)} - \vc^{(i)}) \|_2\\
	={} &
	\|\tau(\va^{(i)} - \vy^{(i - 1)}) \odot (\vone - \vz^{(i - 1)}) \odot (1 - \tau (\vb^{(i)} - \vc^{(i)})) + \tau (\vb^{(i)} - \vc^{(i)}) \odot (\vone - \vy^{(i - 1)})\|_2\\
	\leq{} &
	\|\va^{(i)} - \vy^{(i - 1)}\|_2 + \|(\vy^{(i - 1)} + (\vone-\vy^{(i-1)}) \odot \vb^{(i)}) - \vy^{(i - 1)}\|_2 \\&\mspace{150mu}+ \|(\vy^{(i - 1)} + (\vone-\vy^{(i-1)}) \odot \vc^{(i)}) - \vy^{(i - 1)}\|_2
	\leq
	3D
	\enspace,
\end{align*}
where the penultimate inequality uses the fact that $\|(\vone - \vz^{(i - 1)}) \odot (\vone + \vc^{(i)})\|_\infty \leq \|(\vone - \vz^{(i - 1)}) \odot (\vone + \vz^{(i-1)})\|_\infty \leq 1$.

Combining Inequalities~\eqref{eq:increase_comb_emp} and~\eqref{eq:part1_emp} now yields
\begin{align*}
	e^{2\eps i}\cdot {}&F(\comb{i}) - e^{2\eps (i - 1)}\cdot F(\comb{i - 1})\\\nonumber
	={} &
	[e^{2\eps i} - e^{2\eps(i - 1)}] \cdot F(\comb{i-1}) + e^{2\eps i} \cdot [F(\comb{i}) - F(\comb{i - 1})]\\\nonumber
	\geq{} &
	\int_0^\eps \{2e^{2(\eps (i - 1) + \tau)} \cdot F(\comb{i-1}) + e^{2\eps i} \cdot \langle \nabla F(\comb{i-1}), (\vone - \vy^{(i - 1)}) \odot (\vb^{(i)} - \vc^{(i)}) \\[-2mm]\nonumber&\mspace{150mu} + (\va^{(i)} - \vy^{(i - 1)}) \odot (\vone - \vz^{(i - 1)} - 2\tau \cdot (\vb^{(i)} - \vc^{(i)})) \rangle -111\tau\beta D^2\}d\tau\\\nonumber
	\geq{} &
	2\eps e^{2\eps (i - 1)} \cdot F(\comb{i-1}) + \eps e^{2\eps i} \cdot \langle \nabla F(\comb{i-1}),  (\vone - \vy^{(i - 1)}) \odot (\vb^{(i)} - \vc^{(i)}) \\\nonumber&\mspace{150mu} + (\va^{(i)} - \vy^{(i - 1)}) \odot (\vone - \vz^{(i - 1)} - \eps \cdot (\vb^{(i)} - \vc^{(i)})) \rangle -56\eps^2\beta D^2
	\enspace,
\end{align*}
where the second inequality uses the non-negativity of $F$.
\end{proof}

\begin{lemma} \label{lem:z_guarantee_emp}
For every $i \in [\eps^{-1} t_s]$,
\begin{align*}
	e^{\eps i}(&t_s-\eps i)\cdot F(\vz^{(i)}) - e^{\eps (i - 1)}(t_s-\eps (i - 1))\cdot F(\vz^{(i - 1)})\\
	\geq{} &
	-\eps e^{\eps i}(1 - t_s + \eps i) \cdot F(\vz^{(i - 1)}) + \eps e^{\eps i}(t_s-\eps i)\cdot \inner{\vb^{(i)} - \vc^{(i)}}{\nabla F(\vz^{(i - 1)})} - 6\eps^2\beta D^2 / (1 - m)^2
	\enspace.
\end{align*}
\end{lemma}
\begin{proof}
By the chain rule,
\begin{align*}
	F(&\vz^{(i)}) - F(\vz^{(i - 1)})
	=
	\int_0^\eps \inner{\vb^{(i)} - \vc^{(i)}}{\nabla F(\vz^{(i - 1)} + \tau (\vb^{(i)} - \vc^{(i)}))} d\tau\\
	={} &
	\eps \cdot \inner{\vb^{(i)} - \vc^{(i)}}{\nabla F(\vz^{(i - 1)})} + \int_0^\eps \inner{\vb^{(i)} - \vc^{(i)}}{\nabla F(\vz^{(i - 1)} + \tau (\vb^{(i)} - \vc^{(i)})) - \nabla F(\vz^{(i - 1)})} d\tau\\
	\geq{} &
	\eps \cdot \inner{\vb^{(i)} - \vc^{(i)}}{\nabla F(\vz^{(i - 1)})} - \int_0^\eps \|\vb^{(i)} - \vc^{(i)}\|_2 \cdot \|\nabla F(\vz^{(i - 1)} + \tau (\vb^{(i)} - \vc^{(i)})) - \nabla F(\vz^{(i - 1)})\|_2 d\tau\\
	\geq{} &
	\eps \cdot \inner{\vb^{(i)} - \vc^{(i)}}{\nabla F(\vz^{(i - 1)})} - \beta \|\vb^{(i)} - \vc^{(i)}\|_2^2 \cdot \int_0^\eps \tau d\tau
	\geq
	\eps \cdot \inner{\vb^{(i)} - \vc^{(i)}}{\nabla F(\vz^{(i - 1)})} - \frac{2\eps^2 \beta D^2}{(1 - m)^2}
	\enspace,
\end{align*}
where the first inequality follows from the Cauchy–Schwarz inequality, the second inequality holds by the $\beta$-smoothness of $F$, and the last inequality uses the fact that both $\vb^{(i)}$ and $\vc^{(i)}$ are vectors in $\DM$ (in the case of $\vc^{(i)}$ this holds by the down-monotonicity of $\DM$ since $\vc^{(i)} \leq \vz^{(i - 1)}$), and thus, the $\ell_2$ norm of both these vectors is at most $D / (1 - m)$ according to the proof of Lemma~\ref{lem:z_value}.

Using the last inequality, we now get
\begin{align*}
	e^{\eps i}(&t_s-\eps i)\cdot F(\vz^{(i)}) - e^{\eps (i - 1)}(t_s-\eps (i - 1))\cdot F(\vz^{(i - 1)})\\\nonumber
	={} &
	[e^{\eps i}(t_s-\eps i) - e^{\eps (i - 1)}(t_s-\eps (i - 1))] \cdot F(\vz^{(i - 1)}) + e^{\eps i}(t_s-\eps i)\cdot [F(\vz^{(i)}) - F(\vz^{(i - 1)})]\\\nonumber
	\geq{} &
	\int_0^\eps e^{\eps (i - 1) + \tau}(t_s - 1 - \eps (i - 1) - \tau) \cdot F(\vz^{(i - 1)})d\tau + \eps e^{\eps i}(t_s-\eps i)\cdot \inner{\vb^{(i)} - \vc^{(i)}}{\nabla F(\vz^{(i - 1)})} \\[-2mm]\nonumber&\mspace{550mu}- 6\eps^2\beta D^2 / (1 - m)^2\\\nonumber
	\geq{} &
	-\eps e^{\eps i}(1 - t_s + \eps i) \cdot F(\vz^{(i - 1)}) + \eps e^{\eps i}(t_s-\eps i)\cdot \inner{\vb^{(i)} - \vc^{(i)}}{\nabla F(\vz^{(i - 1)})} - 6\eps^2\beta D^2 / (1 - m)^2
	\enspace,
\end{align*}
where  the second inequality uses $F$'s non-negativity. 
\end{proof}

We are now ready to prove Lemma~\ref{lem:basic_lower_bound} in the context of Algorithm~\ref{alg:OfflineEmp}. Like in the case of Lemma~\ref{lem:increase_second_part}, the version of Lemma~\ref{lem:basic_lower_bound} that we prove here has slightly worse error terms compared to the version from Section~\ref{ssc:unknown_optp}, but the difference again affects the proof of Theorem~\ref{thm:Offline} only by changing the constants hidden by the big $O$ notation as long as $\eps \leq 1/60$.

\lemBasicLowerBound*
\begin{proof}
Adding $(1 - m)(1 - \eps)$ times the guarantee of Lemma~\ref{lem:z_guarantee_emp} to the guarantee of Lemma~\ref{lem:comb_diff_bound}, we get
\begin{align} \label{eq:diff_first_bound_emp}
	\eps^{-1}&e^{2t_s}[\phi(i) - \phi(i - 1)]\\\nonumber
	\geq{}&
	2 e^{2\eps (i - 1)} \cdot F(\comb{i-1}) - (1 - m)(1 - \eps) \cdot e^{\eps i}(1 - t_s + \eps i) \cdot F(\vz^{(i - 1)}) - \tfrac{62\eps\beta D^2}{1 - m}\\ \nonumber
	&+ (1 \mspace{-1mu}-\mspace{-1mu} \eps) e^{2\eps i} \cdot \langle \nabla F(\comb{i-1}), (\vone \mspace{-1mu}-\mspace{-1mu} \vy^{(i - 1)}) \odot (\vb^{(i)} \mspace{-1mu}-\mspace{-1mu} \vc^{(i - 1)}) + (\va^{(i)} \mspace{-1mu}-\mspace{-1mu} \vy^{(i - 1)}) \odot (\vone \mspace{-1mu}-\mspace{-1mu} \vz^{(i - 1)}) \rangle \\\nonumber&+  (1 - m)(1 - \eps) \cdot e^{\eps i}(t_s-\eps i)\cdot \inner{\vb^{(i)} - \vc^{(i)}}{\nabla F(\vz^{(i - 1)})} \\\nonumber&+ \eps e^{2\eps i} \cdot \langle \nabla F(\comb{i-1}), (\va^{(i)} - \vy^{(i - 1)}) \odot (\vone - \vz^{(i - 1)} - (\vb^{(i)} - \vc^{(i)})) \rangle
	\enspace.
\end{align}

To bound the last term on the rightmost side of the last inequality, we need to make a few observations. First, $\vone - \vz^{(i - 1)} - \vb^{(i)} + \vc^{(i)} \in [0, 1]^n$ because $\vc^{(i)} \leq \vz^{(i - 1)}$ and $\vb^{(i)} \leq \vone - \vz^{(i - 1)}$; second, $\comb{i - 1} \geq \vy^{(i - 1)} \geq \vy^{(i - 1)} \odot (\vone - \va^{(i)}) \odot (\vone - \vz^{(i - 1)} - \vb^{(i)} + \vc^{(i)})$; and finally, by Lemma~\ref{lem:normxy}, $\|\vz^{i - 1} + \va^{(i)} \odot \vc^{(i)}/2\|_\infty \leq \|\vz^{i - 1} + \vc^{(i)}/2\|_\infty \leq (3/2) \cdot \|\vz^{(i - 1)}\|_\infty \leq (3/2) \cdot (1 - 1/e) < 1$. Given all these observations, we can use Properties~\ref{prop:dr_bound2_up} and~\ref{prop:dr_bound2_down} of Lemma~\ref{lem:DR_properties} to get
\begin{align*}
	\langle \nabla F(&\comb{i-1}), (\va^{(i)} - \vy^{(i - 1)}) \odot (\vone - \vz^{(i - 1)} - (\vb^{(i)} - \vc^{(i)})) \rangle\\
	={} &
	\langle \nabla F(\comb{i-1}), \va^{(i)} \odot (\vone - \vy^{(i - 1)}) \odot (\vone - \vz^{(i - 1)} - \vb^{(i)}) \rangle \\& +2 \cdot \langle \nabla F(\comb{i-1}), \va^{(i)} \odot (\vone - \vy^{(i - 1)}) \odot \vc^{(i)}/2 \rangle \\&- \langle \nabla F(\comb{i-1}), \vy^{(i - 1)} \odot (\vone - \va^{(i)}) \odot (\vone - \vz^{(i - 1)} - \vb^{(i)} + \vc^{(i)}) \rangle\\
	\geq{} &
	F(\vy^{(i - 1)} \psum (\va^{(i)} \odot (\vone - \vb^{(i)}) + (\vone - \va^{(i)}) \odot \vz^{(i - 1)})) + 2F(\vy^{(i - 1)} \psum (\vz^{(i - 1)} + \va^{(i)}\odot \vc^{(i - 1)}/2)) \\&+ F(\comb{i - 1} - \vy^{(i - 1)} \odot (\vone - \va^{(i)}) \odot (\vone - \vz^{(i - 1)} - \vb^{(i)} + \vc^{(i)}))\\&- 4 \cdot F(\comb{i - 1})
	\geq
	- 4 \cdot F(\comb{i - 1})
	\enspace,
\end{align*}
where the last inequality holds by the non-negativity of $F$. Plugging this inequality into Inequality~\eqref{eq:diff_first_bound_emp} yields
\begin{align*}
	\eps^{-1}&e^{2t_s}[\phi(i) - \phi(i - 1)]\\\nonumber
	\geq{}&
	2 e^{2\eps (i - 1)} \cdot F(\comb{i-1}) - (1 - m)(1 - \eps) \cdot e^{\eps i}(1 - t_s + \eps i) \cdot F(\vz^{(i - 1)}) - \tfrac{62\eps\beta D^2}{1 - m}\\ \nonumber
	&+ (1 \mspace{-1mu}-\mspace{-1mu} \eps) e^{2\eps i} \cdot \langle \nabla F(\comb{i-1}), (\vone \mspace{-1mu}-\mspace{-1mu} \vy^{(i - 1)}) \odot (\vb^{(i)} \mspace{-1mu}-\mspace{-1mu} \vc^{(i - 1)}) + (\va^{(i)} \mspace{-1mu}-\mspace{-1mu} \vy^{(i - 1)}) \odot (\vone \mspace{-1mu}-\mspace{-1mu} \vz^{(i - 1)}) \\\nonumber&+  (1 - m)(1 - \eps) \cdot e^{\eps i}(t_s-\eps i)\cdot \inner{\vb^{(i)} - \vc^{(i)}}{\nabla F(\vz^{(i - 1)})}  - 30\eps\cdot F(\comb{i - 1})
	\enspace.
\end{align*}
The last two terms on the right hand side of the last inequality are related to the objective function of the first linear program of Algorithm~\ref{alg:OfflineEmp}. Specifically, to get from them to the objective function of this linear program, it is necessary to multiply by $(1 - \eps)^{-1}$, and then remove the additive term $e^{2\eps i} \cdot \langle \nabla F(\comb{i-1}), - \vy^{(i - 1)} \odot (\vone - \vz^{(i - 1)}) \rangle$, which does not depend on the variables $\va^{(i)}$, $\vb^{(i)}$ and $\vc^{(i)}$. Therefore, we can lower bound the right hand side of the last inequality by plugging in feasible solutions for the first linear program of Algorithm~\ref{alg:OfflineEmp}. Specifically, we plug in the solutions $(\va^{(i)}, \vb^{(i)}, \vc^{(i)}) = (\optq, (\vone - \vz^{(i - 1)}) \odot \optpone, \vzero)$ and $(\va^{(i)}, \vb^{(i)}, \vc^{(i)}) = (\vy^{(i - 1)}, \vzero, \vzero)$, which implies the two lower bounds stated in the lemma. Notice that both these solutions are guaranteed to be feasible since $\DM$ is down-closed and $(\vone - \vz^{(i - 1)}) \odot \optpone \leq (\vone - \vz^{(i - 1)}) \odot (\vone - \optq)$.
\end{proof}

\bibliographystyle{plain}
\bibliography{main}

\end{document}